\documentclass[11pt]{article}

\usepackage[margin=1in]{geometry}
\usepackage[utf8]{inputenc}
\usepackage{amsmath,amsfonts,amssymb,amsthm}
\usepackage{graphicx}
\usepackage{natbib}
\usepackage{algorithm}
\usepackage{algorithmic}
\usepackage{xcolor}
\usepackage{pifont}
\usepackage{hyperref}
\usepackage{float}
\usepackage{placeins}
\usepackage[normalem]{ulem}
\usepackage{multirow}
\usepackage{enumitem}
\usepackage{tabularx}
\usepackage{booktabs}
\usepackage{array}
\usepackage{subcaption}

\newcommand{\ve}[1]{\mathbf{#1}}
\newcommand{\cmark}{\textcolor{green!70!black}{\ding{51}}}
\newcommand{\xmark}{\textcolor{red!90!black}{\ding{55}}}

\newtheorem{assump}{Assumption}
\newtheorem{thm}{Theorem}

\newtheorem{lem}{Lemma}
\newtheorem{prop}{Proposition}
\newtheorem{cor}{Corollary}

\newtheorem{rem}{Remark}

\newenvironment{keywords}{\par\noindent\textbf{Keywords: }}{\par}
\title{Quantile-based Loss Filtering for Outlier-Robust Stochastic Gradient Descent}
\author{Jamie Haddock\thanks{Department of Mathematics, Harvey Mudd College, Claremont, CA 91711-3116, USA. Email: \texttt{jhaddock@g.hmc.edu}} \and
Anna Ma\thanks{Department of Mathematics, University of California, Irvine, CA 92617, USA. Email: \texttt{anna.ma@uci.edu}} \and
Elizaveta Rebrova\thanks{Mathematical Institute, University of Oxford,  Oxford, OX1 2JD, UK. Email: \texttt{elre@princeton.edu}}}
\date{}

\begin{document}
\maketitle

\begin{abstract}
{We study loss-based filtering for finite-sum optimization with a subset of
corrupted component functions whose gradients may be highly unreliable.
Motivated by minimum-loss-based SGD (min-$k$-loss) and quantile-based methods for corrupted linear
systems, we propose and
analyze a general loss-filtering framework  -- Quantile-\(k\)-Loss SGD (Q\(k\)L-SGD) --   that
samples \(k\) component losses at each iteration and updates using an index chosen uniformly
from the lower empirical \(q\)-quantile.
We prove linear convergence of this family of methods under standard convexity assumptions,
requiring the sample size to scale with the number of corruptions and a subset
strong-convexity threshold. For the cases when large enough sampling is impossible or undesirable, we
give a complementary small-sample probabilistic analysis that covers any sample size $k$ and the
convergence behavior depends on the probability of selecting an outlier and on
the curvature of the selected good step. Experiments on polynomial regression,
regularized logistic regression, and regularized hinge loss show that
intermediate quantiles often outperform both standard SGD and min-\(k\)-loss
SGD. In particular, min-\(k\) often stalls by repeatedly selecting nearly solved
components, while intermediate quantiles retain robustness and produce more
informative updates.}

\end{abstract}


\section{Introduction}

Stochastic Gradient Descent (SGD) is a fundamental method for solving large-scale
finite-sum optimization problems of the form
\begin{equation}
\min_{\ve{x}} F(\ve{x})
:=
\frac{1}{m}\sum_{i=1}^m f_i(\ve{x}),
\label{eq:prob}
\end{equation}
where each component function \(f_i\) corresponds to an individual data point,
measurement, or training example. {A basic component-sampling  SGD replaces the full gradient by the gradient of a randomly selected
component: \begin{equation}
   \ve{x}_{t+1}
        =
        \ve{x}_t-\eta\nabla f_{i_t}(\ve{x}_t).
        \label{eq:sgd}
\end{equation}}
This update is computationally cheap, but can be vulnerable to corrupted
data: if a small fraction of the component functions are affected by mislabeled
data, faulty measurements, adversarial perturbations, or system failures, then
even rare corrupted updates can substantially distort the method's trajectory. Ordinary mini-batching reduces stochastic variance, but it does not by itself
remove this vulnerability when corrupted gradients can be very large; as a
single corrupted component can dominate an averaged batch update.

Across robust regression and stochastic or distributed optimization, related
methods use three broad, overlapping strategies. Some select, remove, or
reweigh observations or candidate updates using losses, residuals, or gradient
diagnostics
\citep{bhatia2015robust,shen2019learning,diakonikolas2019sever}.
Others construct a robust gradient estimate or aggregate, for example through
robust mean estimation, coordinate-wise median, or trimmed mean
\citep{prasad2020robust,yin2018byzantine}. A third line limits the influence
of an individual selected observation through a robust loss or bounded update,
as in \(\ell_1\)-loss SGD and gradient clipping
\citep{pesme2020online,jeong2025stochastic,gorbunov2020stochastic}.
These approaches address different corruption models and target problems.

Our focus is on a finite-sum corruption model in which the index set \([m]\) is
partitioned as
\[
[m]=G\cup O,
\qquad
G\cap O=\emptyset,
\]
{where \(G\) denotes the uncorrupted, or \textbf{G}ood components and \(O\) denotes corrupted components, or \textbf{O}utliers.} Our goal
is not to minimize the corrupted objective \(F\), but rather to approximate a
minimizer of the uncorrupted objective
\begin{equation}
F_G(\ve{x})
:=
\frac{1}{|G|}\sum_{i\in G} f_i(\ve{x}).
\label{eq:prob_good}
\end{equation}
That is, we seek
\(\ve{x}^*\in\operatorname*{argmin}_{\ve{x}} F_G(\ve{x})\).
The sets \(G\) and \(O\) are unknown to the algorithm, and we do not impose a
stochastic model on how the corruptions are generated. Corrupted component
gradients may therefore be highly unreliable or adversarially aligned, so even
a single corrupted update can result in an arbitrarily perturbed iterate.

Rather than modifying the per-sample update or constructing a robust aggregate
from candidate gradients, we use  \emph{iteration-dependent loss filtering: current component losses are used to
restrict which sampled indices are eligible for the next stochastic update,
without permanently identifying or removing observations}. The motivation is
that, in many corrupted-measurement models, good losses become small near the
target solution while severely corrupted losses remain comparatively large.
Thus the lower tail of the current loss distribution can provide useful
information about which components are more reliable candidates for an update -- even before any closeness to the target solution.

To that end, we introduce \emph{Quantile-\(k\)-Loss SGD} (Q\(k\)L-SGD). At each
iteration, the method samples a subset \(S_t\subset[m]\) of size \(k\), computes
the sampled losses
\[
\{f_i(\ve{x}_t):i\in S_t\},
\]
and accepts those sampled indices whose losses lie at or below the empirical
\(q\)-quantile. The update index is then chosen uniformly from this accepted
lower-tail set. Importantly, Q\(k\)L-SGD does not rely on the accepted lower tail being clean.
Corrupted components may also be accepted, and our analysis
shows that their effect can still be controlled.

This construction defines a two-parameter family of stochastic sampling rules,
with \(q\in(0,1]\) controlling the accepted fraction and \(k\) controlling the
number of losses screened at each iteration. The extreme choice \(q=1/k\)
reduces to min-\(k\)-loss SGD \citep{shah2020choosing}, which selects a minimum-loss component from a
sample of size \(k\). Larger values of \(q\) interpolate between this
conservative choice and ordinary SGD, recovered when \(q=1\). Viewed as a
family, Q\(k\)L-SGD unifies these special cases and explains why intermediate quantiles can offer a better balance between
robustness and optimization progress than either endpoint.

Our main contributions are as follows:
\begin{itemize}

\item \textbf{A quantile loss-filtering framework for robust SGD.}
We formulate and analyze Q\(k\)L-SGD as a two-parameter family of loss-based
stochastic sampling rules for general finite-sum objectives (see Section~\ref{sec:algorithm}). The family contains
min-\(k\)-loss SGD and the standard SGD sampling distribution as endpoints and
extends quantile-based residual selection beyond linear least squares. Treating
\(q\) and \(k\) as separate algorithmic parameters makes explicit the
tradeoff between outlier avoidance, progress on selected good components,
and screening cost.

\item \textbf{A deterministic guarantee.}
Under smoothness and convexity assumptions, we prove linear convergence of Q\(k\)L-SGD  (see Theorem~\ref{thm:main}). The required sample size
is governed by the number of corruptions and by a subset strong-convexity
threshold, rather than directly by the total number of component functions. No
stochastic model for the outliers is required, and the accepted lower-tail set
need not be clean: the empirical quantile threshold and deterministic counting
conditions control the effect of accepted outliers. In some regimes where the number of corruptions is
sublinear in \(m\), this permits sample sizes \(k=o(m)\), improving the
sample-size dependence of earlier subsampled quantile analyses in the linear
setting \citep{haddock2023subsampled}.

\item \textbf{A probabilistic small-sample guarantee.}
When the deterministic counting conditions fail, for example, if the afforded sample size $k$ is constant-order, we derive a complementary convergence bound. Its
rate depends on two interpretable quantities: the probability of
selecting an outlier and the expected curvature of a selected good step (see Theorem~\ref{thm:small-sample-prob}).
Further, under warm-start conditions, we show that these quantities yield an explicit robustness--progress
tradeoff in \(q\) and \(k\) (see Section~\ref{sec:small-sample-hard-separation}).

\item \textbf{Numerical evaluation across models and corruption regimes.}
We evaluate Q\(k\)L-SGD on polynomial regression, regularized logistic
regression, and regularized hinge loss under several corruption models,
including settings beyond the exact assumptions of our theory. Across these
problems, intermediate quantiles often improve performance relative to standard
SGD and min-\(k\)-loss SGD. Experiments varying \(q\) and \(k\) further
illustrate the parameter dependence predicted by the analysis. The best
quantile is problem- and corruption-dependent, with strong empirical performance
often occurring near the clean-data fraction (see Section~\ref{sec:numerics}).

\end{itemize}

The remainder of the paper is organized as follows. Section~\ref{sec:related-work} reviews related work, then Section~\ref{sec:algorithm} formally introduces the Q\(k\)L-SGD framework. Sections~\ref{sec:large-sample} and~\ref{sec:small-sample} develop the two complementary theoretical regimes: a deterministic analysis based on sufficiently large samples and a probabilistic analysis for smaller samples, including an explicit specialization under loss separation. Section~\ref{sec:numerics} evaluates the method across several objective classes and corruption models, and Section~\ref{sec:conclusion} concludes with a brief discussion of some natural future directions of this work.

\subsection{Background and related work} \label{sec:related-work}

\paragraph{Robust statistics and robust regression.}
Learning or estimating in the presence of a small fraction of unreliable data
has a long history in robust statistics, dating back to classical works such
as \citep{tukey1960survey, huber1992robust}. In regression, this
perspective led to methods such as least median of squares and least trimmed
squares \citep{rousseeuw1984least,vivsek2006least}, in which residuals
determine which observations influence the fitted model.

More recently, high-dimensional robust statistics has developed efficient
algorithms for estimation under adversarial contamination, including robust
mean estimation and robust regression
\citep{awasthi2014power,lai2016agnostic,charikar2017learning,
diakonikolas2019robust}. Work by \citet{klivans2018efficient} gives
polynomial-time algorithms for outlier-robust linear and polynomial regression
under adversarial corruptions in examples and labels. These works primarily
study the statistical problem of constructing a robust estimator from a
corrupted sample. Our goal is different: we study an iterative
SGD-type method whose individual updates may be corrupted, and we ask when
loss-based sampling can avoid such corrupted updates.

\paragraph{Residual and loss trimming.}
An algorithmic use of residual information appears in robust regression
through hard thresholding. TORRENT alternates between selecting observations
with small residuals under the current iterate and updating the regressor on
the resulting active set; for squared loss, this is a residual- or loss-trimming
procedure \citep{bhatia2015robust}. Related methods include consistent robust
regression and adaptive hard thresholding for response corruption
\citep{bhatia2017consistent,suggala2019adaptive}. These methods use residuals
to address response or right-hand-side corruption, but are formulated primarily
for linear regression as threshold-and-refit procedures.

A related line of work studies iterative trimmed loss minimization.
\citet{shen2019learning} propose a general framework that repeatedly selects
low-loss samples and updates or retrains a model on the selected subset, while
iterative least trimmed squares methods apply the same principle in linear and
mixed linear regression \citep{shen2019iterative}.
\citet{awasthi2022trimmed} analyze iterative trimmed maximum likelihood for
robust generalized linear models under adversarial corruptions, including label
corruptions. These methods are close in spirit to Q\(k\)L-SGD because they use
losses to select samples. Q\(k\)L-SGD turns this principle into a sampled,
single-component stochastic update rule for general finite-sum objectives.

\paragraph{Robust gradient filtering, estimation, aggregation, and clipping.}
Another line of work uses gradient information to robustify stochastic or
distributed optimization. SEVER uses the spectral structure of per-sample
gradients to score and remove suspect sample functions
\citep{diakonikolas2019sever}, whereas robust gradient-estimation methods
replace the ordinary empirical gradient by a robust estimate
\citep{prasad2020robust}. Byzantine-resilient methods select or aggregate
candidate worker gradients
\citep{blanchard2017machine,alistarh2018byzantine}, including
coordinate-wise median and trimmed-mean aggregation
\citep{yin2018byzantine}.

Related Byzantine-resilient methods use additional information to screen or
stabilize candidate updates. For example, Zeno++ accepts asynchronous worker
updates according to an estimated loss-descent score
\citep{xie2020zeno++}, while other approaches incorporate information across
iterations, such as worker momentum, to address attacks coupled over time
\citep{karimireddy2021learning}. Gradient clipping and norm-based filtering
instead limit the influence of individual large gradients
\citep{gorbunov2020stochastic,gupta2021byzantine}. Thus, these methods use
gradient information to filter, estimate, aggregate, score, or control
candidate updates. In contrast, Q\(k\)L-SGD uses current component losses to
determine which sampled components are eligible to generate the next
stochastic update.

\paragraph{Robust losses and bounded-influence stochastic updates.}
A complementary line of work builds robustness into the per-sample loss or
update itself.
For online linear and ReLU regression, SGD on the \(\ell_1\) loss limits the
influence of the magnitude of a corrupted residual; combined with iterate
averaging or decaying step sizes, this yields convergence guarantees under
oblivious or Massart response corruption
\citep{pesme2020online,jeong2025stochastic}. The Q\(k\)L-SGD method retains
a general component-loss formulation and modifies which sampled component is
used for the update.

\paragraph{Small-loss sample selection in noisy-label learning.}
A broader noisy-label learning literature also uses small loss as a proxy for
sample reliability. MentorNet learns a curriculum weighting rule that
emphasizes examples whose labels are likely to be correct
\citep{jiang2018mentornet}; Co-teaching uses two networks to select small-loss
examples for one another \citep{han2018coteaching}; and DivideMix models the
per-sample loss distribution to separate likely clean and noisy examples
\citep{li2020dividemix}. These are primarily deep-learning procedures
specialized to noisy-label training, whereas Q\(k\)L-SGD defines an explicit
finite-sum stochastic sampling rule with convergence theory. Moreover, our
main analysis does not require the accepted lower-tail set to be
clean: accepted outliers are permitted and their effect is controlled through
the empirical quantile threshold and deterministic counting conditions.

\bigskip

The next two paragraphs review the two closest technical predecessors of
Q\(k\)L-SGD.

\paragraph{Quantile randomized Kaczmarz and corrupted linear systems.}
For linear systems with corrupted right-hand sides, quantile-based randomized
Kaczmarz methods select rows using residual quantiles in order to avoid corrupted
equations. This line of work includes quantile randomized Kaczmarz and its
subsampled or block variants
\citep{haddock2019randomized,steinerberger2023quantile,HNRS20, jarman2021quantilerk,cheng2023block,haddock2023subsampled}.
The linear least-squares setting is a special case of finite-sum optimization:
for
\[
F(\ve{x})
=
\frac{1}{2m}\|\ve{A}\ve{x}-\ve{b}\|^2
=
\frac{1}{2m}\sum_{i=1}^m(\ve{a}_i^\top\ve{x}-b_i)^2,
\]
an SGD step on one squared residual is
\[
\ve{x}_{t+1}
=
\ve{x}_t
-
\eta(\ve{a}_i^\top\ve{x}_t-b_i)\ve{a}_i.
\]
When the rows are normalized and \(\eta=1\), this is the randomized Kaczmarz
update. Thus, for squared residual losses, residual-quantile selection is equivalent
to loss-quantile selection, and Q\(k\)L-SGD can be viewed as an extension of
quantile-based Kaczmarz sampling to general finite-sum objectives.

\paragraph{Min-\(k\)-loss SGD.}
The closest nonlinear SGD method is min-\(k\)-loss SGD, which samples \(k\)
components and updates the current iterate using the gradient of a component
attaining the smallest current loss \citep{shah2020choosing}. That work develops landscape and convergence analyses for the minimum-loss rule under separation and noise assumptions.
Q\(k\)L-SGD generalizes min-\(k\)-loss SGD from the minimum to a general sampled
quantile. The additional quantile parameter is important because minimum-loss
selection can be highly conservative, especially when \(k\) is moderate or
large. Our analysis and experiments show that intermediate quantiles can
retain substantial robustness while producing more informative updates.

\bigskip

Overall, the existing works give rich evidence that losses and residuals can provide useful
signals for robust estimation and optimization. They show that robustness can be introduced at
different stages of a stochastic method: by controlling which observations or
candidate updates influence a step, by constructing a robust gradient estimate
or aggregate, or by limiting the influence of each selected observation.
Q\(k\)L-SGD develops the first strategy using current component losses as the
selection statistic while retaining an otherwise standard component-gradient
update. Existing loss- and residual-based theory has focused on
problem-specific trimming and refitting procedures, quantile selection in
linear systems, or the minimum-loss endpoint of stochastic sampling. We
analyze the full sampled lower-quantile family for general finite-sum
objectives, treating \(q\) and \(k\) as separate design parameters and making
explicit their tradeoff between outlier avoidance and optimization progress.

\subsection{Notation}

We write \([m]=\{1,\ldots,m\}\). The index set is partitioned as
\([m]=G\cup O\), where \(G\) is the good set and \(O\) is the outlier set. For
any nonempty \(S\subseteq[m]\), define
\[
F_S(\ve{x})
:=
\frac{1}{|S|}\sum_{i\in S}f_i(\ve{x}).
\]
In particular,
\[
F_G(\ve{x})
=
\frac{1}{|G|}\sum_{i\in G}f_i(\ve{x}).
\]

For a multiset \(R=\{z_1,\ldots,z_k\}\), let
\[
z_{(1)}\le z_{(2)}\le\cdots\le z_{(k)}
\]
denote its order statistics. Throughout the paper, the sample size \(k\) and quantile level \(q\in(0,1]\) are
chosen so that \(qk\in\{1,\ldots,k\}\) is an integer. We denote the empirical
\(q\)-quantile by
\[
Q_q(R):=z_{( qk)}.
\]
Thus \(q=1/k\) gives the minimum of the sample, while \(q=1\) gives the maximum.

At iteration \(t\), \(S_t\subset[m]\) denotes a uniformly sampled subset of size
\(k\). {The lower-tail set at iteration $t$ is the set, or \textbf{B}ag of acceptable components, i.e., sampled components whose loss falls below the empirical quantile:}
\[
B_t
:=
\{j\in S_t:\,
f_j(\ve{x}_t)\le Q_q(\{f_i(\ve{x}_t)\}_{i\in S_t})\}.
\]
We also write
\[
O_t:=B_t\cap O,
\qquad
G_t:=B_t\cap G,
\]
{to denote the set of outlier and good components in $B_t$, respectively.}
\section{Framework: Quantile-\(k\)-loss SGD}
\label{sec:algorithm}

Algorithm~\ref{alg:qklsgd} formalizes Q\(k\)L-SGD: at each iteration, the
method restricts the eligible update indices to the lower empirical
\(q\)-quantile of the losses in a uniformly sampled set.

\begin{algorithm}
    \caption{Quantile-\(k\)-loss SGD (Q\(k\)L-SGD)}
    \label{alg:qklsgd}
    \begin{algorithmic}[1]
    \STATE \textbf{Input:} sample size \(k\), quantile \(q\), step size \(\eta\), number of iterations \(T\)
    \STATE Initialize \(\ve{x}_0\)
    \FOR{\(t=0,\ldots,T-1\)}
        \STATE Uniformly select a subset \(S_t\subset [m]\) with \(|S_t|=k\)
        \STATE Set
        \[
        B_t
        :=
        \{j\in S_t:
        f_j(\ve{x}_t)
        \le
        Q_q(\{f_\ell(\ve{x}_t)\}_{\ell\in S_t})\}
        \]
        \STATE Uniformly select \(i_t\in B_t\)
        \STATE Update
        \[
        \ve{x}_{t+1}
        =
        \ve{x}_t-\eta\nabla f_{i_t}(\ve{x}_t)
        \]
    \ENDFOR
    \RETURN \(\ve{x}_T\)
    \end{algorithmic}
\end{algorithm}

Note that the accepted set \(B_t\) is recomputed from the current losses at every
iteration; the method does not permanently identify or remove outliers. The
quantile parameter controls the size of the eligible set. When \(q=1/k\),
Q\(k\)L-SGD reduces to min-\(k\)-loss SGD, up to ties. When \(q=1\), every
sampled index is eligible, and the resulting update index has the standard
uniform component-sampling distribution of SGD. Intermediate values of \(q\)
interpolate between these two rules. The next sections develop the analysis showing when the
accepted good components provide sufficient progress and when the effect of
accepted outliers remains controlled.

\subsection{Analysis for the sample size larger than the number of corruptions}\label{sec:large-sample}

{In this section, we analyze QkL-SGD in the regime where the sample size $k$ is large relative to the number of corruptions $s$. Under this condition, we can guarantee deterministically that every sampled batch contains enough good components both below and above the quantile threshold. We first state the sampling and model assumptions (Assumptions~\ref{assump:samplesize} and~\ref{assump:sgd}), which require that the accepted lower-tail set contains at least $\tau = qk - s$ good indices and that the averaged objective over any sufficiently large good subset is strongly convex. Our main result, Theorem~\ref{thm:main}, then establishes linear convergence in expectation of the QkL-SGD iterates to the minimizer $\ve{x}^*$ of the uncorrupted objective, with a contraction rate governed by the probability $\rho$ of selecting an accepted outlier and the potential error inflation $\chi$ from a corrupted step. }

\begin{assump}
{\textbf{(Sampling and corruption regime.)}}
    Let
$s:=|O|$
be the number of outliers, $k$ be the sample size ($1\le k\le m$), $q \in (0,1)$ be the quantile, and assume $qk$ is an integer. We assume that
\begin{equation}\label{eq:tau-nu}
{\tau:= qk -s > 0}
\qquad \text{ and } \qquad
{\nu:=k- qk-s> 0.}
\end{equation}
    Equivalently,
    \[
   { s< qk}
    \qquad
    \text{and}
    \qquad
    {s<k- qk}.
    \]
    The first inequality ensures that the accepted lower-tail set contains enough good indices. The second inequality ensures that the quantile threshold can be controlled by good components above the threshold.

\label{assump:samplesize}
\end{assump}

Before discussing Assumption~\ref{assump:samplesize} and its implications on the sample size, we first present additional model assumptions.

\begin{assump} {\textbf{(Model setting.)} We adopt the following model assumptions on the component functions $f_i$.}
\label{assump:sgd}

\begin{enumerate}[label=\Roman*., ref=\Roman*]
    \item \textbf{Good-component regularity.}
    For every $i\in G$, the function $f_i$ is convex and
    \[
    f_i(\ve{x}^*)=0,
    \qquad
    \nabla f_i(\ve{x}^*)=0,
    \] where $\ve{x}^*$ minimizes the uncorrupted objective $F_G(\ve{x})$.
    \label{ass:good_set_regularity}

    \item \textbf{Outlier nonnegativity.}
    For every $i\in O$ and every $\ve{v}\in\mathbb{R}^n$,
    \[
    f_i(\ve{v})\ge 0.
    \]\label{ass:outlier_nonnegativity}

    \item \textbf{Component smoothness.}
    For every $i\in [m]$, the function $f_i$ has $L_i$-Lipschitz gradient, i.e.,
    \[
    \|\nabla f_i(\ve{x})-\nabla f_i(\ve{y})\|_2
    \le
    L_i\|\ve{x}-\ve{y}\|_2
    \qquad
    \text{for all }\ve{x},\ve{y}\in\mathbb{R}^n.
    \]
    {For sets $S\subseteq [m]$, we use the conventions
    \begin{equation}\label{eq:lmax}
    L_{S}^*
    :=
    \begin{cases}
    \max_{i\in S}L_i, & S\neq\emptyset,\\
    0, & S=\emptyset.
    \end{cases}
    \quad \text{ and } \quad
    L_{G,k}^*
    :=
    \max_{\substack{U\subseteq G\\ |U|\le k}}
    \sum_{i\in U}\frac{L_i}{2}.
    \end{equation}
    Note that in particular
    \[
    L_{G,k}^*\le \frac{k}{2}{L_G^*}.
    \]}
    \label{ass:component_smoothness}
     \item
    \textbf{Subset strong convexity on good accepted sets.}
    There exists integer $r_{\mathrm{sc}}$ such that $1 \le r_{\mathrm{sc}} \le \tau$, such that for every subset $S\subseteq G$ with
    $
    |S|\ge r_{\mathrm{sc}},
    $
    the averaged objective $F_S$ is $\mu_S$-strongly convex, i.e.,
    \[
    \bigl\langle \ve{x}-\ve{y},\,
    \nabla F_S(\ve{x})-\nabla F_S(\ve{y})\bigr\rangle
    \ge
    \mu_S \|\ve{x}-\ve{y}\|^2
    \qquad
    \text{for all }\ve{x},\ve{y}\in\mathbb{R}^n.
    \]
    {Furthermore, we denote the smallest strong convexity parameter over good sets of size at least $\tau$ by $\mu^*$
    \[
    \mu^*
    :=
    \min_{\substack{T\subseteq G\\ |T|\ge \tau}}\mu_T.
    \]}

    \label{ass:subset_strong_convexity}
  \end{enumerate}
\end{assump}

The role of the sample size $k$ is to ensure that, at every iteration, the accepted set contains enough good indices both
(i) below the quantile threshold, so that the update step can {sample from} a sufficiently large good subset, and
(ii) above the quantile threshold, so that the threshold itself can be controlled by good components. As a result, Assumption~\ref{assump:samplesize} requires the two conditions
\[
|G_t| \ge  \tau \ge r_{\mathrm{sc}},
\qquad
{k- qk-s > 0}.
\]
{The first condition is a result of the simple calculation \begin{equation}\label{eq:ind-lower-bd}
|G_t|
=
|B_t\cap G|
\ge
|B_t|-|O_t|
\ge
 qk-s
=
\tau.
\end{equation}} The latter condition is needed because, after removing up to $s$ outliers, there remain at least {$\nu= k- qk-s >0$} good sampled indices above the threshold, {which allows us to relate the loss quantile value and the error}. In particular, we can view the terms $\tau$ and $\nu$
   ``slack in the accepted and rejected set," where the ``slack" is for the worst case (all corruptions in the accepted/rejected set).

Up to integer rounding, this means that the sample size should be chosen so that
\[
k
\gtrsim
\max\left\{
\frac{s+r_{\mathrm{sc}}}{q},
\frac{s}{1-q}
\right\}.
\]
This gives a practical rule for selecting $k$: it must be large enough to overcome the number of outliers $s$, and large enough so that the accepted good set reaches the subset strong convexity threshold $r_{\mathrm{sc}}$.

    A useful point is that this requirement does \emph{not} force $k$ to scale with the total number of components $m$. In earlier subsampled methods for the linear setting, such as \cite{haddock2023subsampled, HNRS20}, theoretical guarantees typically required taking a sample size of order $O(m)$. In contrast, the present result allows $k$ to be chosen in terms of the outlier level $s$ and the structural threshold $r_{\mathrm{sc}}$, which can be much smaller than $m$.

The present assumptions always exclude the extreme choice $q=1/k$ whenever $s>0$, since then
\[
{ qk  = 1},
\qquad
\tau = 1-s \le 0.
\]
In other words, as soon as there is even one outlier, the proof requires an accepted set larger than a single index. This is also consistent with a basic robustness obstruction: if $s\ge k$, then the random sample $S_t$ can consist entirely of outliers, and under uniform sampling without replacement this happens with positive probability.
This event is detrimental to the analysis since, in this case, every accepted index is corrupted,
since
\[
B_t\subseteq S_t\subseteq O,
\]
regardless of the quantile level $q$. Hence the update necessarily uses a corrupt gradient, and if those outlier gradients are very large or adversarially aligned, a single such step can be arbitrarily harmful. {Practically}, for moderate or large sample size $k$, the extreme choice $q=1/k$ (that is, selecting the smallest residual/loss) is typically too conservative and leads to unnecessarily slow progress. Thus both the theory and the practical behavior suggest choosing $k$ large enough to ensure robustness, without relying on the minimum-loss rule. We illustrate this tradeoff in the numerical experiments in Section~\ref{sec:numerics}.

{For simplicity of analysis, we assume that $f_i(\ve{x}^*)=0$ {for $i \in G$}. However, in general, this is not necessary, and we expect that a nonzero value of $f_i$ at $\ve{x}^*$ will, at worst, introduce a convergence horizon. For example, see the analysis of QRK in the noisy least squares case~\citep{battaglia2026quantile, CHPTimeVarying24}.} {We note that our experiments in Sections~\ref{subsec:poly_regression}, \ref{subsec:logistic_regression}, and~\ref{subsec:hinge_loss} support this hypothesis in and beyond the quadratic least-squares objective.}

{We now present our main result in this section, which illustrates that the QkL-SGD iterates converge to the solution of the optimization problem restricted to only good components under Assumptions~\ref{assump:samplesize} and~\ref{assump:sgd}.}

\begin{thm}\label{thm:main}
Suppose {Assumptions~\ref{assump:samplesize} and~\ref{assump:sgd}} hold, and let $\ve{x}_t$ denote the iterates
of Algorithm~\ref{alg:qklsgd}.
Assume moreover that $\mu^*$ satisfies the structural condition
\begin{equation}\label{eq:mu-star}
    \mu^* =
\min_{\substack{T\subseteq G\\ |T|\ge \tau}}\mu_T
>
\frac{\rho \chi}{ 1 - \rho},
\end{equation}
where we denote
\[
{\rho:=\frac{s}{ qk }},
\qquad
\chi:=
\sqrt{
\frac{2L_{O}^*L_{G,k}^*}{\nu}
},
\]
and choose the step size so that
\[
0<\eta<
\frac{
2\bigl((1-\rho)\mu^*-\rho \chi\bigr)
}{
(1-\rho)\mu^*L_{G}^*+\rho \chi^2
}.
\]
Then
\[
\mathbb{E}\|\ve{x}_{t}-\ve{x}^*\|^2
\le
\kappa^t_{\mathrm{rate}}\|\ve{x}_0-\ve{x}^*\|^2,
\]
where
\[
\kappa_{\mathrm{rate}}
=
(1-\rho)(1-2\eta\mu^*+\eta^2\mu^*L_{G}^*)+\rho (1+\eta \chi)^2<1.
\]
\end{thm}

We first develop some auxiliary results to prove Theorem~\ref{thm:main}.
Lemma~\ref{lem:quantilebound-sgd} and Lemma~\ref{lem:corruptProjection-sgd}  handle the case where the selected update index is an outlier, first by showing that the quantile can be bounded, and then by showing that any inflation in error can be controlled. Lemma~\ref{lem:noncorruptProjection-sgd} handles the case where the selected update index is good.

\begin{lem}\label{lem:quantilebound-sgd}
Let $\ve{v}\in\mathbb{R}^n$ be an arbitrary fixed vector. Under {Assumptions~\ref{assump:samplesize} and~\ref{assump:sgd}},
\[
Q_q(\{f_i(\ve{v})\}_{i\in S_t})
\le
\frac{L_{G,k}^*}{\nu}
\|\ve{v}-\ve{x}^*\|^2,
\]
\end{lem}

\begin{proof}
Let
\begin{equation}\label{eq:gamma-quant}
\gamma:=Q_q(\{f_i(\ve{v})\}_{i\in S_t}).
\end{equation}
Recall that $|S_t| = k$ and by definition
of the empirical $q$-quantile,
{$
\gamma=z_{( qk )}$} {where the $\{z_{(j)}\}_{j=1}^k$ values are the order statistics of $\{f_i(\ve{v})\}_{i \in S_t}$}, and
at least
{$
k- qk -s
=
\nu
$}

indices in $S_t\cap G$ satisfy
\[
f_i(\ve{v})\ge \gamma.
\]

By outlier nonnegativity and by convexity plus good-set {regularity} for the good components, all sampled losses are nonnegative. Hence $\gamma\ge 0$ and
\begin{equation}\label{eq:est-1}
\gamma\,\nu
\le
\sum_{i\in S_t\cap G} f_i(\ve{v}).
\end{equation}
Let us further bound the right-hand side. Using the descent inequality implied by the $L_i$-Lipschitz gradient assumption {(Assumption~\ref{assump:sgd}.\ref{ass:component_smoothness})}, together with the good-component regularity at the point $\ve{x}^*$, we obtain for each $i\in G$
\[
{f_i(\ve{v}) =} f_i(\ve{v})-f_i(\ve{x}^*)
\le
\nabla f_i(\ve{x}^*)^T(\ve{v}-\ve{x}^*)
+
\frac{L_i}{2}\|\ve{v}-\ve{x}^*\|^2
=
\frac{L_i}{2}\|\ve{v}-\ve{x}^*\|^2.
\]
Summing over $i\in S_t\cap G$ yields
\begin{equation}\label{eq:est-2}
\sum_{i\in S_t\cap G} f_i(\ve{v})
\le
\sum_{i\in S_t\cap G}
\frac{L_i}{2}\|\ve{v}-\ve{x}^*\|^2 \le L_{G,k}^*
\|\ve{v}-\ve{x}^*\|^2,
\end{equation}
since $S_t\cap G\subseteq G$, $|S_t\cap G|\le k$, and by definition of
$L_{G,k}^*$,
$
\sum\limits_{i\in S_t\cap G}
\frac{L_i}{2}
\le
L_{G,k}^*.$
Combining \eqref{eq:gamma-quant}, \eqref{eq:est-1} and \eqref{eq:est-2}, we conclude that
\[
Q_q(\{f_i(\ve{v})\}_{i\in S_t})
\le
\frac{L_{G,k}^*}{\nu}
\|\ve{v}-\ve{x}^*\|^2.
\]
\end{proof}

Lemma~\ref{lem:gradbound} is the standard self-bounding inequality for smooth nonnegative functions; see, e.g., \cite[Lemma 4]{li2019convergence}.

\begin{lem}\label{lem:gradbound}

Let $f$ be a nonnegative differentiable function satisfying the descent
inequality with constant $L>0$, i.e.,
\[
f(\ve{x})-f(\ve{y})
\le
\nabla f(\ve{y})^T(\ve{x}-\ve{y})
+\frac{L}{2}\|\ve{x}-\ve{y}\|^2
\qquad
\text{for all }\ve{x},\ve{y}\in\mathbb{R}^n.
\]
Then, for any $\ve{v}\in\mathbb{R}^n$,
\[
\|\nabla f(\ve{v})\|^2\le 2L f(\ve{v}).
\]
\end{lem}

{Now, we combine Lemmas~\ref{lem:quantilebound-sgd} and~\ref{lem:gradbound} to bound the increase in error caused by an iteration using an outlier component.}

\begin{lem}\label{lem:corruptProjection-sgd}
Suppose {Assumptions~\ref{assump:samplesize} and~\ref{assump:sgd}} hold, and let $\ve{x}_t$ denote the iterates
of Algorithm~\ref{alg:qklsgd}. If the selected update index satisfies
$i_t\in O_t = B_t\cap O$, then
\[
\|\ve{x}_{t+1}-\ve{x}^*\|^2
\le
r_O\|\ve{x}_t-\ve{x}^*\|^2,
\]
where
\[
r_O=(1+\eta \chi)^2,
\qquad
\chi=
\sqrt{
\frac{2L_{O}^*L_{G,k}^*}{\nu}
}.
\]
\end{lem}

\begin{proof}
Fix a realization of $S_t$ and suppose $i_t\in O_t$. Starting with
$\|\ve{x}_{t+1}-\ve{x}^*\|$, we have
\begin{align*}
\|\ve{x}_{t+1}-\ve{x}^*\|
=
\|\ve{x}_t-\ve{x}^*-\eta\nabla f_{i_t}(\ve{x}_t)\| \le
\|\ve{x}_t-\ve{x}^*\|
+
\eta\|\nabla f_{i_t}(\ve{x}_t)\|.
\end{align*}
Since $i_t\in O_t\subseteq B_t$, Lemma~\ref{lem:quantilebound-sgd}  and Lemma~\ref{lem:gradbound}
give
\begin{align*}
\|\nabla f_{i_t}(\ve{x}_t)\|^2
\le
2L_{i_t} f_{i_t}(\ve{x}_t)
\le
2L_{i_t} Q_q(\{f_j(\ve{x}_t)\}_{j\in S_t})
\le
\frac{2L_{O}^*L_{G,k}^*}{\nu}
\|\ve{x}_t-\ve{x}^*\|^2.
\end{align*}
Thus,
\[
\|\nabla f_{i_t}(\ve{x}_t)\|
\le
\chi\|\ve{x}_t-\ve{x}^*\|.
\]
Consequently,
\[
\|\ve{x}_{t+1}-\ve{x}^*\|
\le
(1+\eta \chi)
\|\ve{x}_t-\ve{x}^*\|.
\]
Squaring both sides gives the claimed result.
\end{proof}

{We now turn to the case that the sampled component is good.}

\begin{lem}\label{lem:noncorruptProjection-sgd}
Suppose {Assumptions~\ref{assump:samplesize} and~\ref{assump:sgd}} hold, and let $\ve{x}_t$ denote the iterates
of Algorithm~\ref{alg:qklsgd}. Assume that the step size satisfies
\[
\eta \le \frac{2}{L_{G}^*}.
\]
Conditioned on the selected update index satisfying $i_t\in G_t$, the algorithm
selects $i_t$ uniformly from $G_t$, and
\[
\mathbb{E}_{i_t\sim \mathrm{Unif}(G_t)}
\|\ve{x}_{t+1}-\ve{x}^*\|^2
\le
r_G \|\ve{x}_t-\ve{x}^*\|^2,
\]
where
\[
r_G := 1-2\eta\mu^*+\eta^2\mu^*L_{G}^*,
\quad \text{ and } \quad
\mu^*
:=
\min_{\substack{T\subseteq G\\ |T|\ge \tau}}\mu_T.
\]
\end{lem}

\begin{proof}
Since conditional on $(\ve{x}_t,S_t)$, the algorithm is uniform on $B_t$; then conditional on the event $i_t\in G_t$, it is uniform on $G_t$.

Recall that by \eqref{eq:ind-lower-bd} and Assumption~\ref{assump:sgd} the number of good accepted indices is large enough for the subset strong convexity
assumption to apply. Now fix $i_t\in G_t$. Since $i_t\in G$, Assumption~\ref{assump:sgd} gives
$\nabla f_{i_t}(\ve{x}^*)=0.$
We have
\begin{align*}
\|\ve{x}_{t+1}-\ve{x}^*\|^2
&=
\|\ve{x}_t-\ve{x}^*-\eta\nabla f_{i_t}(\ve{x}_t)\|^2 \\
&=
\|\ve{x}_t-\ve{x}^*\|^2
-2\eta
\langle \ve{x}_t-\ve{x}^*,\nabla f_{i_t}(\ve{x}_t)\rangle
+\eta^2\|\nabla f_{i_t}(\ve{x}_t)\|^2 \\
&=
\|\ve{x}_t-\ve{x}^*\|^2
-2\eta
\langle
\ve{x}_t-\ve{x}^*,
\nabla f_{i_t}(\ve{x}_t)-\nabla f_{i_t}(\ve{x}^*)
\rangle
+\eta^2
\|\nabla f_{i_t}(\ve{x}_t)-\nabla f_{i_t}(\ve{x}^*)\|^2 \\
&\le \|\ve{x}_t-\ve{x}^*\|^2
-
\bigl(2\eta-\eta^2L_{i_t}\bigr)
\bigl\langle
\ve{x}_t-\ve{x}^*,
\nabla f_{i_t}(\ve{x}_t)-\nabla f_{i_t}(\ve{x}^*)
\bigr\rangle,
\end{align*}
where the last step holds by co-coercivity for convex $L_{i_t}$-smooth functions.
Since $i_t\in G$, convexity implies
\[
\bigl\langle
\ve{x}_t-\ve{x}^*,
\nabla f_{i_t}(\ve{x}_t)-\nabla f_{i_t}(\ve{x}^*)
\bigr\rangle
\ge 0.
\]
Using $L_{i_t}\le L_{G}^*$, we get
\begin{align*}
\|\ve{x}_{t+1}-\ve{x}^*\|^2
&\le
\|\ve{x}_t-\ve{x}^*\|^2 -
\bigl(2\eta-\eta^2L_{G}^*\bigr)
\bigl\langle
\ve{x}_t-\ve{x}^*,
\nabla f_{i_t}(\ve{x}_t)-\nabla f_{i_t}(\ve{x}^*)
\bigr\rangle.
\end{align*}

Now take expectation over the uniform choice of $i_t\in G_t$. Since
\[
F_{G_t}(\ve{x})
=
\frac{1}{|G_t|}\sum_{i\in G_t}f_i(\ve{x}),
\]
we have
\[
\mathbb{E}_{i_t\sim\mathrm{Unif}(G_t)}
\left[
\nabla f_{i_t}(\ve{x}_t)-\nabla f_{i_t}(\ve{x}^*)
\right]
=
\nabla F_{G_t}(\ve{x}_t)-\nabla F_{G_t}(\ve{x}^*).
\]
Therefore
\begin{align*}
\mathbb{E}_{i_t\sim\mathrm{Unif}(G_t)}
\|\ve{x}_{t+1}-\ve{x}^*\|^2
&\le
\|\ve{x}_t-\ve{x}^*\|^2 -
\bigl(2\eta-\eta^2L_{G}^*\bigr)
\bigl\langle
\ve{x}_t-\ve{x}^*,
\nabla F_{G_t}(\ve{x}_t)-\nabla F_{G_t}(\ve{x}^*)
\bigr\rangle.
\end{align*}
Since $|G_t|\ge \tau$ and $G_t\subseteq G$, the subset strong convexity
assumption implies
\[
\bigl\langle
\ve{x}_t-\ve{x}^*,
\nabla F_{G_t}(\ve{x}_t)-\nabla F_{G_t}(\ve{x}^*)
\bigr\rangle
\ge
\mu_{G_t}\|\ve{x}_t-\ve{x}^*\|^2
\ge
\mu^*\|\ve{x}_t-\ve{x}^*\|^2.
\]
Since $\eta\le 2/L_{G}^*$,
$2\eta-\eta^2L_{G}^*\ge 0,$
and we conclude
\begin{align*}
\mathbb{E}_{i_t\sim\mathrm{Unif}(G_t)}
\|\ve{x}_{t+1}-\ve{x}^*\|^2
&\le
\|\ve{x}_t-\ve{x}^*\|^2
-
\bigl(2\eta-\eta^2L_{G}^*\bigr)\mu^*
\|\ve{x}_t-\ve{x}^*\|^2 \\
&=
\left(
1-2\eta\mu^*
+\eta^2\mu^*L_{G}^*
\right)
\|\ve{x}_t-\ve{x}^*\|^2.
\end{align*}
This proves the claim.
\end{proof}

We are now ready to prove the main Theorem~\ref{thm:main}. We will shorten

\[
\mathbb{E}_t[\cdot]
:=
\mathbb{E}[\cdot\mid \ve{x}_t],
\qquad
\mathbb{P}_t(\cdot)
:=
\mathbb{P}(\cdot\mid \ve{x}_t),
\]
where the randomness is over the fresh sample set $S_t$ and the uniform choice
of $i_t\in B_t$.

\begin{proof}{(Proof of Theorem~\ref{thm:main})}

\textbf{Step 1 {(Checking step size assumption)}.} First, we check that the stated step-size bound indeed implies the admissibility condition. Recall {that}
\begin{equation}\label{eq:eta-standard}
0<\eta\le \frac{2}{L_{G}^*}
\end{equation}
{is} required in Lemma~\ref{lem:noncorruptProjection-sgd}.
Indeed, by \eqref{eq:mu-star}, we have
\[
(1-\rho)\mu^*>\rho \chi,
\]
and therefore
\[
\bar\eta:=
\frac{
2\bigl((1-\rho)\mu^*-\rho \chi\bigr)
}{
(1-\rho)\mu^*L_{G}^*+\rho \chi^2
}
>0.
\]
We note that
\[
L_{G}^*\bigl((1-\rho)\mu^*-\rho \chi\bigr)
\le
(1-\rho)\mu^* L_{G}^*+\rho \chi^2,
\]
and so
\[
\bar\eta=
\frac{
2\bigl((1-\rho)\mu^*-\rho \chi\bigr)
}{
(1-\rho)\mu^* L_{G}^*+\rho \chi^2
}
\le
\frac{2}{L_{G}^*}
\]
and \eqref{eq:eta-standard} is satisfied for the considered step sizes $\eta \in (0, \bar{\eta}]$.

\textbf{Step 2 {(A one-step bound)}.} We now prove the one-step estimate. Given $(\ve{x}_t,S_t)$, the sets $B_t$, $G_t$, and $O_t$ are deterministic, and the algorithm chooses $i_t$ uniformly from $B_t$. Therefore,
\begin{align*}
\mathbb{E}\!\left[\|\ve{x}_{t+1}-\ve{x}^*\|^2 \mid \ve{x}_t,S_t\right]
&=
\mathbb{P}(i_t\in G_t\mid \ve{x}_t,S_t)\,
\mathbb{E}\!\left[
\|\ve{x}_{t+1}-\ve{x}^*\|^2
\,\middle|\,
\ve{x}_t,S_t,i_t\in G_t
\right] \\
&\qquad
+
\mathbb{P}(i_t\in O_t\mid \ve{x}_t,S_t)\,
\mathbb{E}\!\left[
\|\ve{x}_{t+1}-\ve{x}^*\|^2
\,\middle|\,
\ve{x}_t,S_t,i_t\in O_t
\right].
\end{align*}
Now define
\[
p_t(S_t):=\mathbb{P}(i_t\in O_t\mid \ve{x}_t,S_t)=\frac{|O_t|}{|B_t|}.
\]
Since $|O_t|\le s$ and {$|B_t|\ge  qk$}, we have
{\[
p_t(S_t)\le \frac{s}{ qk}=\rho.
\]}
Also, conditional on $(\ve{x}_t,S_t)$ and the event $i_t\in G_t$, the index $i_t$ is uniformly distributed on $G_t$; similarly, conditional on $(\ve{x}_t,S_t)$ and the event $i_t\in O_t$, the selected index lies in $O_t$. Thus, with $r_O$ as defined in Lemma~\ref{lem:corruptProjection-sgd} and with $r_G$ as defined in Lemma~\ref{lem:noncorruptProjection-sgd}, we can bound
\begin{align*}
\mathbb{E}\!\left[\|\ve{x}_{t+1}-\ve{x}^*\|^2 \mid \ve{x}_t,S_t\right]
&\le
(1-p_t(S_t))r_G\|\ve{x}_t-\ve{x}^*\|^2
+
p_t(S_t) r_O\|\ve{x}_t-\ve{x}^*\|^2 \\
&=
\bigl((1-p_t(S_t))r_G+p_t(S_t) r_O\bigr)
\|\ve{x}_t-\ve{x}^*\|^2.
\end{align*}
Next, $r_O=(1+\eta \chi)^2\ge 1$ and, since $\eta\le 2/L_{G}^*$,
$
r_G
=
1-\eta\mu^*(2-\eta L_{G}^*)
\le 1.$
Hence the function
\[
p\mapsto (1-p)r_G+pr_O
\]
is increasing on $[0,1]$. Using $p_t(S_t)\le \rho$, we obtain
\[
\mathbb{E}\!\left[\|\ve{x}_{t+1}-\ve{x}^*\|^2 \mid \ve{x}_t,S_t\right]
\le \kappa_{\mathrm{rate}}
\|\ve{x}_t-\ve{x}^*\|^2,
\qquad
\kappa_{\mathrm{rate}}=(1-\rho)r_G+\rho r_O.
\]
Finally, taking conditional expectation over $S_t$ given $\ve{x}_t$ yields
\[
\mathbb{E}_t\|\ve{x}_{t+1}-\ve{x}^*\|^2
\le \kappa_{\mathrm{rate}}
\|\ve{x}_t-\ve{x}^*\|^2.
\]
This proves the one-step bound.

\textbf{Step 3 {(Contraction guarantee)}.} It remains to show that $\kappa_{\mathrm{rate}}<1$.
We have
\[
\kappa_{\mathrm{rate}}
=
(1-\rho)(1-2\eta\mu^*+\eta^2\mu^* L_{G}^*)
+
\rho(1+\eta \chi)^2.
\]
Therefore
\[
\kappa_{\mathrm{rate}}-1
=
-2\eta\bigl((1-\rho)\mu^*-\rho \chi\bigr)
+
\eta^2\bigl((1-\rho)\mu^* L_{G}^*+\rho \chi^2\bigr).
\]
By the structural condition on $\mu^*$,
\[
(1-\rho)\mu^*>\rho \chi,
\]
and by the assumed step-size bound,
\[
0<\eta<
\frac{
2\bigl((1-\rho)\mu^*-\rho \chi\bigr)
}{
(1-\rho)\mu^* L_{G}^*+\rho \chi^2
}.
\]
Thus $\kappa_{\mathrm{rate}}-1<0$,
or equivalently,
$
\kappa_{\mathrm{rate}}<1.
$
Taking total expectation in the one-step estimate and iterating gives
\[
\mathbb{E}\|\ve{x}_t-\ve{x}^*\|^2
\le
\kappa_{\mathrm{rate}}^t
\|\ve{x}_0-\ve{x}^*\|^2.
\]
This completes the proof.
\end{proof}

\subsection{Analysis for the sample size smaller than the number of corruptions}\label{sec:small-sample}

The deterministic analysis above allows the sample size \(k\) to be chosen in
terms of the number of outliers {$s$} and the subset strong convexity threshold {$r_{sc}$},
rather than in terms of the full dataset size \(m\). This is already an improvement over earlier subsampled robust methods in the linear
setting \citep{haddock2023subsampled, HNRS20}, whose guarantees typically
required sample sizes of order \(O(m)\).

We now consider a complementary regime in which the same algorithm,
Q\(k\)L-SGD, is used with even smaller sample sizes, possibly constant. In this
regime, the deterministic counting arguments from the previous section are no
longer {enough to provide convergence guarantees}: a random sample may contain too many outliers, or even only
outliers, and the sample quantile need not reliably reflect the behavior of the
uncorrupted gradients of the full dataset. On the other hand, in practice
corrupted residuals often become comparatively larger later in the iterations
as the method approaches the optimizer, as observed in
\citep{haddock2019randomized,shvaiko2026quantile}. This observation has
motivated quantile-based robust rules on small samples in both the linear
\citep{haddock2023subsampled} and nonlinear \citep{shah2020choosing} settings, and there is also recent evidence
that logarithmic sample sizes may suffice under additional assumptions
\citep{cai2025subsample}.

Motivated by these observations, we consider a modified probabilistic model with
two additional assumptions: (a) the probability that
Algorithm~\ref{alg:qklsgd} selects a corrupted index is bounded from above,
and (b) second moments of the corrupted gradients cannot be arbitrarily large:
\begin{assump}[Small-sample probabilistic assumptions]
\label{assump:small-sample-prob}
Fix a region \(\mathcal R\subseteq\mathbb{R}^n\) containing \(\ve{x}^*\), and
suppose the iterates

under consideration remain in \(\mathcal R\). We assume
the following.

\begin{enumerate}[label=\Roman*., ref=\Roman*]
    \item \textbf{Good-component regularity.}
    For every \(i\in G\),
    \[
    f_i(\ve{x}^*)=0,
    \qquad
    \nabla f_i(\ve{x}^*)=0,
    \]
    and \(f_i\) is convex and has \(L_i\)-Lipschitz gradient.

    \item \textbf{Soft quantile separation.}
    There exists \(\pi_{q,k}\in[0,1)\) such that, for every \(\ve{x}\in\mathcal R\),
    \[
    \mathbb{P}(i_t\in O\mid \ve{x}_t=\ve{x})\le \pi_{q,k},
    \]
    where \(i_t\) is the index selected by Algorithm~\ref{alg:qklsgd}.
    This assumption requires that the selected lower-tail quantile rule
    chooses a corrupted index with probability at most \(\pi_{q,k}\).

    \item \textbf{Good-step expected curvature.} Let \(\mu_{q,k}\ge0\) be a uniform curvature lower bound such that, for
every \(\ve{x}\in\mathcal R\),
    \[
    \mathbb{E}\!\left[
    \left\langle
    \ve{x}-\ve{x}^*,
    \nabla f_{i_t}(\ve{x})-\nabla f_{i_t}(\ve{x}^*)
    \right\rangle
    \,\middle|\,
    \ve{x}_t=\ve{x},\ i_t\in G
    \right]
    \ge
    \mu_{q,k}\|\ve{x}-\ve{x}^*\|^2.
    \]

    \item \textbf{Bounded corrupted updates.}
    There exists \(G_O<\infty\) such that, for every \(\ve{x}\in\mathcal R\),
    \[
    \mathbb{E}\!\left[
    \|\nabla f_{i_t}(\ve{x})\|^2
    \,\middle|\,
    \ve{x}_t=\ve{x},\ i_t\in O
    \right]
    \le
    G_O^2.
    \]
\end{enumerate}
\end{assump}

One can obtain meaningful convergence rates in the small-sample regime under probabilistic assumptions, that is, assuming the probability of sampling an outlier is sufficiently small. We first formulate and prove our next result, {Theorem~\ref{thm:small-sample-prob}}, and then discuss how rather strong probabilistic {assumptions} naturally appear from the geometric assumptions, such as that the outlier losses are large enough compared to the uncorrupt losses.

\begin{thm}[Small-sample convergence]
\label{thm:small-sample-prob}
Suppose Assumption~\ref{assump:small-sample-prob} holds and the step size satisfies the standard condition
$0<\eta\le 2/L_{G}^*.$
Then
\[
\mathbb{E}_t[\|\ve{x}_{t+1}-\ve{x}^*\|^2]
\le
\kappa_{q,k}\|\ve{x}_t-\ve{x}^*\|^2+b_{q,k},
\]
where
\[
\kappa_{q,k}
=
1 + \pi_{q,k}
-
(1-\pi_{q,k})
\bigl(2\eta-\eta^2L_{G}^*\bigr)\mu_{q,k}
\quad \text{ and } \quad
b_{q,k}
=
2\pi_{q,k}\eta^2G_O^2.
\]
where where \(L_G^*:=\max_{i\in G}L_i\) and \(\pi_{q,k}\) is an upper bound on the probability that
Algorithm~\ref{alg:qklsgd} selects an outlier at iteration \(t\).   Therefore, whenever the probability of selecting a corrupted component is small enough, that is,
\[
(1-\pi_{q,k})
\bigl(2\eta-\eta^2L_{G}^*\bigr)\mu_{q,k}
>
\pi_{q,k},
\]
we have \(\kappa_{q,k}<1\), and the theorem implies
\[
\mathbb{E}\|\ve{x}_t-\ve{x}^*\|^2
\le
\kappa_{q,k}^t\|\ve{x}_0-\ve{x}^*\|^2
+
\frac{1-\kappa_{q,k}^t}{1-\kappa_{q,k}}
\,b_{q,k}
\]
{with $\kappa_{q,k}^t \searrow 0$ and $\frac{1-\kappa_{q,k}^t}{1-\kappa_{q,k}}
\,b_{q,k} \nearrow \frac{b_{q,k}}{1-\kappa_{q,k}}$ as $t \rightarrow \infty.$}
\end{thm}

\begin{proof}
Let us denote
\[
D_t:=\|\ve{x}_t-\ve{x}^*\|^2 \quad \text{ and } \quad p_t:=\mathbb{P}(i_t\in O\mid \ve{x}_t) \le \pi_{q,k}.
\]
We first control the update conditioned on \(i_t\in G_t\).
Since
\(\nabla f_i(\ve{x}^*)=0\) for every \(i\in G\),
\begin{align*}
D_{t+1}
&=
\|\ve{x}_t-\ve{x}^*-\eta\nabla f_{i_t}(\ve{x}_t)\|^2 \\
&=
D_t
-2\eta
\left\langle
\ve{x}_t-\ve{x}^*,
\nabla f_{i_t}(\ve{x}_t)-\nabla f_{i_t}(\ve{x}^*)
\right\rangle
+\eta^2
\|\nabla f_{i_t}(\ve{x}_t)-\nabla f_{i_t}(\ve{x}^*)\|^2.
\end{align*}
For \(i_t\in G_t\), convexity

and \(L_{i_t}\)-smoothness imply co-coercivity:
\[
\|\nabla f_{i_t}(\ve{x}_t)-\nabla f_{i_t}(\ve{x}^*)\|^2
\le
L_{i_t}
\left\langle
\ve{x}_t-\ve{x}^*,
\nabla f_{i_t}(\ve{x}_t)-\nabla f_{i_t}(\ve{x}^*)
\right\rangle.
\]
Since \(L_{i_t}\le {L_G^*}\),
we obtain
\begin{align*}
\mathbb{E}_t[D_{t+1}\mid i_t\in G_t]
&\le
D_t
-
\bigl(2\eta-\eta^2 {L_G^*}
\bigr)
\mathbb{E}_t\!\left[
\left\langle
\ve{x}_t-\ve{x}^*,
\nabla f_{i_t}(\ve{x}_t)-\nabla f_{i_t}(\ve{x}^*)
\right\rangle
\,\middle|\,
i_t\in G_t
\right].
\end{align*}
Using the good-step expected curvature assumption gives
\[
\mathbb{E}_t[D_{t+1}\mid i_t\in G_t]
\le
r_G^{\mathrm{prob}}D_t,
\]
where
\[
r_G^{\mathrm{prob}}
:=
1-\bigl(2\eta-\eta^2L_{G}^*\bigr)\mu_{q,k}.
\]

Next, condition on \(i_t\in O_t\). By Cauchy--Schwarz, the bounded corrupted
update assumption, {and Jensen's inequality,}
\begin{align*}
\mathbb{E}_t[D_{t+1}\mid i_t\in O_t]
&=
\mathbb{E}_t\!\left[
\|\ve{x}_t-\ve{x}^*-\eta\nabla f_{i_t}(\ve{x}_t)\|^2
\,\middle|\,
i_t\in O_t
\right] \\
&\le
D_t
+
2\eta \|\ve{x}_t-\ve{x}^*\|
\mathbb{E}_t\!\left[
\|\nabla f_{i_t}(\ve{x}_t)\|
\,\middle|\,
i_t\in O_t
\right]
+
\eta^2G_O^2 \\
&\le
D_t
+
2\eta G_O\sqrt{D_t}
+
\eta^2G_O^2.
\end{align*}
Using \(2ab\le a^2+b^2\) with \(a=\sqrt{D_t}\) and \(b=\eta G_O\), we obtain
\[
2\eta G_O\sqrt{D_t}
\le
D_t+\eta^2G_O^2.
\]
Therefore
\[
\mathbb{E}_t[D_{t+1}\mid i_t\in O_t]
\le
2D_t+2\eta^2G_O^2.
\]

Combining the two conditional estimates yields
\begin{align*}
\mathbb{E}_t[D_{t+1}]
&=
(1-p_t)\mathbb{E}_t[D_{t+1}\mid i_t\in G_t]
+
p_t\mathbb{E}_t[D_{t+1}\mid i_t\in O_t] \\
&\le
(1-p_t)r_G^{\mathrm{prob}}D_t
+
2p_t D_t
+
2p_t\eta^2G_O^2.
\end{align*}
Since \(0<\eta\le 2/{L_G^*}\),
we have \(r_G^{\mathrm{prob}}\le 1<2\).
Thus the right-hand side is increasing in \(p_t\). Using \(p_t\le \pi_{q,k}\),
we obtain
\[
\mathbb{E}_t[D_{t+1}]
\le
\kappa_{q,k}D_t+b_{q,k}.
\]
If \(\kappa_{q,k}<1\), taking total expectation and iterating this
recursion yields
\[
\mathbb E D_t
\le
\kappa_{q,k}^tD_0
+
\sum_{j=0}^{t-1}\kappa_{q,k}^j b_{q,k}
=
\kappa_{q,k}^tD_0
+
\frac{1-\kappa_{q,k}^t}{1-\kappa_{q,k}}b_{q,k}.
\]
\end{proof}

\begin{rem}
The additive term in Theorem~\ref{thm:small-sample-prob} comes only from the
possibility that a selected corrupted update may have nonvanishing magnitude
even near \(\ve{x}^*\). In more benign local regimes, the selected outliers may
also become weak near \(\ve{x}^*\). For example, if
\[
\mathbb{E}\!\left[
\|\nabla f_{i_t}(\ve{x})\|^2
\,\middle|\,
\ve{x}_t=\ve{x},\ i_t\in O_t
\right]
=
o(\|\ve{x}-\ve{x}^*\|^2)
\qquad\text{as }\ve{x}\to\ve{x}^*,
\]
then the corrupted-step contribution can be absorbed into the multiplicative
term and the convergence result becomes an exact linear contraction, that is, $\mathbb{E}\|\ve{x}_t-\ve{x}^*\|^2 \to 0$ as $t \to \infty$.
\end{rem}

Unlike the deterministic large-sample regime, the accepted good set \(G_t\)
is random and need not exceed the subset strong-convexity threshold at every
iteration. Additionally, the expected-curvature constant \(\mu_{q,k}\) in
Assumption~\ref{assump:small-sample-prob}.III can be difficult to verify
directly, and its dependence on \(q\) and \(k\) is not transparent. We next
use the same subset strong-convexity structure as in
Section~\ref{sec:large-sample} to construct a monotone curvature profile that
yields an explicit admissible value of \(\mu_{q,k}\).

\begin{prop}[Subset curvature implies good-step expected curvature]
\label{prop:subset-curvature-small-sample}

Suppose Assumption~\ref{assump:small-sample-prob}.I--II holds. Recall that
\(G\subseteq[m]\) is the fixed set of good component indices and that
\(G_t=B_t\cap G\) is the set of accepted good indices at iteration \(t\). Assume that there exists
$r_{\mathrm{sc}}\in\{1,\ldots,|G|\}$
such that, for every subset \(S\subseteq G\) with
$|S|\ge r_{\mathrm{sc}},$
the averaged objective \(F_S\) is \(\mu_S\)-strongly convex. Define the
subset-curvature profile
\[
\underline{\mu}(h)
:=
\begin{cases}
0,
& 1\le h<r_{\mathrm{sc}},\\[2mm]
\displaystyle
\min_{\substack{S\subseteq G\\ |S|\ge h}}\mu_S,
& r_{\mathrm{sc}}\le h\le |G|.
\end{cases}
\]
Then \(\underline{\mu}(h)\) is nondecreasing in \(h\) and Assumption~\ref{assump:small-sample-prob}.III holds with
\[
\mu_{q,k}
:=
\inf_{\ve{x}\in\mathcal R}
\mathbb{E}\!\left[
\underline{\mu}(|G_t|)
\,\middle|\,
\ve{x}_t=\ve{x},\ i_t\in G
\right].
\]
\end{prop}
\begin{proof}
The monotonicity of \(\underline{\mu}(h)\) is immediate from the definition as minimizing over the smaller family of sets can only increase the value of the minimum. Fix \(\ve{x}\in\mathcal R\). Conditional on \(\ve{x}_t=\ve{x}\), on the
realized sample \(S_t\), and on the event \(i_t\in G_t\), the algorithm selects
\(i_t\) uniformly from the good accepted set \(G_t\). Therefore
\[
\mathbb{E}\!\left[
\nabla f_{i_t}(\ve{x})
\,\middle|\,
\ve{x}_t=\ve{x},\,S_t,\,i_t\in G_t
\right]
=
\nabla F_{G_t}(\ve{x}),
\]
and similarly, since \(\nabla f_i(\ve{x}^*)=0\) for all \(i\in G\),
\[
\mathbb{E}\!\left[
\nabla f_{i_t}(\ve{x}^*)
\,\middle|\,
\ve{x}_t=\ve{x},\,S_t,\,i_t\in G_t
\right]
=
\nabla F_{G_t}(\ve{x}^*)
=
0.
\]
Hence
\begin{align*}
&\mathbb{E}\!\left[
\left\langle
\ve{x}-\ve{x}^*,
\nabla f_{i_t}(\ve{x})-\nabla f_{i_t}(\ve{x}^*)
\right\rangle
\,\middle|\,
\ve{x}_t=\ve{x},\,S_t,\,i_t\in G_t
\right] =
\left\langle
\ve{x}-\ve{x}^*,
\nabla F_{G_t}(\ve{x})-\nabla F_{G_t}(\ve{x}^*)
\right\rangle.
\end{align*}

If \(|G_t|\ge r_{\mathrm{sc}}\), then the subset strong-convexity
assumption gives
\[
\left\langle
\ve{x}-\ve{x}^*,
\nabla F_{G_t}(\ve{x})-\nabla F_{G_t}(\ve{x}^*)
\right\rangle
\ge
\mu_{G_t}\|\ve{x}-\ve{x}^*\|^2
\ge
\underline{\mu}(|G_t|)\|\ve{x}-\ve{x}^*\|^2.
\]
If \(|G_t|<r_{\mathrm{sc}}\), then \(F_{G_t}\) is convex, since it is
an average of convex good components. Hence
\[
\left\langle
\ve{x}-\ve{x}^*,
\nabla F_{G_t}(\ve{x})-\nabla F_{G_t}(\ve{x}^*)
\right\rangle
\ge 0
=
\underline{\mu}(|G_t|)\|\ve{x}-\ve{x}^*\|^2.
\]
Taking conditional expectation over \(S_t\) and the infimum over
\(\ve{x}\in\mathcal R\) gives the claim.
\end{proof}

\begin{rem}
The subset-curvature profile \(\underline{\mu}(h)\) provides a monotone
lower bound on the curvature associated with a realized accepted good set.
Thus, Assumption~\ref{assump:small-sample-prob}.III holds with
\[
\mu_{q,k}
:=
\inf_{\ve{x}\in\mathcal R}
\mathbb{E}\!\left[
\underline{\mu}(|G_t|)
\,\middle|\,
\ve{x}_t=\ve{x},\ i_t\in G_t
\right].
\]

In the deterministic large-sample regime, when
$|G_t|\ge\tau\ge r_{\mathrm{sc}},$
and hence
$
\underline{\mu}(|G_t|)
\ge
\underline{\mu}(\tau)
=
\mu^*.$ As a result, $\mu_{q,k}\ge\mu^*,$
so we recover the large-sample curvature lower bound from the section above.  In the
general small-sample regime, \(|G_t|\) may be smaller than \(\tau\), and even smaller than
\(r_{\mathrm{sc}}\). On these iterations, Theorem~\ref{thm:small-sample-prob} guarantees no contraction even for small $\pi_{q,k}$.
\end{rem}

\subsubsection{Interpretation under large separation regime}
\label{sec:small-sample-hard-separation}

We now specialize the small-sample analysis to a simplified ``warm start and large outliers"
regime, in which every good loss lies strictly below every corrupted loss.

{\begin{assump}[Idealized Model]
    For every \(\ve{x}\in\mathcal R\),
\begin{equation}\label{eq:hard-separation}
f_i(\ve{x})<f_j(\ve{x})
\qquad
\text{for all } i\in G,\ j\in O.
\end{equation}
\label{assump:hard-separation}
\end{assump}}
{The remainder of this section discusses the implications of Assumption~\ref{assump:hard-separation} with respect to the small-sample regime. Proposition~\ref{prop:hard-pi} provides an upper bound on the probability of selecting an outlier, $\pi_{q,k}$. Corollary~\ref{cor:mu-lower-hard} relates the good-step curvature parameter to subset strong convexity. Lastly, we end this section with a discussion on the consequences of this idealized corruption setting. }

\begin{prop}[Hard-separation upper bound on the outlier-selection probability]
\label{prop:hard-pi}
Suppose Assumption~\ref{assump:hard-separation} holds. Let \({S_t =} \{J_{t,1},\dots,J_{t,k}\}\) denote the
sampled indices in iteration $t$ and define
\begin{equation}\label{eq:pg}
p_G:=\frac{|G|}{m}=\frac{m-s}{m},
\qquad
H:=|S_t\cap G|.
\end{equation}
Then, for every \(\ve{x}\in\mathcal R\) and every \(q<p_G\),
\[
\mathbb{P}(i_t\in O\mid \ve{x}_t=\ve{x})
\le
\mathbb{P}(H<qk)
\le
\exp\left(-2k(p_G-q)^2\right).
\]
In particular, Assumption~\ref{assump:small-sample-prob}.II holds with
$
\pi_{q,k}
=
\exp\left(-2k(p_G-q)^2\right).
$
\end{prop}

\begin{proof}
Under Assumption~\ref{assump:hard-separation}, if at least \(qk\) good
indices are sampled, then the \(qk\) smallest sampled losses all correspond
to good components. In that case, the empirical quantile threshold is the
loss of a good component.
Therefore, the algorithm can select an outlier only if \(H<qk\), and hence
\[
\mathbb{P}(i_t\in O\mid \ve{x}_t=\ve{x})
\le
\mathbb{P}(H<qk).
\]

Define the outlier indicators
\[
Y_j:=
\mathbf{1}_{\{J_{t,j}\in O\}},
\qquad j=1,\ldots,k.
\]
The variables \(Y_1,\ldots,Y_k\) form a sample without replacement from a finite
set consisting of \(s\) ones and \(m-s\) zeros, with the  mean $
\frac{s}{m}=1-p_G$,
also observe that
\[
\sum_{j=1}^k Y_j=k-H.
\]
Hoeffding's inequality for sampling without replacement (e.g., \cite[Proposition~1.2]{bardenet2015concentration}, based on the classical techniques from~\cite{hoeffding1963probability}) yields, for \(q<p_G\),
\begin{align*}
\mathbb{P}(H<qk)
&=
\mathbb{P}\left[
\frac{1}{k}\sum_{j=1}^k Y_j>1-q
\right]\\
&\le
\mathbb{P}\left[
\frac{1}{k}\sum_{j=1}^k Y_j-(1-p_G)
\ge p_G-q
\right] \le
\exp\left(-2k(p_G-q)^2\right).
\end{align*}
\end{proof}

We next show how the subset strong-convexity profile of the good components
provides an admissible choice of the good-step curvature constant
\(\mu_{q,k}\) in Assumption~\ref{assump:small-sample-prob}.III.

\begin{cor}[Hard-separation verification of good-step curvature]
\label{cor:mu-lower-hard}
Suppose Assumption~\ref{assump:hard-separation} and the hypotheses of
Proposition~\ref{prop:subset-curvature-small-sample} hold, and let \(H\)
be as defined in Proposition~\ref{prop:hard-pi}. Then, Assumption~\ref{assump:small-sample-prob}.III holds with
\[
\mu_{q,k}
:=
\underline{\mu}(qk)
\left(1-\exp\left(-2k(p_G-q)^2\right)\right),
\]
if $qk \ge r_sc$ and $q < p_G$.
\end{cor}

\begin{proof}
Let
\[
\mathcal{A}:=\{H\ge qk\}.
\]
Under Assumption~\ref{assump:hard-separation}, on the event \(A\) all
accepted indices are good and
$|G_t|\ge qk.$
Therefore,
\[
\underline{\mu}(|G_t|)
\ge
\underline{\mu}(qk)
\qquad\text{on }\mathcal{A}.
\]
Moreover, \(\mathcal{A}\subseteq\{i_t\in G_t\}\). Hence, for every
\(\ve{x}\in\mathcal R\),
\begin{align*}
\mathbb{E}\!\left[
\underline{\mu}(|G_t|)
\,\middle|\,
\ve{x}_t=\ve{x},\,i_t\in G_t
\right] &\ge
\underline{\mu}(qk)\,
\mathbb{P}\!\left(
\mathcal{A}
\,\middle|\,
\ve{x}_t=\ve{x},\,i_t\in G_t
\right) \\
&\qquad=
\underline{\mu}(qk)\,
\frac{
\mathbb{P}(\mathcal{A}\mid\ve{x}_t=\ve{x})
}{
\mathbb{P}(i_t\in G_t\mid\ve{x}_t=\ve{x})
} \ge
\underline{\mu}(qk)\,
\mathbb{P}(\mathcal{A}).
\end{align*}
Here the final inequality uses
\(\mathbb{P}(i_t\in G_t\mid\ve{x}_t=\ve{x})\le1\), and the fresh sample
\(S_t\) is independent of \(\ve{x}_t\). Taking the infimum over \(\ve{x}\in\mathcal R\) gives
\[
\mu_{q,k}
\ge
\underline{\mu}(qk)\,
\mathbb{P}(H\ge qk).
\]
The claim follows from Proposition~\ref{prop:hard-pi}.
\end{proof}

Proposition~\ref{prop:hard-pi} and
Corollary~\ref{cor:mu-lower-hard} immediately imply

\begin{cor}[Explicit rate and error floor under hard separation]
\label{cor:hard-separation-rate}
Suppose Assumption~\ref{assump:hard-separation} and the hypotheses of
Proposition~\ref{prop:subset-curvature-small-sample} hold. Assume
\[
q<p_G:=\frac{m-s}{m},
\qquad
qk\ge r_{\mathrm{sc}},
\]
and define
\[
\varepsilon_{q,k}
:=
\exp\left(-2k(p_G-q)^2\right).
\]
Then Theorem~\ref{thm:small-sample-prob} applies with
\[
\pi_{q,k}
=
\varepsilon_{q,k},
\qquad
\mu_{q,k}
=
\underline{\mu}(qk)(1-\varepsilon_{q,k}).
\]
Therefore,
\[
\mathbb E_t\|\ve{x}_{t+1}-\ve{x}^*\|^2
\le
\overline{\kappa}_{q,k}\|\ve{x}_t-\ve{x}^*\|^2
+
\overline b_{q,k},
\]
where
\[
\overline{\kappa}_{q,k}
=
1+\varepsilon_{q,k}
-
(1-\varepsilon_{q,k})^2
\bigl(2\eta-\eta^2L_G^*\bigr)
\underline{\mu}(qk),
\]
and
\[
\overline b_{q,k}
=
2\eta^2G_O^2\varepsilon_{q,k}.
\]
If \(\overline{\kappa}_{q,k}<1\), then
\[
\limsup_{t\to\infty}
\mathbb E\|\ve{x}_t-\ve{x}^*\|^2
\le
\frac{\overline b_{q,k}}
{1-\overline{\kappa}_{q,k}}.
\]
\end{cor}

Choosing \(q=1/k\) recovers min-\(k\)-loss SGD
\citep{shah2020choosing}. This prior work derives a state-dependent one-step
recursion for this rule, in which the rank-dependent sampling probabilities
control both the good-step contraction and the outlier contribution. That
analysis observes that minimum-loss selection may reduce the outlier term
while also slowing contraction. Corollary~\ref{cor:hard-separation-rate} gives a different type of
guarantee: under hard separation, it replaces these state-dependent
quantities by explicit, uniform bounds in terms of
\(\varepsilon_{q,k}\) and \(\underline{\mu}(qk)\). This yields a
multi-iteration convergence bound and an explicit limiting-error bound for
any choice of quantile, while making the tradeoff between outlier avoidance
and good-step progress directly interpretable.

Corollary~\ref{cor:hard-separation-rate} illustrates the dependence of the convergence of Algorithm~\ref{alg:qklsgd} on the basic sample size $k$ and quantile $q$ parameters:

\paragraph{Effect of increasing \(k\) at fixed \(q<p_G\).}
As \(k\) increases, \(\varepsilon_{q,k}\) decreases exponentially, while
\(\underline{\mu}(qk)\) is nondecreasing. Hence
\(\overline{\kappa}_{q,k}\) and \(\overline b_{q,k}\) are nonincreasing in
\(k\). Whenever \(\overline{\kappa}_{q,k}<1\), the corresponding
limiting-error bound is also nonincreasing. This illustrates that larger samples improve the
guaranteed reliability and convergence behavior, at the expense of more loss
evaluations per iteration.

\paragraph{Effect of varying \(q\) at fixed \(k\).}
Decreasing \(q\) lowers the outlier-selection and additive-error bounds, but
may also lower the curvature factor \(\underline{\mu}(qk)\). Therefore, its
effect on the multiplicative rate and limiting-error bound is governed by the
balance between robustness and good-step progress. This illustrates that Smaller quantiles produce
more robust loss selection, but may also result in slower progress on selected good
steps.

\section{Experiments}\label{sec:numerics}

{In this section, we present numerical experiments} across several objective classes to demonstrate how quantile-based robust SGD remains effective across increasingly general
finite-sum optimization problems. The most basic instance of such problem classes is the corrupted linear system setting, which was studied in earlier work on quantile-based Kaczmarz methods, including~\cite{haddock2019randomized,haddock2023subsampled}. Here we focus on finite-sum objectives that move beyond the linear setting. Code for all numerical experiments are publicly available at \href{https://github.com/jamiehadd/QuantileSGD}{https://github.com/jamiehadd/QuantileSGD}.

\begin{table}[H]
\caption{The table records model properties, as per Assumption~\ref{assump:sgd}: I = good-component regularity, II = outlier nonnegativity, III = component smoothness, and IV = subset strong convexity on good accepted sets. A checkmark in column IV indicates the
existence of positive subset curvature. For the nonsmooth hinge loss, strong convexity is understood in the subgradient sense.}\label{tab:assumptions}\centering \small \setlength{\tabcolsep}{6pt} \renewcommand{\arraystretch}{1.15}\begin{tabularx}{\textwidth}{@{}
>{\raggedright\arraybackslash}p{0.26\textwidth}
>{\raggedright\arraybackslash}X
*{4}{>{\centering\arraybackslash}p{0.055\textwidth}}
@{}}
\toprule
\multirow{2}{*}{Section}
& \multirow{2}{*}{Model}
& \multicolumn{4}{c}{Assumption~\ref{assump:sgd}} \\
\cmidrule(l){3-6}
& & I & II & III & IV \\
\midrule
Subsection~\ref{subsec:noiseless_poly_regression}
& Noiseless polynomial regression
& \cmark & \cmark & \cmark & \cmark \\

Subsection~\ref{subsec:poly_regression}
& Noisy polynomial regression
& \xmark & \cmark & \cmark & \cmark \\

Subsection~\ref{subsec:logistic_regression}
& Regularized logistic regression
& \xmark & \cmark & \cmark & \cmark \\

Subsection~\ref{subsec:hinge_loss}
& Regularized hinge loss
& \xmark & \cmark & \xmark & \cmark \\
\bottomrule
\end{tabularx}
\end{table}

We begin with polynomial regression, first in a noiseless interpolating regime that is closest to the assumptions of Theorem~\ref{thm:main}, and then consider noisy observations. We next consider regularized logistic regression, which has a smooth nonlinear loss, and regularized hinge loss, which provides a nonsmooth stress test.

Table~\ref{tab:assumptions} reports model-level properties relevant to our
analyses. Note that a checkmark for subset strong convexity means that
sufficiently large good subsets admit a positive strong-convexity; it
does not verify that it is large enough to satisfy the quantitative
contraction conditions of Theorem~\ref{thm:main} (and in most experimental settings it is not). Indeed, except for the
noiseless polynomial-regression model, the experiments intentionally consider
settings in which one or more assumptions of the deterministic theory fail,
to stress-test Q\(k\)L-SGD on increasingly challenging
problems. The probabilistic assumptions in
Assumption~\ref{assump:small-sample-prob} are not included in the table,
because they depend on \(q\), \(k\), the corruption model, and the resulting
algorithmic trajectory rather than on the component-loss model alone.

For each objective class, we compare QkL-SGD with standard SGD and
min-$k$-loss SGD (MkL-SGD), where MkL-SGD denotes the rule that selects the
sampled component with the smallest loss (see, e.g., \cite{shah2020choosing}). We present first the {large-sample} experiments in Sections~\ref{subsec:noiseless_poly_regression}--\ref{subsec:hinge_loss}, where we
take $k=m$ for both QkL-SGD and MkL-SGD. These comparisons isolate the effect
of the loss-based selection rule, and they are plotted per iteration.

Then, in {Section~\ref{subsec:vary_k}}, we
study the effect of the sample size and compare the methods as
functions of the total number of component-loss evaluations {across all iterations}. {In the fixed-quantile experiments, we use the oracle choice \(q=1-\beta\)
to isolate the behavior of the loss-filtering rule. The varying-\(q\)
experiments then assess sensitivity to misspecification of this choice.}

\subsection{Noiseless Polynomial Regression}\label{subsec:noiseless_poly_regression}

We begin by considering a noiseless polynomial regression task, a least-squares problem with zero optimal loss.  We aim to illustrate the robustness of QkL-SGD to a small fraction $\beta$ of corrupted regression response values.  In each experiment, we generate ideal noiseless polynomial regression sample data and introduce significant corruption into a small number of outlier response values.  We note that the typical polynomial regression task is not noiseless and thus does not have zero optimal uncorrupted loss; we consider such a task in the next subsection, Subsection~\ref{subsec:poly_regression}.  However, we begin with the noiseless case since this task satisfies all {model-level} assumptions {in Table~\ref{tab:assumptions} and } Assumption~\ref{assump:sgd}.

We consider the problem of fitting a polynomial of degree $3$ to data $\{(x_i, y_i)\}_{i=1}^m$, where the data is generated according to
\[
y_i = p(x_i), \quad \text{with } p(x) = 1 + 0.5x - 0.2x^2 + 0.05x^3,
\]
the $x_i$ are uniformly distributed over the interval $[-3,3)$, and $m=200$.
We will call $p(x)$ the `clean polynomial' and note that $p(x) = \ve{d}^\top \ve{x}^\star_{\rm clean}$ where $\ve{x}^\star_{\rm clean} = [1, 0.5, -0.2, 0.05]^\top$ and $\ve{d} = [1, x, x^2, x^3]^\top$. In the plots below, we visualize $p(x)$ and label this ``Clean".
To introduce outliers, we randomly select a fraction $\beta$ of the response variables $y_i$ and add large corruptions sampled from $\mathcal{N}(0,100)$; this set is the outlier set $O$ and the remaining uncorrupted set is $G$.

The optimal regression coefficients $\ve{x}^\star_{\rm clean} \in \mathbb{R}^{4}$ corresponding to clean polynomial $p(x)$ are estimated by minimizing the empirical least-squares objective
\begin{equation}
F(\ve{x}) = \frac{1}{m} \sum_{i=1}^m f_i(\ve{x}), \qquad f_i(\ve{x}) = \tfrac{1}{2}(y_i - \ve{d}_i^\top \ve{x})^2, \label{eq:poly_regression}
\end{equation}
with respect to the parameter vector $\ve{x}$, where $\ve{d}_i = [1, x_i, x_i^2, x_i^3]^\top$ denotes the feature vector corresponding to sample $x_i$. Any coefficient parameter vector $\ve{x} \in \mathbb{R}^4$ defines a polynomial as $\ve{x}^\top \ve{d}$ where $\ve{d} = [1, x, x^2, x^3]^\top.$  We may visualize any such univariate polynomial by plotting it alongside the clean polynomial $p(x)$ and the samples; we do so in the right plot of the experiments below.

The objective in each experiment is to learn a parameter vector $\ve{x} \in \mathbb{R}^4$ that defines a polynomial near to clean polynomial $p(x)$, and thus fitting the samples whose response values have not been corrupted.  In all experiments, SGD, QkL-SGD, and MkL-SGD are applied to objective $F(x)$ using step size $\eta = 0.001$.

\paragraph{Assumptions for noiseless polynomial regression loss}\label{subsubsec:noiselesspolyregression_assumptions}

The polynomial regression components~\eqref{eq:poly_regression} are nonnegative, convex, and smooth with $$L_i \le \|\ve{d}_i\|_2^2.$$  Here, for $r_{\mathrm{sc}} = 4$ and every subset $S \subseteq G$ with $|S| \ge r_{\mathrm{sc}} = 4$, we have that the least-squares objective $F_S(x)$ is strongly convex, since in our case where each $x_i$ is distinct, the design matrix $D_S$ of subset $S$ has full column rank.  Because the uncorrupted data is assumed to be noiseless, we have $f_i(\ve{x}^\star_{\rm clean}) = 0$ and
     $\nabla f_i(\ve{x}^\star_{\rm clean}) = 0$ for $i \in G$.  Thus, Assumptions~\ref{assump:sgd}.\ref{ass:good_set_regularity}, \ref{assump:sgd}.\ref{ass:outlier_nonnegativity}, \ref{assump:sgd}.\ref{ass:component_smoothness}, and~\ref{assump:sgd}.\ref{ass:subset_strong_convexity} all hold for this model.

\subsubsection{Various outlier settings}

In this section, we illustrate the behavior of QkL-SGD on noiseless polynomial regression models under a small to moderate fraction of outliers $\beta$.  Here we apply QkL-SGD with $q = 1 - \beta$.

\paragraph{Few outliers (\texorpdfstring{$\beta = 0.05$, $q = 0.95$, $k = m$}{beta = 0.05, q = 0.95, k = m})}

In our first experiment, whose results are presented in Figure~\ref{fig:figure-01a}, 5\% of the sample response values $y_i$ are corrupted by large additive corruption sampled from $\mathcal{N}(0,100)$. We compare QkL-SGD with $q =0.95$ and $k = m$ to SGD and MkL-SGD with $k = m$.  We see QkL-SGD converges towards the zero optimal relative loss
and parameter vector faster and to a lower final relative error
than SGD.  MkL-SGD stalls nearly immediately with respect to both loss and error, as the algorithmic selection of the component with minimum loss begins near zero and thus the update is nearly negligible. We plot the polynomials learned by SGD, QkL-SGD, and MkL-SGD; see the right plot. The learned polynomial from QkL-SGD closely matches the true function even in regions influenced by outliers, while the SGD fit is visibly distorted by the corrupted points.

\paragraph{Moderate outliers (\texorpdfstring{$\beta = 0.1$, $q = 0.9$, $k = m$}{beta = 0.1, q = 0.9, k = m})}

In our next experiment, presented in Figure~\ref{fig:figure-01b}, 10\% of the sample response values $y_i$ are corrupted by large additive corruption sampled from $\mathcal{N}(0,100)$.  We see that QkL-SGD with $q = 0.9$ and $k = m$ again outperforms SGD with respect to both relative loss
and relative error,
and MkL-SGD again stagnates early.  We visualize the learned polynomials in the right plot; the polynomial learned by QkL-SGD is nearly identical to the true polynomial $p(x)$, while the SGD-learned polynomial suffers due to the corrupted samples.

\begin{figure}
    \begin{subfigure}{\textwidth}
    \centering
    \includegraphics[width=0.3\textwidth]{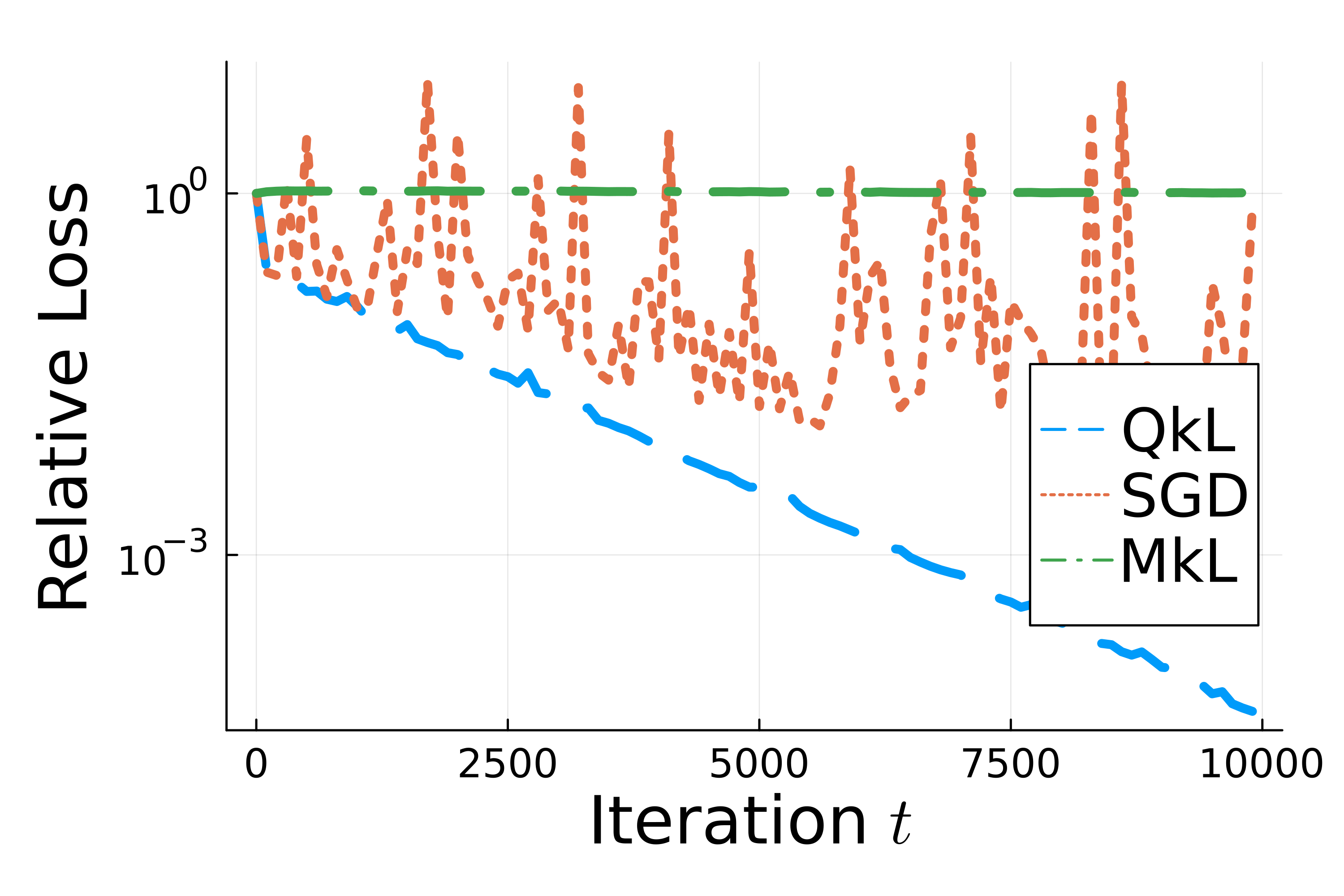}
\quad\includegraphics[width=0.3\textwidth]{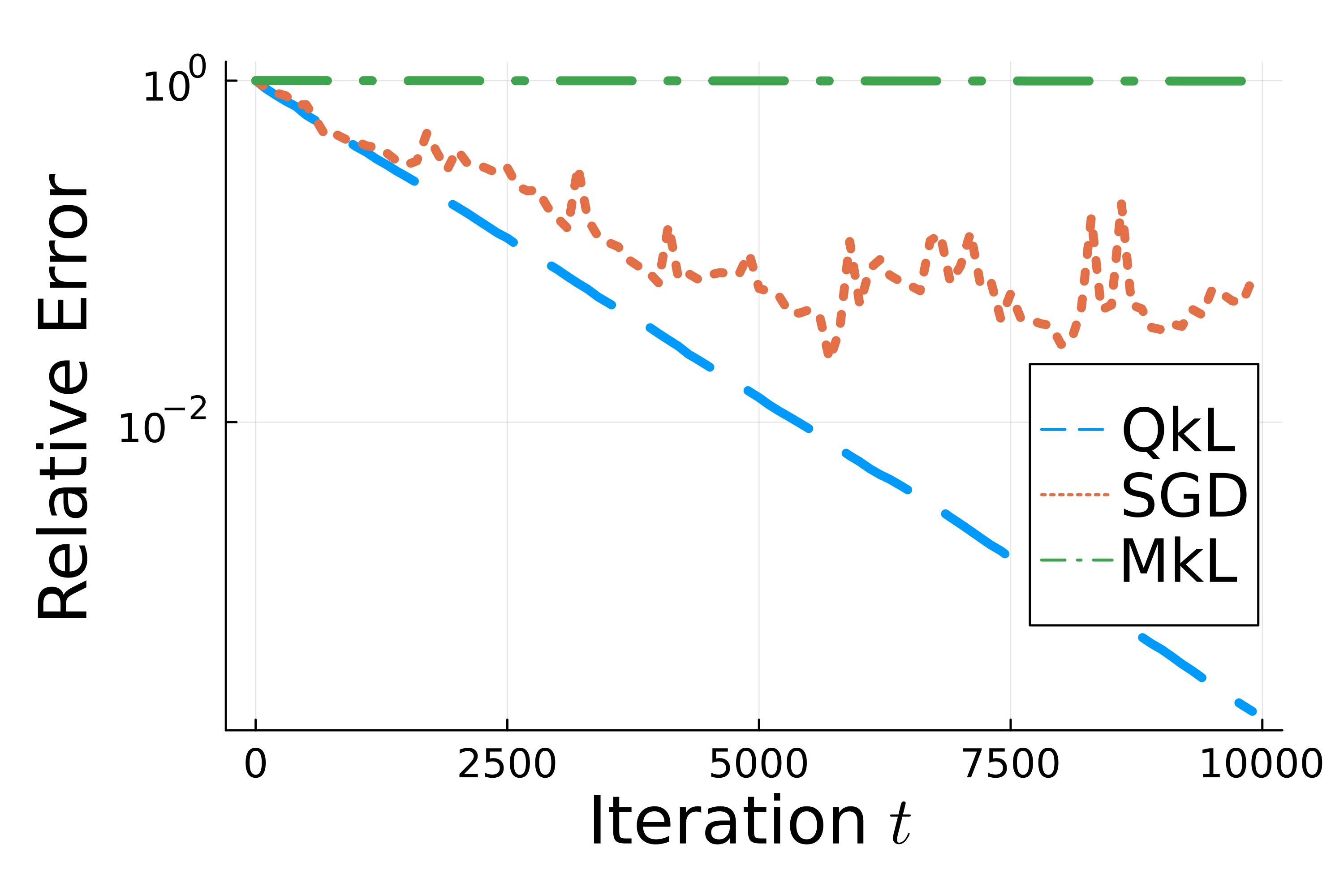}
\hfill\includegraphics[width=0.32\textwidth]{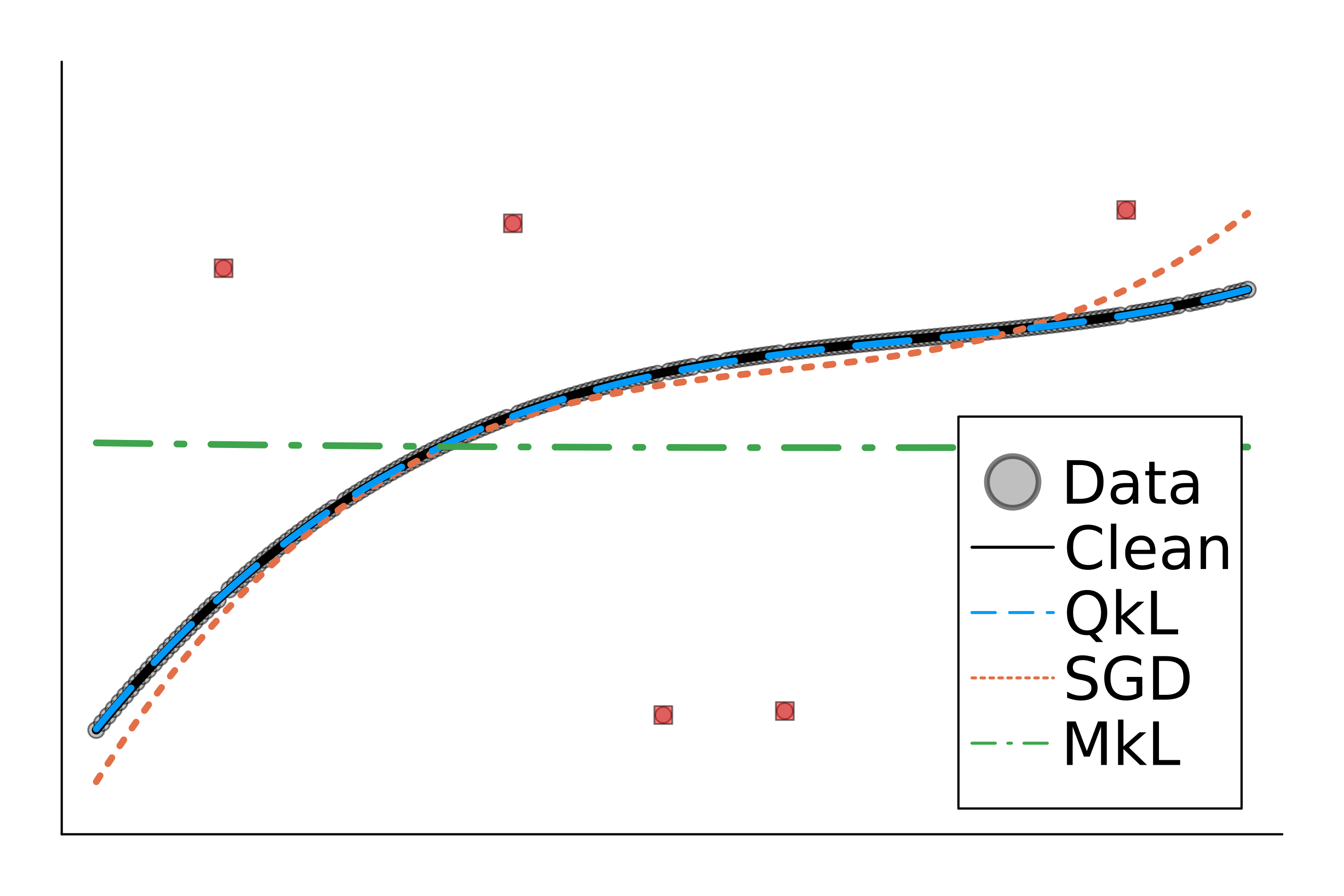}
    \caption{$\beta = 0.05$ response corruption.}
    \label{fig:figure-01a}
\end{subfigure}

\begin{subfigure}{\textwidth}
    \centering
    \includegraphics[width=0.3\textwidth]{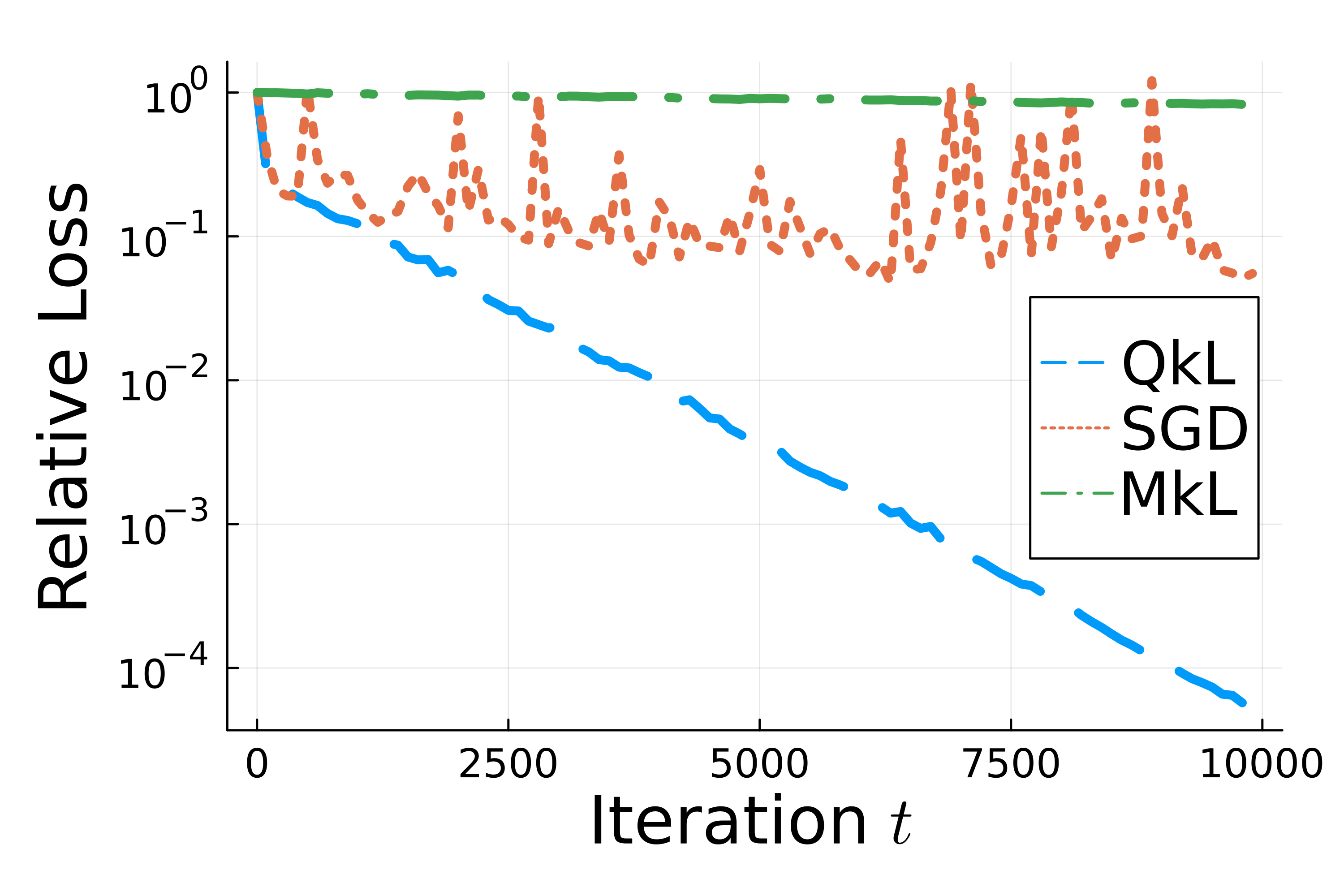}
\quad\includegraphics[width=0.3\textwidth]{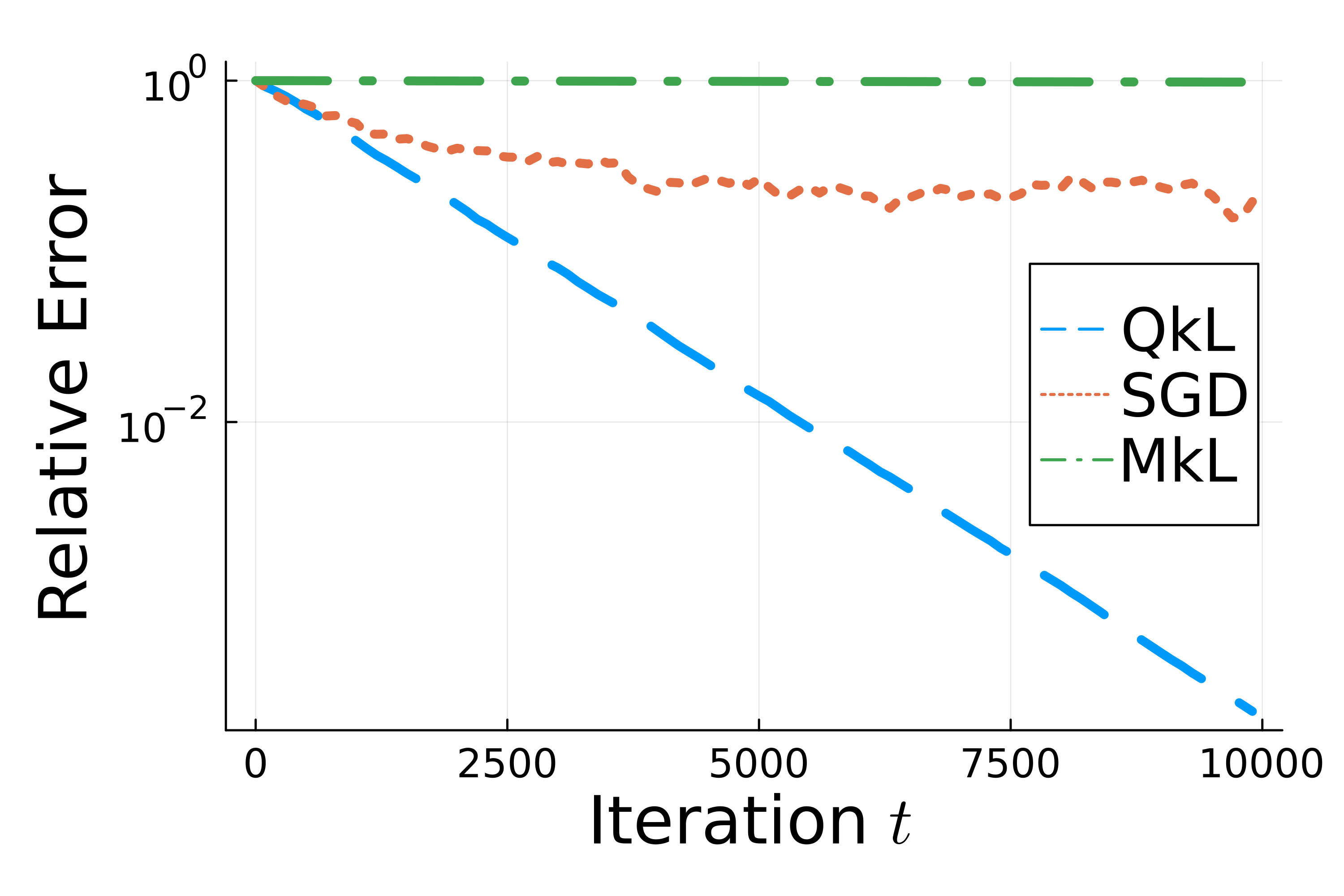}
\hfill\includegraphics[width=0.32\textwidth]{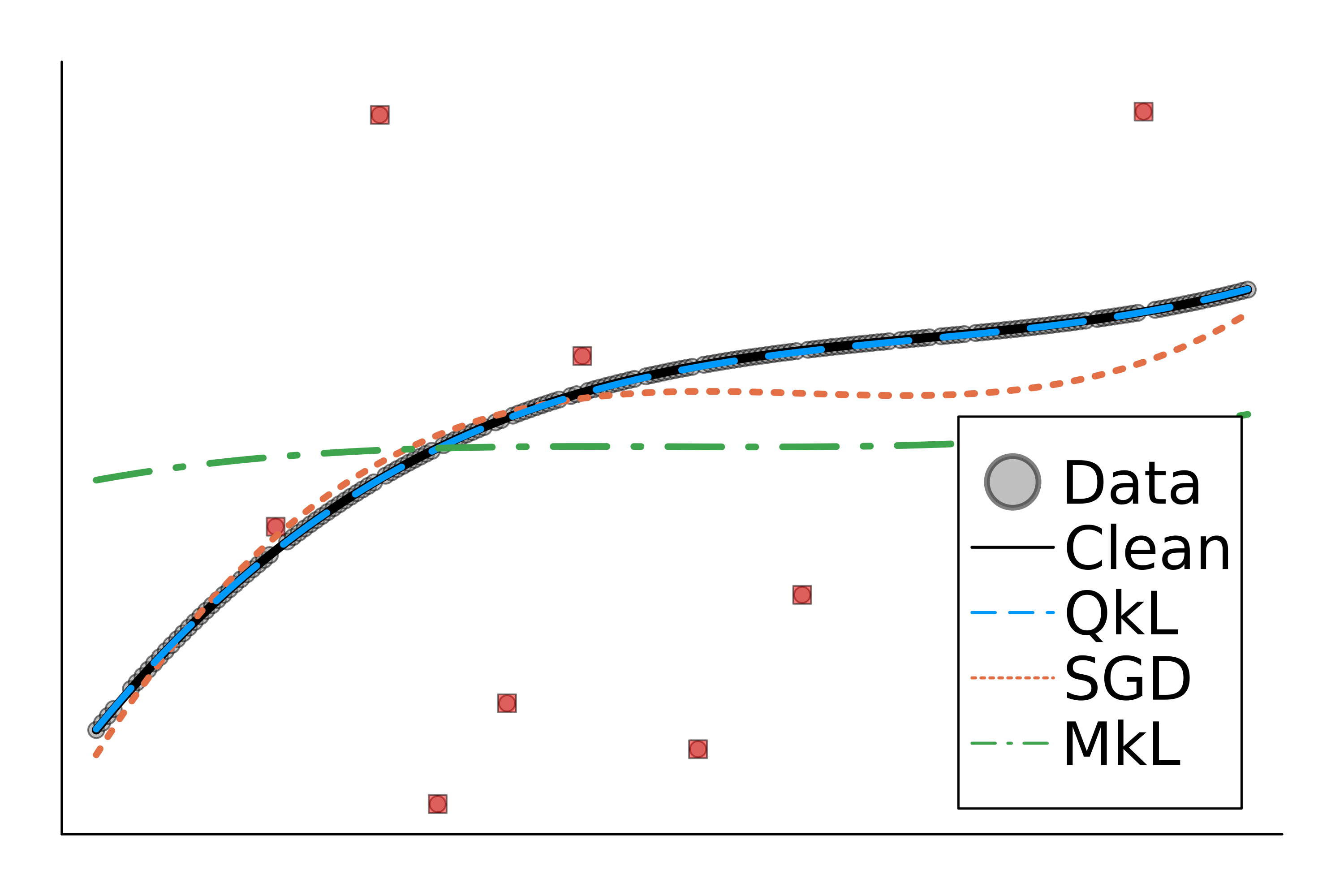}
    \caption{$\beta = 0.1$ response corruption.
    }
    \label{fig:figure-01b}
\end{subfigure}

    \caption{QkL-SGD with $(k,q)=(m,1-\beta)$ and MkL-SGD with $k=m$ applied to polynomial regression loss~\eqref{eq:poly_regression}. (Left) Relative good-set loss $F_G(\ve x_t)/F_G(\ve x_0)$ per iteration. (Center) Relative squared parameter error $\|\ve x_t-\ve x^\star_{\rm clean}\|_2^2/\|\ve x_0-\ve x^\star_{\rm clean}\|_2^2$. (Right) Learned polynomials at the final iterate, together with the clean polynomial $p(x)$. Uncorrupted data samples, visualized as circles, are noiseless; corrupted data samples are visualized as squares.
    }
\end{figure}

\subsubsection{Varying $q$ with fixed corruption fraction $\beta$}

In our final experiment with the noiseless polynomial regression model, we explore how QkL-SGD behaves with different choices of quantile $q$ under a fixed corruption fraction $\beta$.  In Figure~\ref{fig:figure-02}, we plot the final relative loss (left plot) and final relative error (right plot) after 10000 iterations of QkL-SGD with $q \in \{0.05, 0.1, \cdots, 1.0\}$ and $k = m$.  We apply QkL-SGD to the noiseless polynomial regression loss~\eqref{eq:poly_regression} applied to the data model described above where $\beta \in \{0.1, 0.2, 0.3, 0.4\}$ fraction of data is corrupted. Relative loss and relative errors are averaged over 100 independent trials of 10000 iterations of QkL-SGD.  Here, we see that the quantile value $q$ that achieves the lowest relative loss and relative error is near $1 - \beta$.  For larger corruption fractions $\beta$, quantiles $q$ very slightly larger than $1-\beta$ can be optimal. Lowest values of relative error and loss are similar for each fraction $\beta$, illustrating that QkL-SGD provides robustness to even significant amounts of data corruption.

\begin{figure}
    \includegraphics[width=0.45\textwidth]{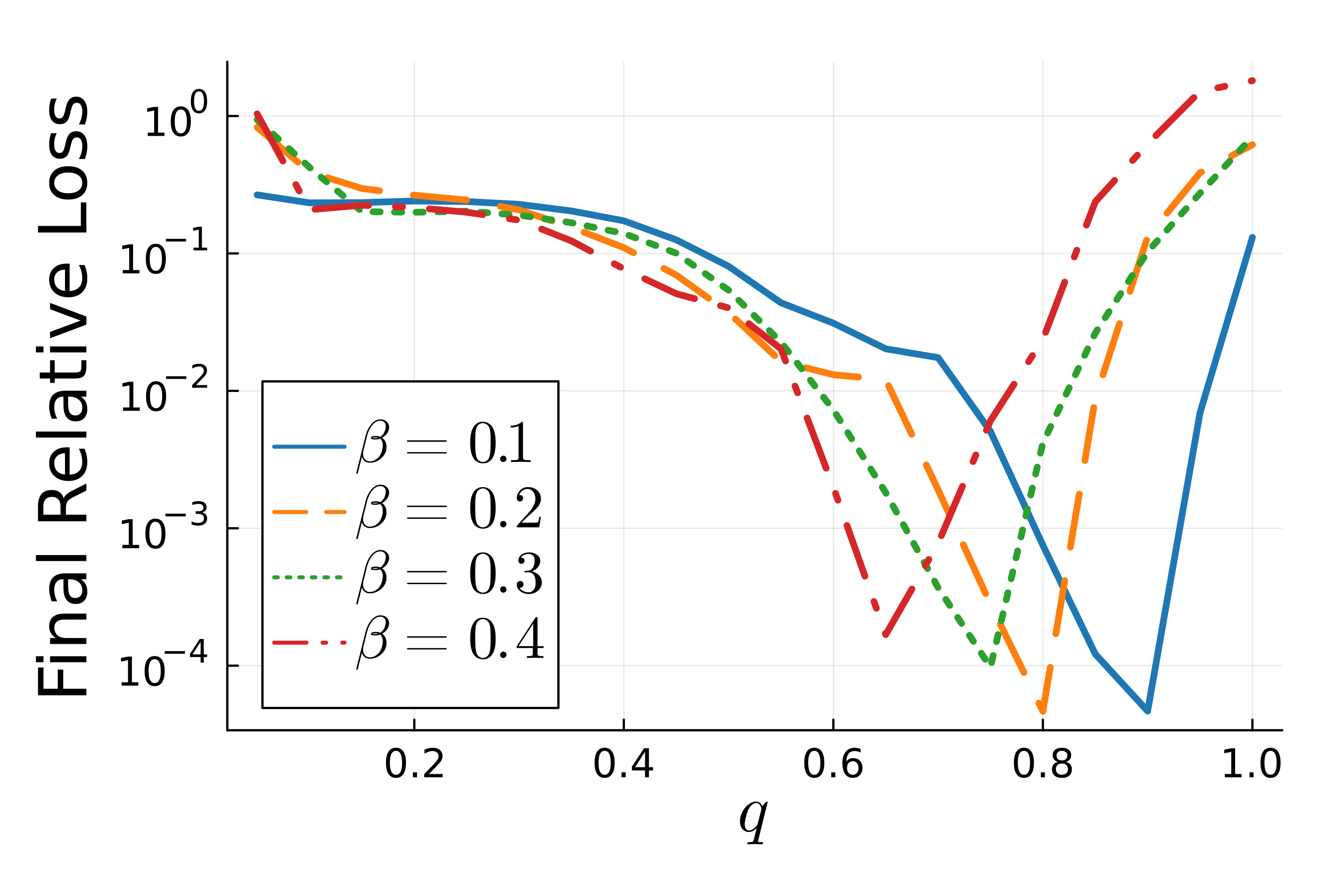}\hfill\includegraphics[width=0.45\textwidth]{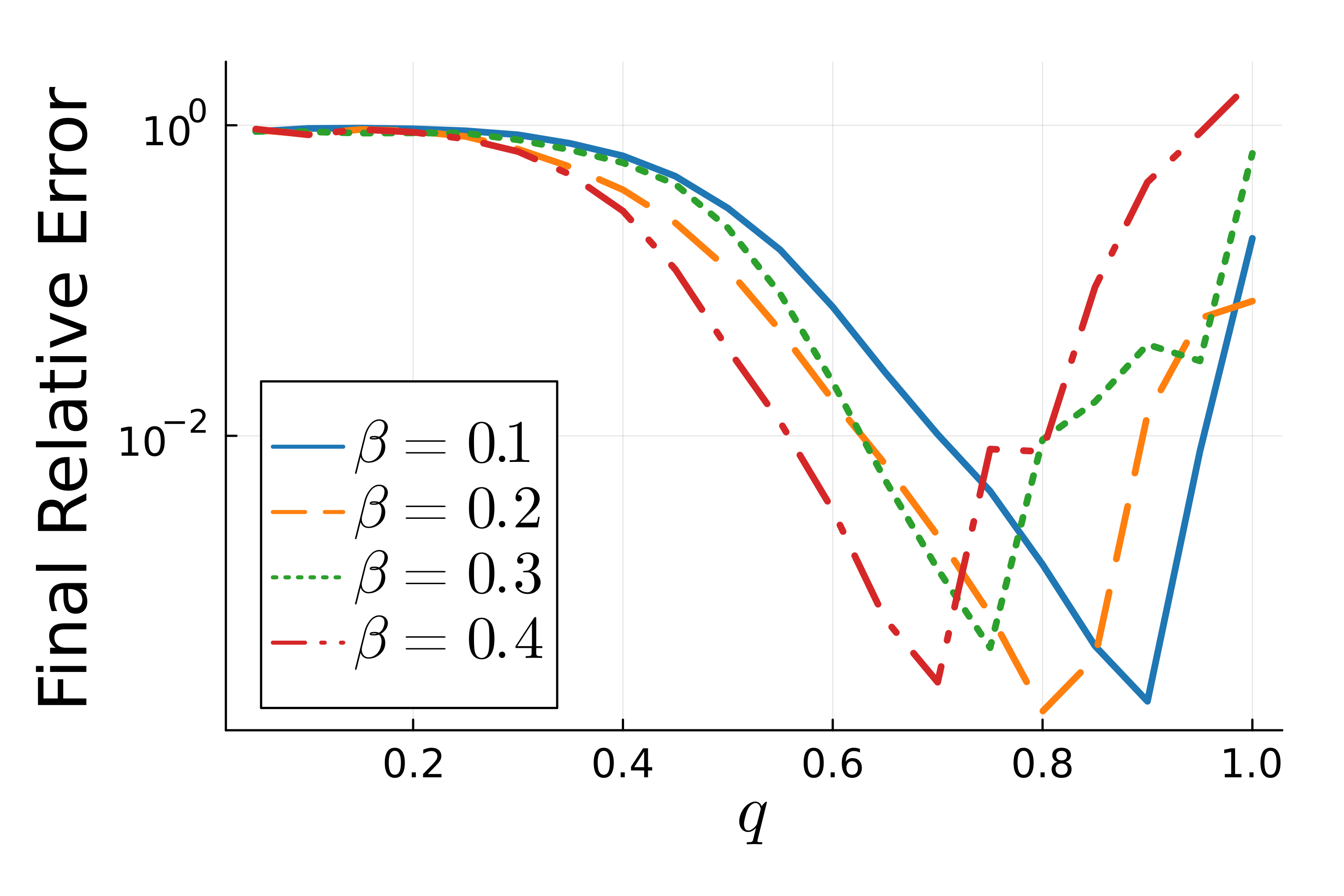}
    \caption{(Left) Final relative loss $F_G(\mathbf{x}_{10000})/F_G(\mathbf{x}_0)$ for QkL-SGD with $q \in (0,1]$ applied to the noiseless polynomial regression loss~\eqref{eq:poly_regression} with fraction of corruption $\beta \in \{0.1, 0.2, 0.3, 0.4\}$; (right) final relative squared error $\|\mathbf{x}_{10000} - \ve{x}^\star_{\rm clean}\|^2/\|\mathbf{x}_0 - \ve{x}^\star_{\rm clean}\|^2$ for QkL-SGD with $q \in (0,1]$ (log scale). }\label{fig:figure-02}
\end{figure}

\subsection{Polynomial Regression}\label{subsec:poly_regression}

We next turn to the typical polynomial regression task, where the data samples are all assumed to be affected by a small amount of noise in the response values.  This change causes Assumption~\ref{assump:sgd}.\ref{ass:good_set_regularity} to fail.  As in the experiments in the previous section, Subsection~\ref{subsec:noiseless_poly_regression}, our goal is to assess the robustness of QkL-SGD to a small fraction of samples with response values affected by significant corruption.

The data setting is identical to Subsection~\ref{subsec:noiseless_poly_regression}, with the exception that the data $\{(x_i, y_i)\}_{i=1}^m$ is generated according to
\[
y_i = p(x_i) + \epsilon_i, \quad \text{with } p(x) = 1 + 0.5x - 0.2x^2 + 0.05x^3,
\]
where noise values $\epsilon_i$ are sampled from $\mathcal{N}(0,0.3^2)$, and as before, the $x_i$ are uniformly distributed over the interval $[-3,3)$ and $m=200$. The outlier set and corruptions are generated as before in Subsection~\ref{subsec:noiseless_poly_regression}.

The model~\eqref{eq:poly_regression} and the experimental setup is the same as in Subsection~\ref{subsec:noiseless_poly_regression}.

\paragraph{Assumptions for polynomial regression loss}

As before for the noiseless polynomial regression model, Assumptions~\ref{assump:sgd}.\ref{ass:outlier_nonnegativity}, \ref{assump:sgd}.\ref{ass:component_smoothness}, and~\ref{assump:sgd}.\ref{ass:subset_strong_convexity} all hold for the noisy polynomial regression model with the same parameters.  However, due to the noise in the uncorrupted response variables, we no longer have $f_i(\ve{x}^\star_{\rm clean}) = 0$ and
     $\nabla f_i(\ve{x}^\star_{\rm clean}) = 0$ for $i \in [G]$, and thus Assumption~\ref{assump:sgd}.\ref{ass:good_set_regularity} fails.

\subsubsection{Various outlier settings}

We now turn to QkL-SGD for the noisy polynomial regression model under a variety of small to moderate corruption fractions $\beta$.  We take the ideal quantile value $q = 1 - \beta$.

\paragraph{Few outliers (\texorpdfstring{$\beta = 0.05$, $q = 0.95$, $k = m$}{beta = 0.05, q = 0.95, k = m})}

In the noisy polynomial regression experiments shown in Figure~\ref{fig:figure-03a}, we again corrupt 5\% of the sample response values $y_i$ with large additive corruptions drawn from $\mathcal{N}(0,100)$. We compare QkL-SGD with $q = 0.95$ and $k = m$ against SGD and MkL-SGD with $k = m$. Consistent with the noiseless setting, QkL-SGD achieves faster convergence and lower final relative loss
and relative squared error
than SGD. In contrast, MkL-SGD exhibits little progress, as selecting the minimum-loss component yields updates of negligible magnitude early in training. The learned polynomials (right plot) further illustrate the robustness of QkL-SGD: its fit remains close to the underlying target function despite the presence of corrupted observations, whereas the polynomial obtained with SGD is noticeably skewed by the outliers.

\paragraph{Moderate outliers (\texorpdfstring{$\beta = 0.1$, $q = 0.9$, $k = m$}{beta = 0.1, q = 0.9, k = m})}

Figure~\ref{fig:figure-03b} reports results for a more heavily corrupted setting in which 10\% of the response values $y_i$ are perturbed by additive corruption drawn from $\mathcal{N}(0,100)$. In this regime, QkL-SGD with $q = 0.9$ and $k = m$ continues to substantially outperform SGD in both relative loss
and relative error,
while MkL-SGD once again stagnates after only limited progress. The polynomial visualization (right plot) highlights the robustness of QkL-SGD under corruption: the recovered polynomial remains closely aligned with the clean polynomial $p(x)$, whereas the polynomial obtained by SGD deviates significantly due to sensitivity to the corrupted observations.

\begin{figure}
    \begin{subfigure}{\textwidth}
    \centering
    \includegraphics[width=0.3\textwidth]{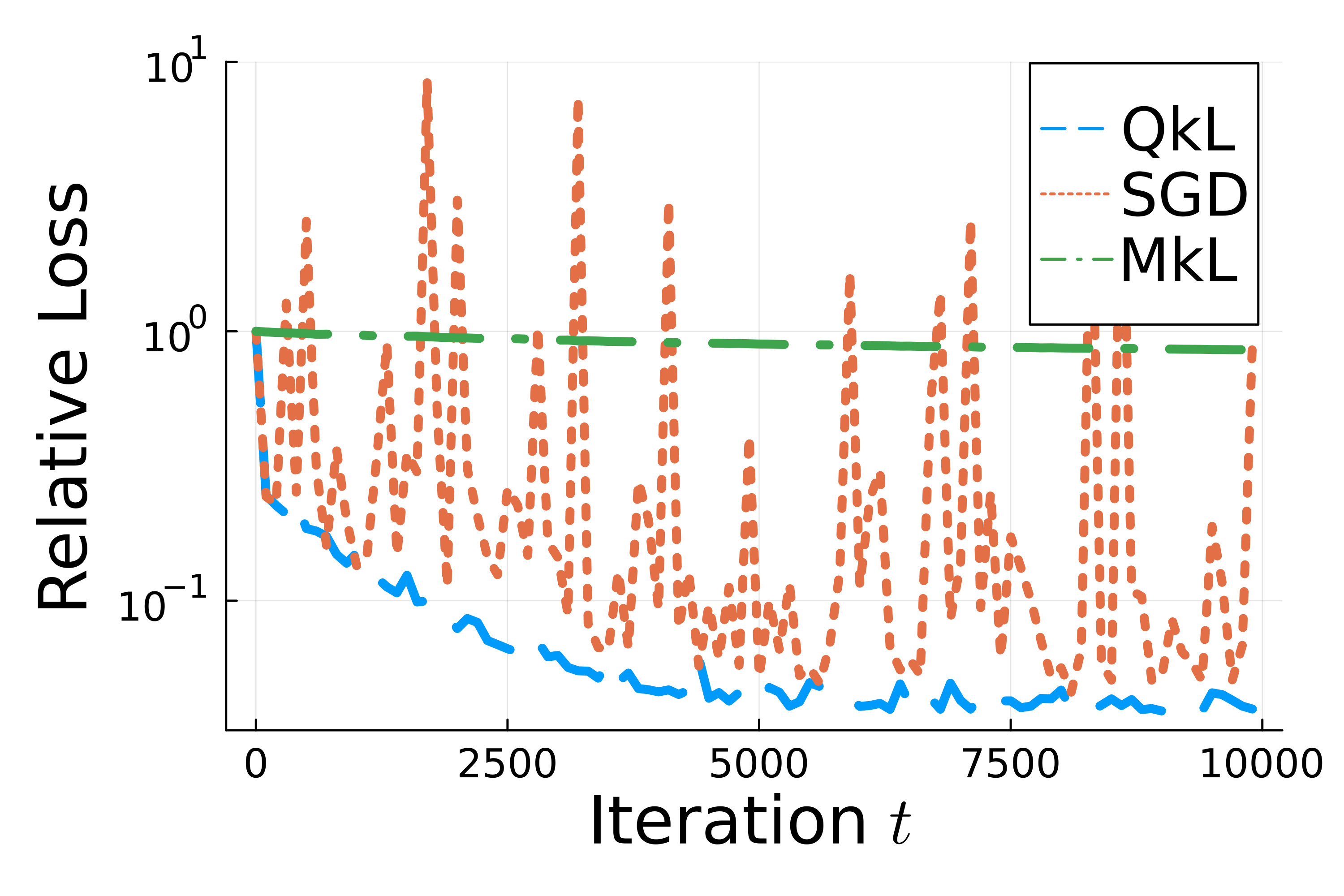}
\quad\includegraphics[width=0.3\textwidth]{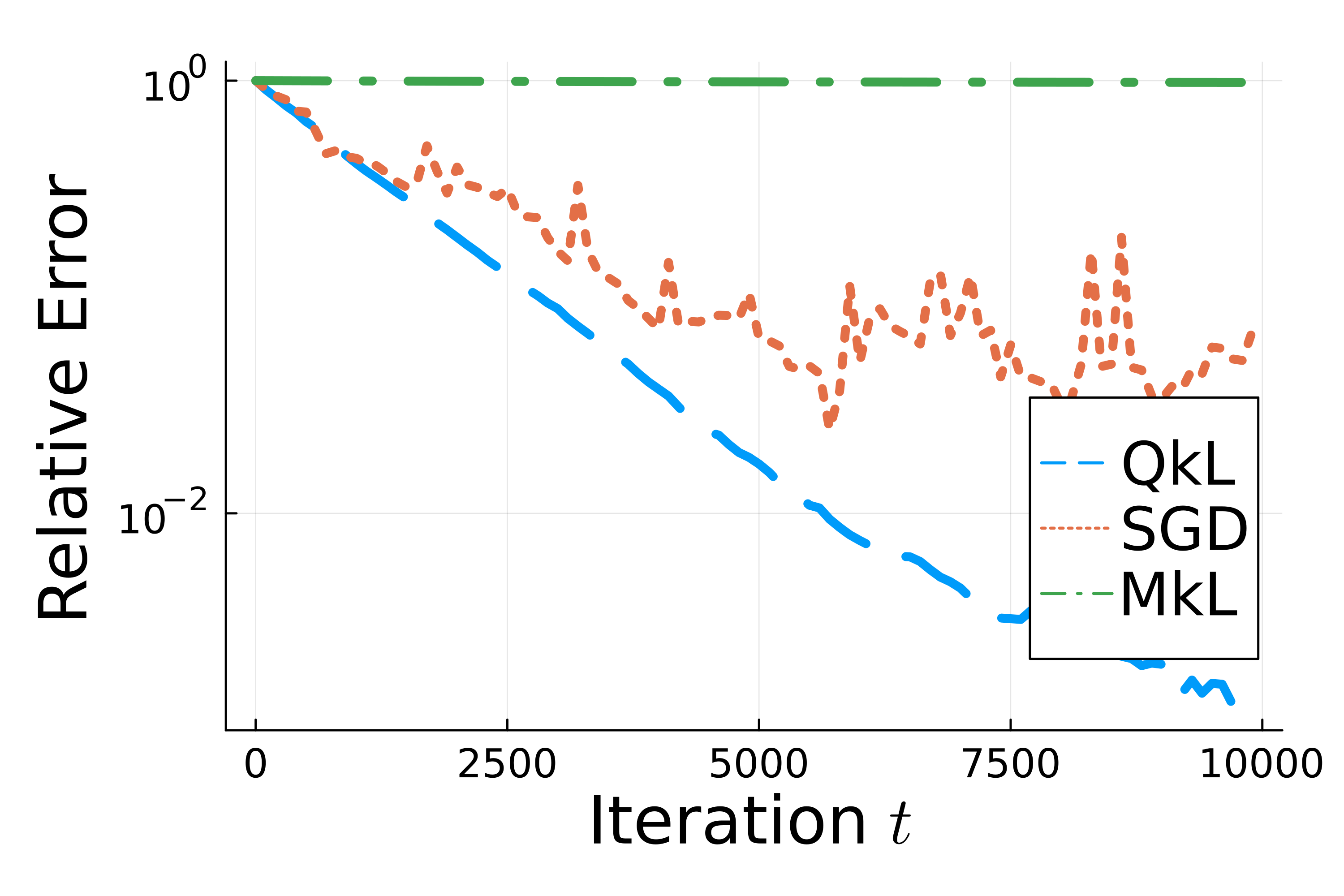}
\hfill\includegraphics[width=0.32\textwidth]{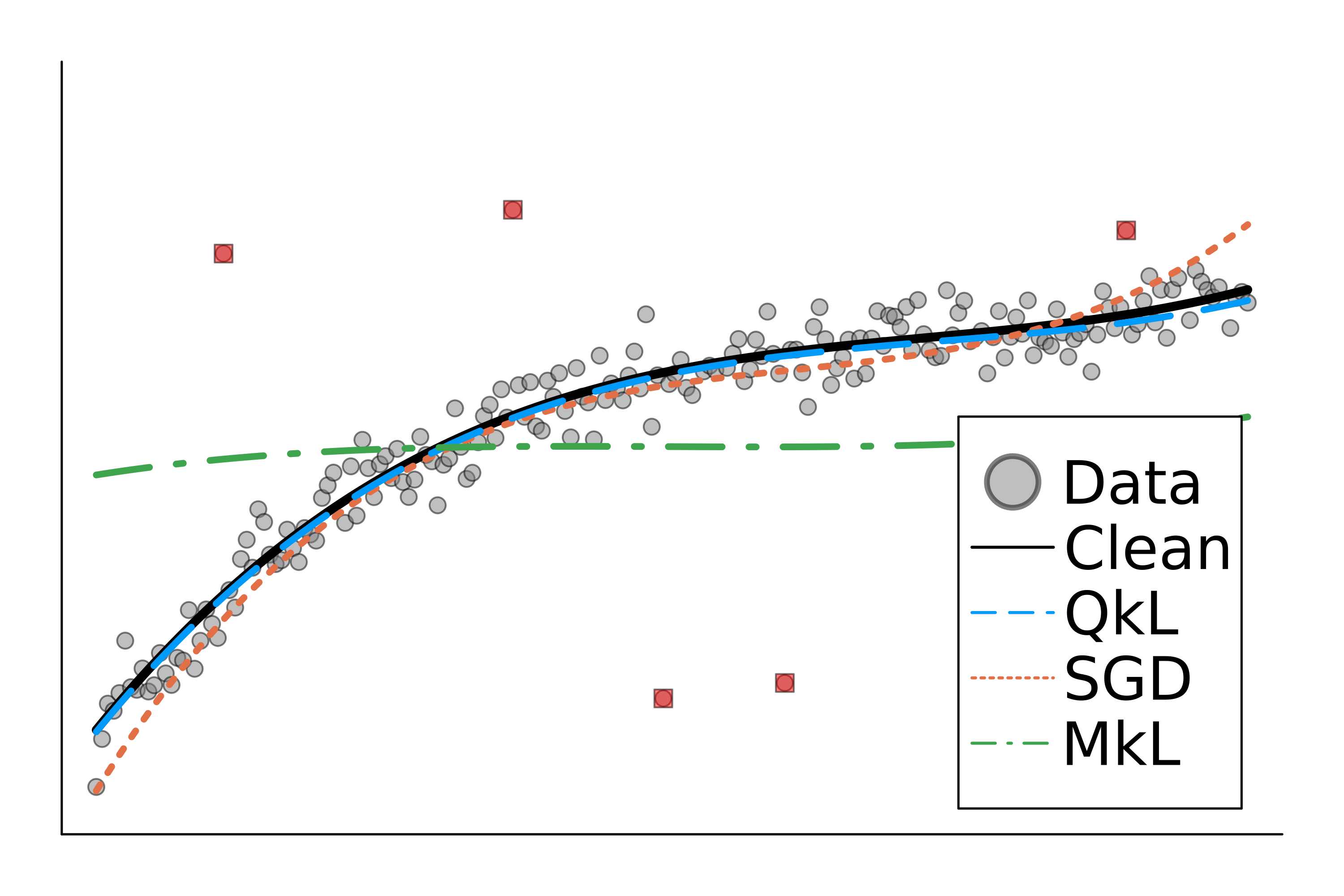}
    \caption{$\beta = 0.05$ response corruption.}
    \label{fig:figure-03a}
\end{subfigure}

\begin{subfigure}{\textwidth}
    \centering
    \includegraphics[width=0.3\textwidth]{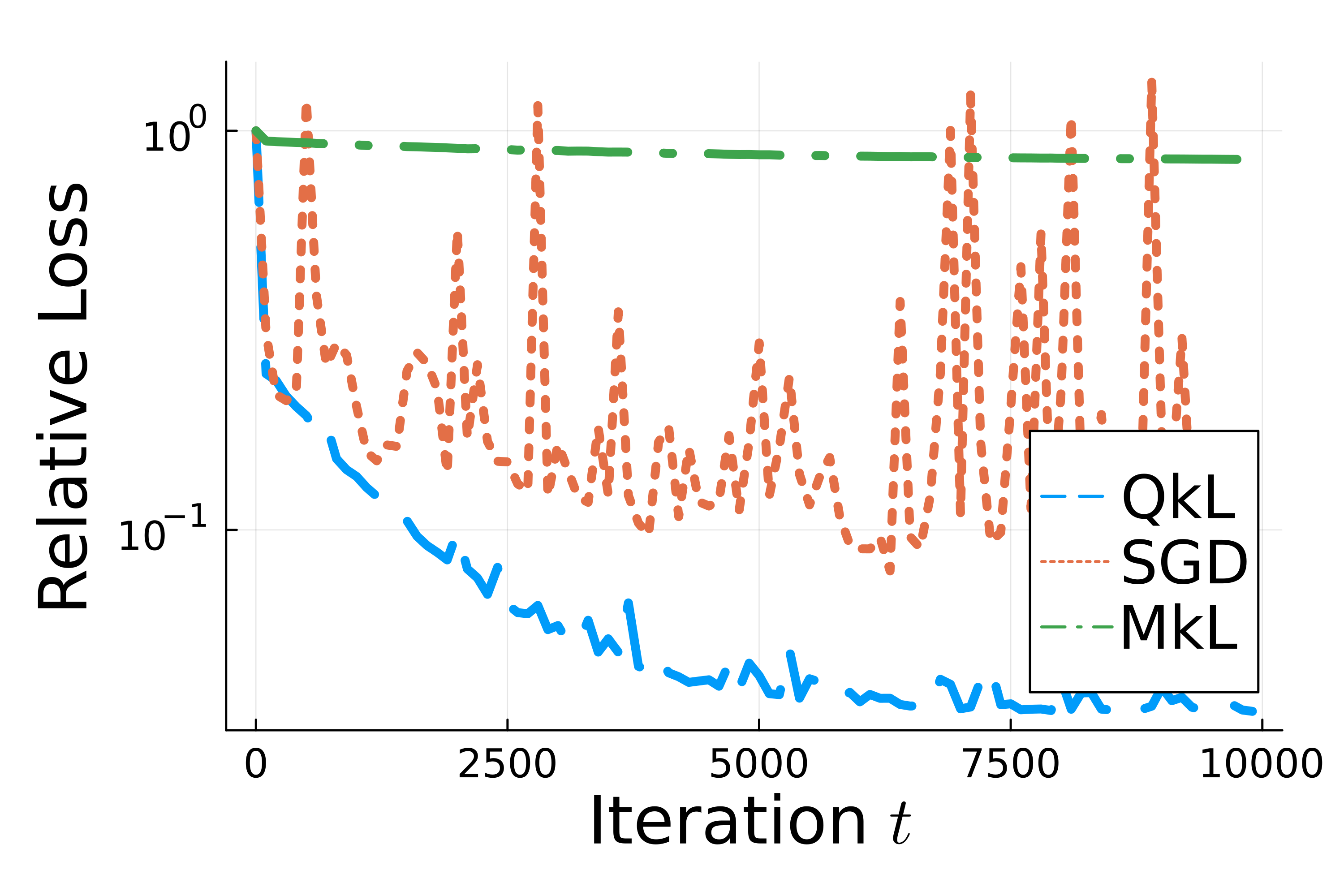}
\quad\includegraphics[width=0.3\textwidth]{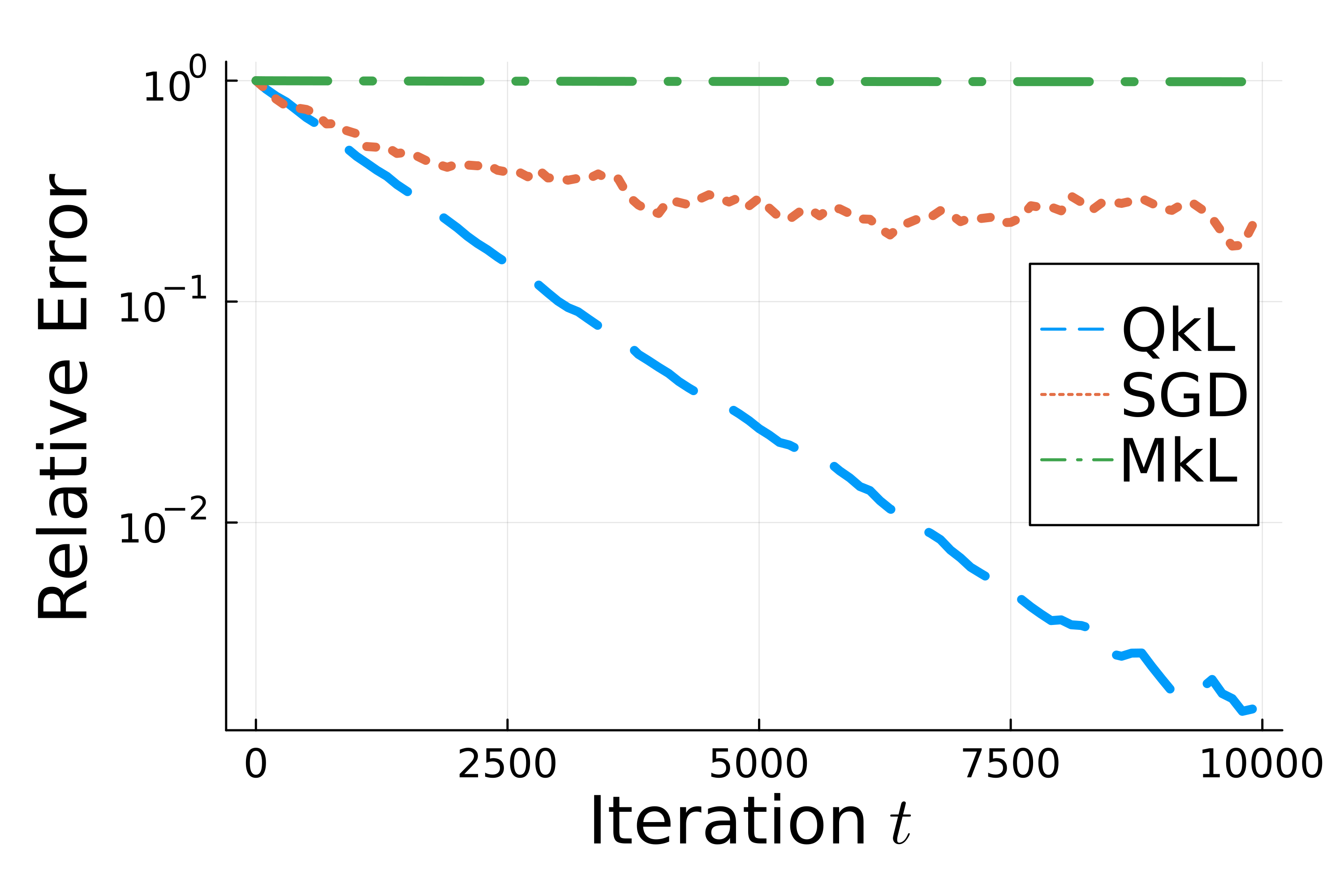}
\hfill\includegraphics[width=0.32\textwidth]{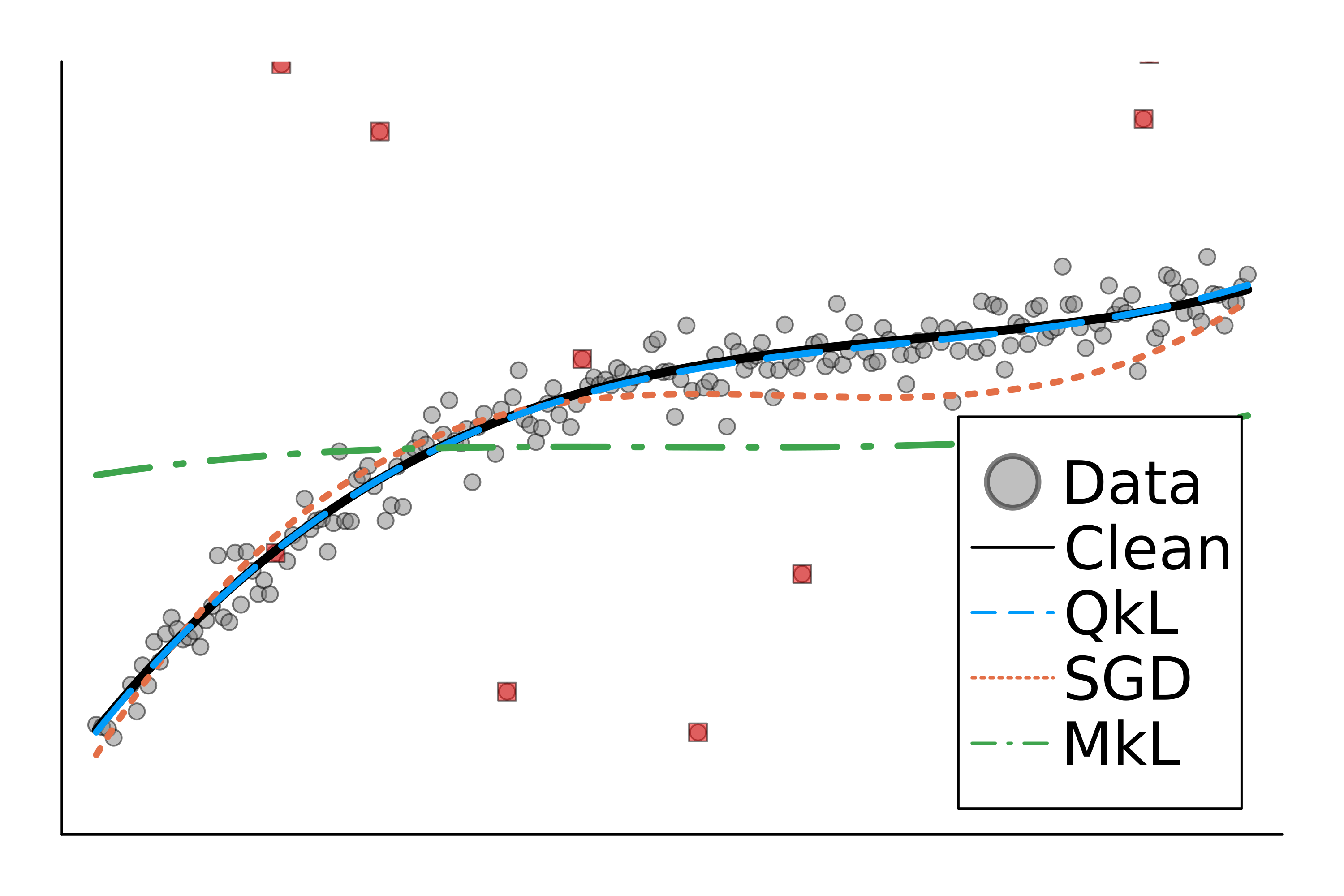}
    \caption{$\beta = 0.1$ response corruption.
    }
    \label{fig:figure-03b}
\end{subfigure}

\caption{QkL-SGD with $(k,q)=(m,1-\beta)$ and MkL-SGD with $k=m$ applied to polynomial regression loss~\eqref{eq:poly_regression}. (Left) Relative good-set loss $F_G(\ve x_t)/F_G(\ve x_0)$ per iteration. (Center) Relative squared parameter error $\|\ve x_t-\ve x^\star_{\rm clean}\|_2^2/\|\ve x_0-\ve x^\star_{\rm clean}\|_2^2$. (Right) Learned polynomials at the final iterate, together with the clean polynomial $p(x)$. Uncorrupted data samples, visualized as circles, are noisy; corrupted data samples are visualized as squares.}
\end{figure}

\subsubsection{Varying $q$ with fixed corruption fraction $\beta$}

Figure~\ref{fig:figure-04} illustrates the effect of varying the quantile parameter $q \in \{0.05, 0.1, \cdots, 1.0\}$ on the final relative loss (left) and relative error (right) after 10,000 iterations of QkL-SGD with $k = m$. The method is applied to the noisy polynomial regression problem~\eqref{eq:poly_regression} using the data model described above with corruption fractions $\beta \in \{0.1, 0.2, 0.3, 0.4\}.$  Relative loss and relative errors are averaged over 100 independent trials of 10000 iterations of QkL-SGD. Like in the noiseless polynomial regression case, we see that the quantile $q$ producing lowest relative error and loss is approximately $1 - \beta$.  However, in this case, the relative loss values are very similar for quantiles nearby to the optimal, but the relative error illustrates the sensitivity of QkL-SGD to the quantile if an accurate solution is desired.

\begin{figure}
    \includegraphics[width=0.45\textwidth]{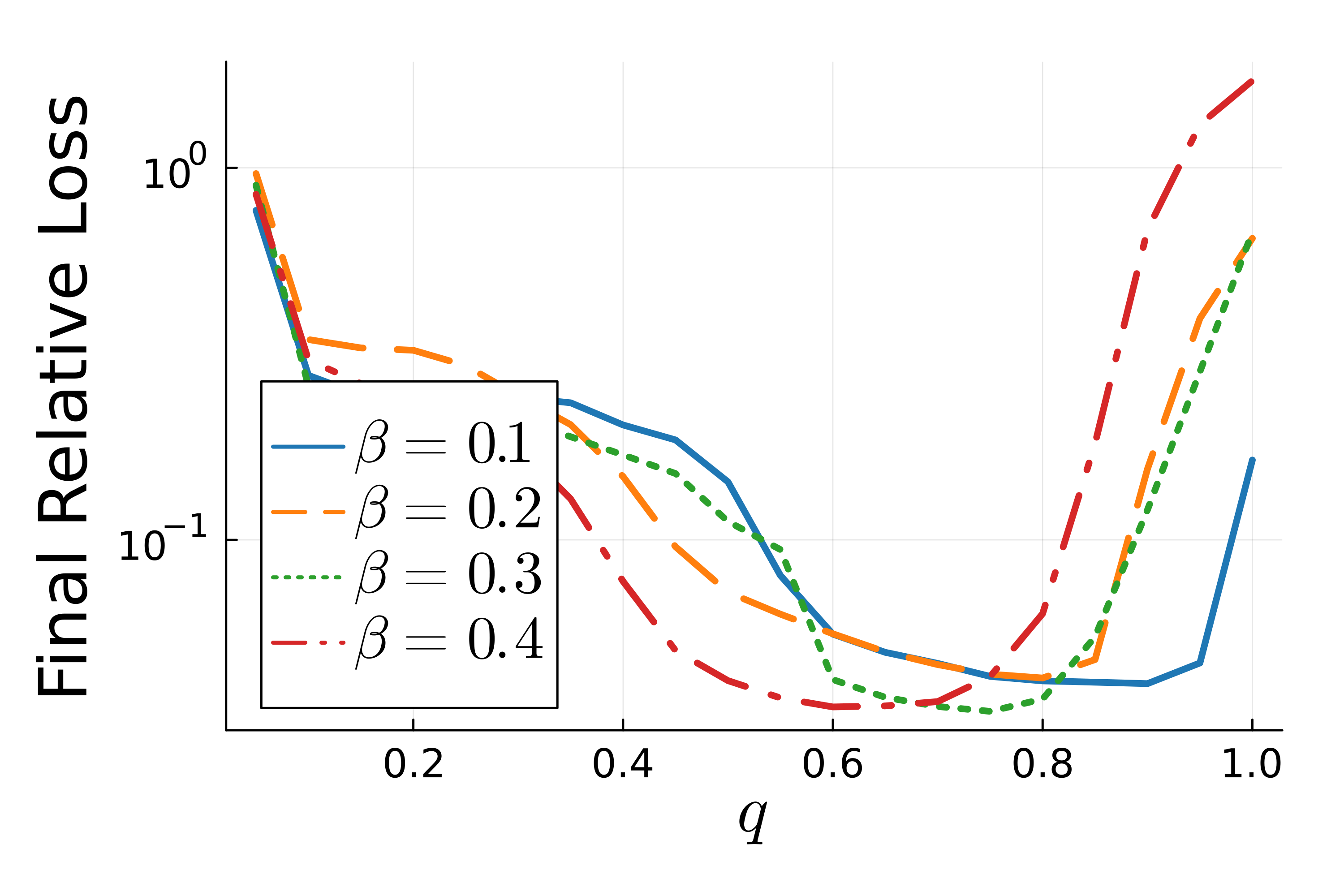}\hfill\includegraphics[width=0.45\textwidth]{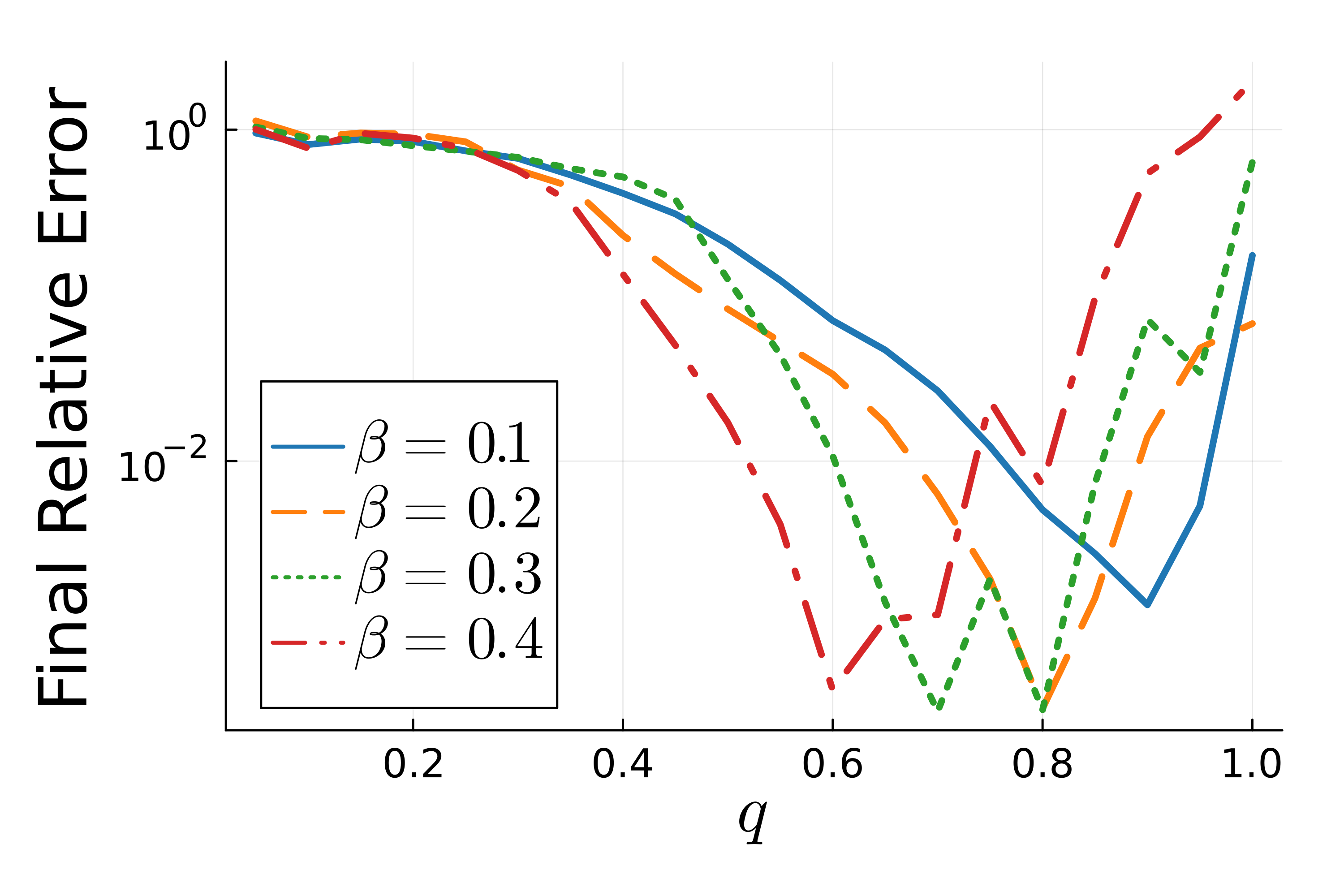}
    \caption{(Left) Final relative loss $F_G(\mathbf{x}_{10000})/F_G(\mathbf{x}_0)$ for QkL-SGD with $q \in (0,1]$ applied to the noisy polynomial regression loss~\eqref{eq:poly_regression} with fraction of corruption $\beta \in \{0.1, 0.2, 0.3, 0.4\}$; (right) final relative squared error $\|\mathbf{x}_{10000} - \ve{x}^\star_{\rm clean}\|^2/\|\mathbf{x}_0 - \ve{x}^\star_{\rm clean}\|^2$ for QkL-SGD with $q \in (0,1]$ (log scale). }\label{fig:figure-04}
\end{figure}

\subsection{Regularized Logistic Regression}\label{subsec:logistic_regression}

We next consider $\ell_2$-regularized logistic regression as a smooth nonlinear classification problem with corrupted labels. This experiment keeps the smoothness present in the polynomial-regression examples, but moves from quadratic least squares to a nonquadratic classification loss. As in noisy polynomial regression, the interpolation part of Assumption~\ref{assump:sgd} fails: the individual good losses and gradients need not vanish at the clean-data reference solution. In this section, we generate synthetic, linearly separable two-dimensional data and examine the impact of introducing label corruption on optimization performance.  The experiments are designed to assess robustness to outliers, which are introduced by randomly flipping (corrupting) a fraction $\beta$ of the labels.

For a parameter vector $\ve{x}$, the logistic model predicts the probability of the positive class for data point $\ve d_i$ via the sigmoid function as
\[
\widehat y_i=\sigma(\ve d_i^\top \ve x)
=\frac{1}{1+\exp(-\ve d_i^\top \ve x)}.
\]
Given binary labels $l_i\in\{0,1\}$, we use the regularized component losses
\begin{equation}
f_i(\ve x)
=
-\left[
l_i\log(\widehat y_i)
+(1-l_i)\log(1-\widehat y_i)
\right]
+\frac{\lambda}{2m}\|\ve x\|_2^2,
\label{eq:logistic_regression_component}
\end{equation}
and the finite-sum objective
\begin{equation}
F(\ve x)=\frac1m\sum_{i=1}^m f_i(\ve x).
\label{eq:logistic_regression}
\end{equation}

In all experiments, we generate $m=200$ augmented data points $\ve d_i=(1,u_i,v_i)^\top\in\mathbb R^3$. The first component of each data point is fixed to be $(\ve{d}_i)_1 = 1$ to allow for an affine bias term in the logistic regression model.

Half of the points belong to the positive class, with $(u_i,v_i)$ sampled i.i.d.\ from a Gaussian distribution centered at $(2,2)$ with identity covariance and clean label $l_i^{\rm clean}=1$. The remaining half belong to the negative class, with $(u_i,v_i)$ sampled i.i.d.\ from a Gaussian distribution centered at $(-2,-2)$ with identity covariance and clean label $l_i^{\rm clean}=0$. Note that before corruption, the data are linearly separable with high probability. To create outliers, we select a fraction $\beta$ of the data points and flip their labels. Thus, for $i\in G$ the observed label agrees with the clean label, while for $i\in O$ the observed label is $l_i=1-l_i^{\rm clean}$. The algorithms use only the observed labels. In the visualizations below, we plot the last two coordinates $(u_i,v_i)$ together with the learned decision boundary
\[
\{\,\ve y\in\mathbb R^2: x_1+\ve y^\top \ve x_{2:}=0\,\}.
\]

The objective is to learn a parameter vector $\mathbf{x} \in \mathbb{R}^3$ that correctly classifies the data points whose labels have not been corrupted, and ideally generalizes to correctly classify even data points whose labels have been corrupted and no longer represents to which class the data truly belongs.
Aligned with that, we run SGD on the
clean dataset before any labels are flipped. We denote the resulting parameter
after $10000$ iterations by $\ve{x}^\star_{\rm clean}$ and plot the
corresponding clean-data reference boundary
\[
\mathcal B_{\rm clean}
=
\{\,\ve y\in \mathbb R^2 :
(\ve{x}^\star_{\rm clean})_1
+
\ve y^\top(\ve{x}^\star_{\rm clean})_{2:}=0
\,\}.
\]
This boundary is used only as a visual reference and is labeled
``Clean'' in the plots.

In all logistic-regression experiments, SGD, QkL-SGD, and min-$k$-loss SGD (MkL-SGD) are run with stepsize $\eta=0.01$ and regularization parameter $\lambda=1$. The reported objective values are evaluated on the good set $G$, and the reported parameter errors are measured relative to $\ve x^\star_{\rm clean}$.

\paragraph{Assumptions for regularized logistic regression}
The regularized logistic regression components in \eqref{eq:logistic_regression_component} are nonnegative, convex, and smooth, with
\[
L_i \le \frac14\|\ve d_i\|_2^2+\frac{\lambda}{m}.
\]
Moreover, because of the quadratic regularizer, $F_S$ is $\lambda/m$-strongly convex for every nonempty subset $S\subseteq G$. Thus Assumptions~\ref{assump:sgd}.\ref{ass:outlier_nonnegativity}, \ref{assump:sgd}.\ref{ass:component_smoothness}, and \ref{assump:sgd}.\ref{ass:subset_strong_convexity} hold at the model level. However, Assumption~\ref{assump:sgd}.\ref{ass:good_set_regularity} fails: although the good component losses are convex, they do not satisfy $f_i(\ve x^\star_{\rm clean})=0$ and $\nabla f_i(\ve x^\star_{\rm clean})=0$ individually.

\subsubsection{Various outlier settings}
In this section, we explore the convergence of QkL-SGD in a variety of outlier settings.  Here we let quantile $q = 1 - \beta$.

\paragraph{Few outliers from a single class
(\texorpdfstring{$\beta=0.01$, $q=0.99$, $k=m$}
{beta=0.01, q=0.99, k=m})}
In the first logistic-regression experiment, shown in Figure~\ref{fig:figure-05a}, we corrupt $1\%$ of the data points, all chosen from the positive class, by flipping their labels. We compare standard SGD, QkL-SGD with $(k,q)=(m,0.99)$, and MkL-SGD with $k=m$. In this full-sample setting, QkL-SGD and MkL-SGD both screen all component losses at each iteration, so the comparison isolates the effect of selecting from the lower empirical quantile rather than selecting the minimum-loss component. QkL-SGD decreases the good-set objective $F_G(\ve x_t)$ faster than SGD and tracks the clean-data reference parameter more closely. By contrast, MkL-SGD stalls early: selecting the minimum-loss component can repeatedly choose nearly solved components and produce updates of negligible size. The decision-boundary plot shows the same qualitative behavior: QkL-SGD remains close to the clean-data reference boundary, while the SGD boundary is more visibly affected by the flipped labels.

\paragraph{Moderate outliers from both classes
(\texorpdfstring{$\beta=0.1$, $q=0.9$, $k=m$}
{beta=0.1, q=0.9, k=m})}
In the experiment shown in Figure~\ref{fig:figure-05b},
$10\%$ of the data points are selected uniformly from both classes and their
labels are flipped. We compare standard SGD, QkL-SGD with screening size
$k=m$ and quantile parameter $q=0.9$, and MkL-SGD with $k=m$. Thus, both
QkL-SGD and MkL-SGD evaluate all component losses before selecting an update;
the comparison between them isolates the effect of selecting from the lower
empirical $q$-quantile rather than selecting the minimum-loss component.

QkL-SGD achieves smaller good-set loss $F_G(\ve x_t)$ and smaller parameter
error relative to the clean-data reference than standard SGD, while full-sample
MkL-SGD stalls early. The decision-boundary visualization shows the same
qualitative behavior: the QkL-SGD boundary remains closer to the clean-data
reference boundary, whereas the SGD boundary is influenced by the
flipped labels: a single data point with a non-corrupted label in the positive class is misclassified.

\paragraph{Moderate outliers from a single class
(\texorpdfstring{$\beta=0.1$, $q=0.9$, $k=m$}
{beta=0.1, q=0.9, k=m})}
We next consider a more structured corruption pattern, where $10\%$ of the
labels are flipped but all corrupted points are chosen from the positive
class. This asymmetric corruption concentrates the outlier effect in one
region of the data. We again compare standard SGD, QkL-SGD with screening
size $k=m$ and quantile parameter $q=0.9$, and MkL-SGD with $k=m$.

As shown in Figure~\ref{fig:figure-05c}, QkL-SGD remains more
robust to the corrupted labels than standard SGD, achieving lower good-set loss and smaller parameter error relative to the clean-data reference. We see that QkL-SGD learns a decision boundary that correctly classifies all correctly labeled data points.  Additionally, the extreme imbalance of the mislabeled data causes the decision boundary learned by SGD to misclassify additional correctly labeled data in the positive class. Finally, full-sample MkL-SGD again makes
little progress, consistent with the conservative behavior of minimum-loss selection.

\begin{figure}
    \centering
    \begin{subfigure}{\textwidth}
    \centering
    \includegraphics[width=0.3\textwidth]{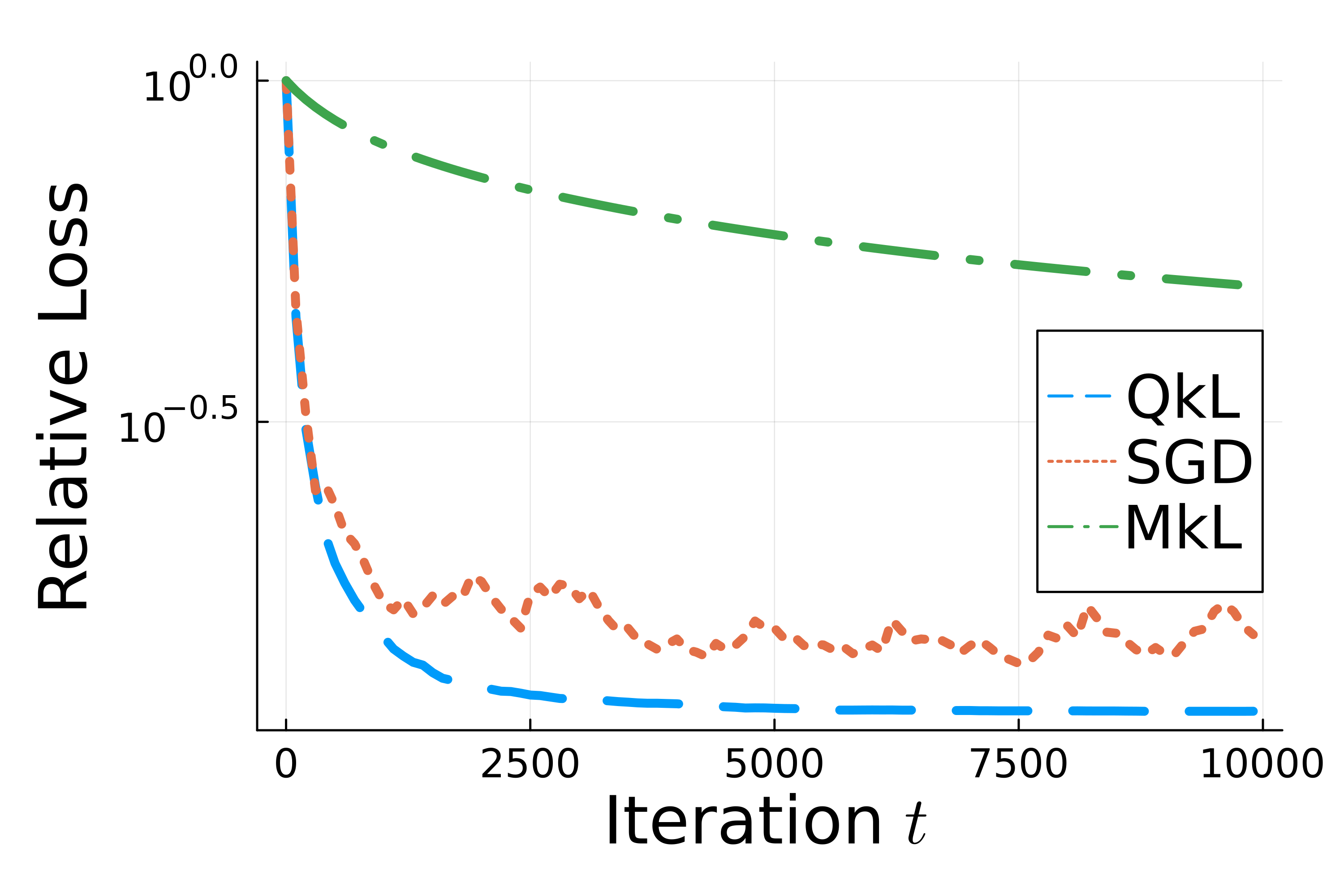}
\quad\includegraphics[width=0.3\textwidth]{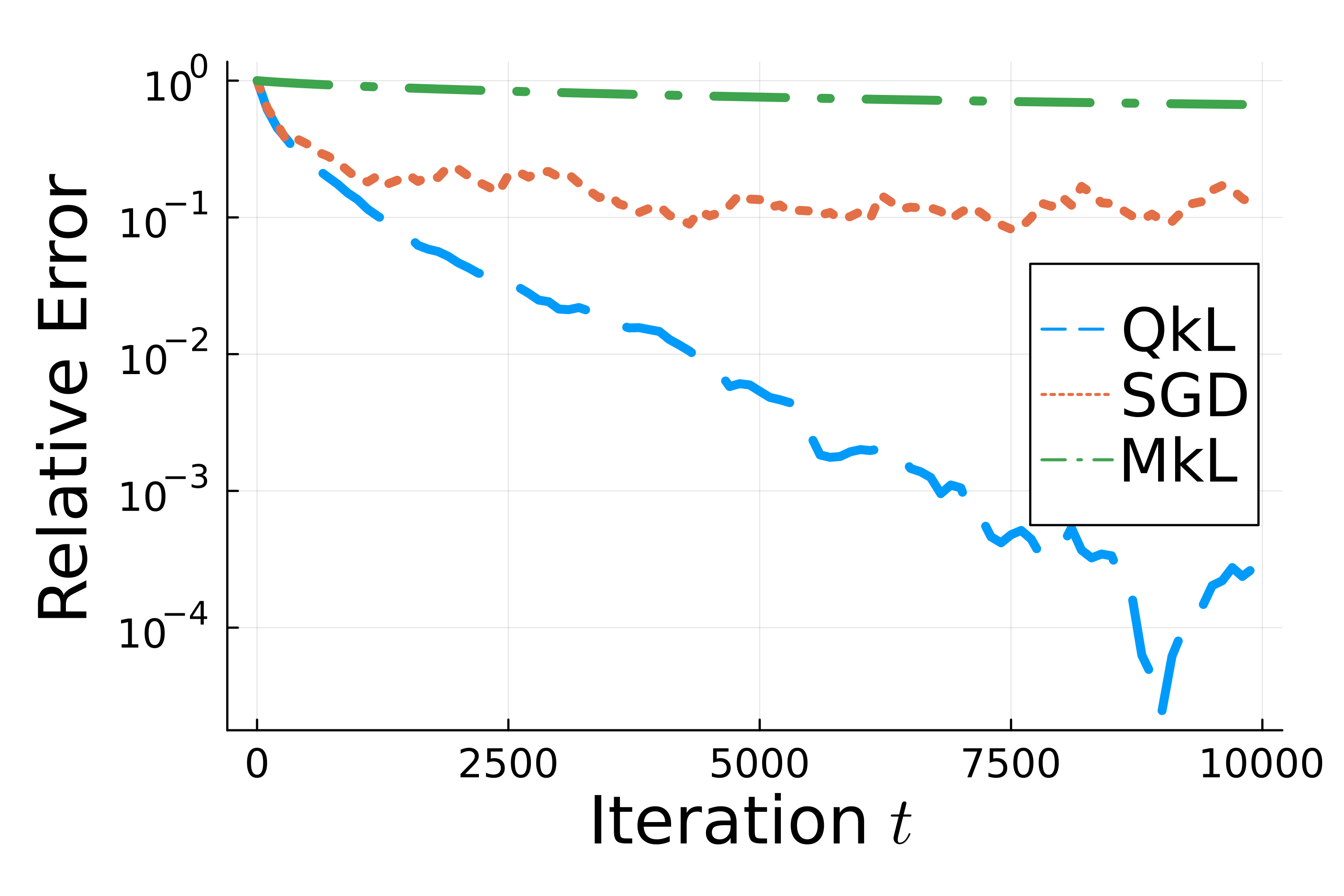}
\hfill\includegraphics[width=0.3\textwidth]{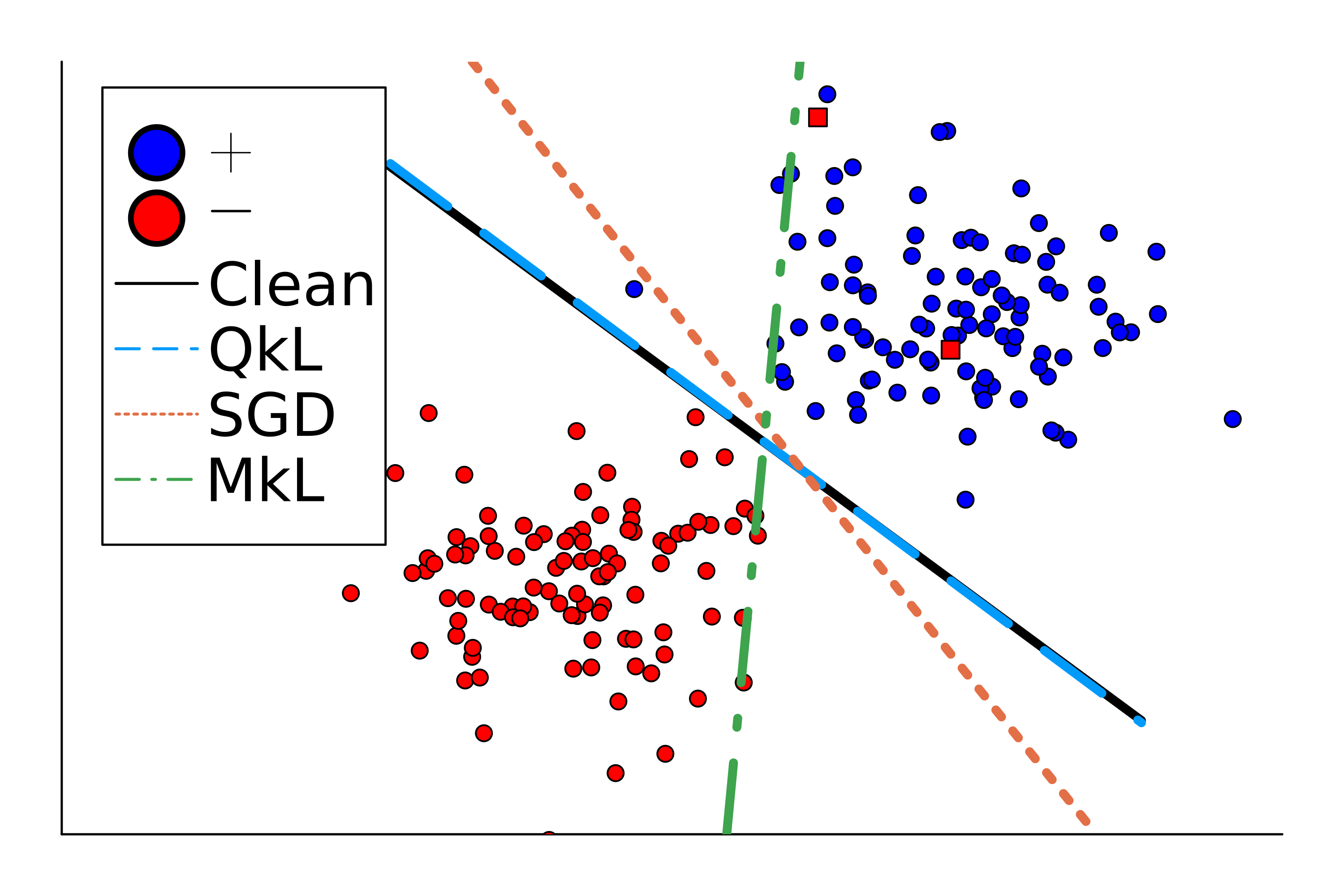}
    \caption{$\beta=0.01$ label corruption, with outliers sampled from the positive class.  }

\label{fig:figure-05a}
\end{subfigure}

\begin{subfigure}{\textwidth}
    \centering
    \includegraphics[width=0.3\textwidth]{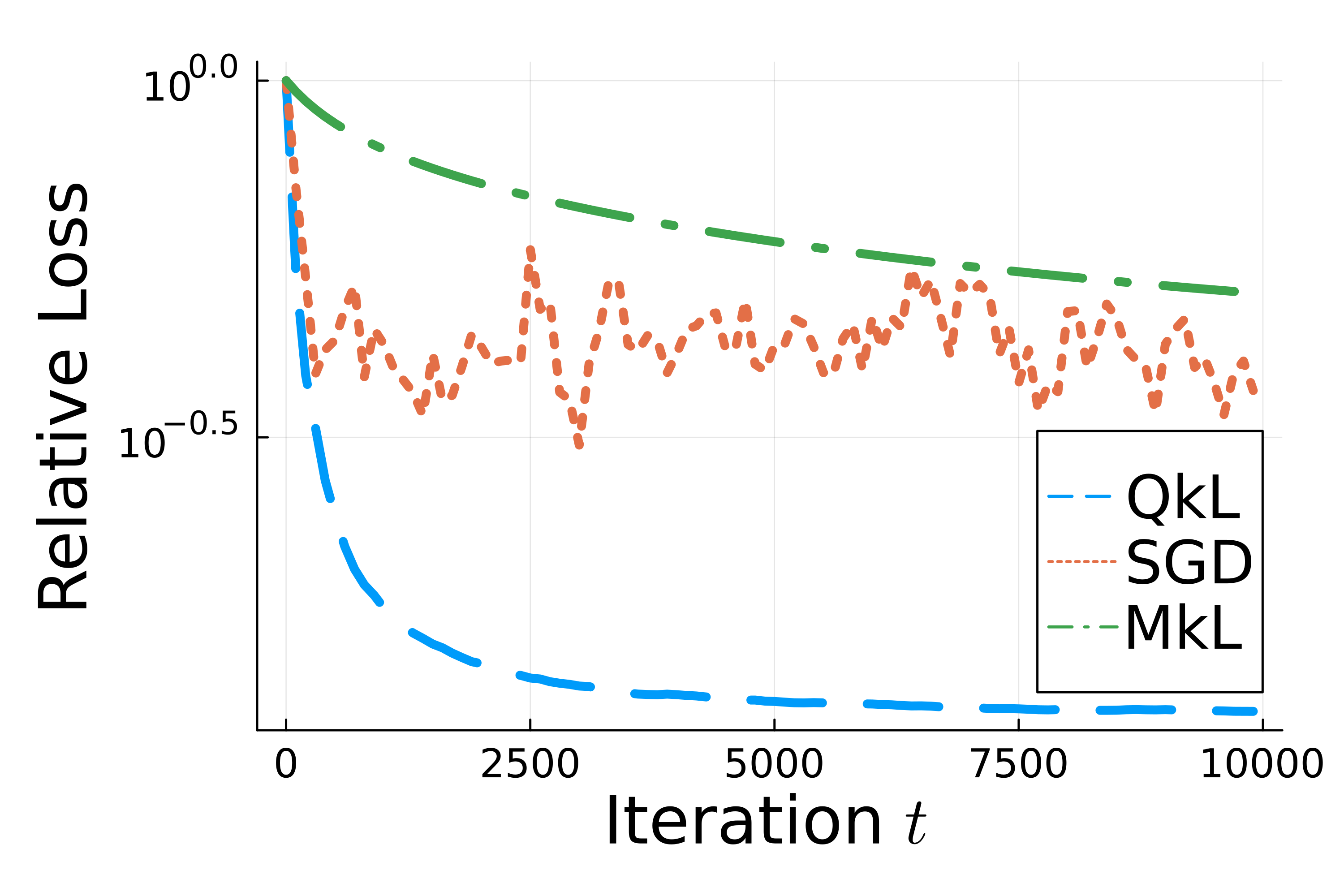}
    \quad\includegraphics[width=0.3\textwidth]{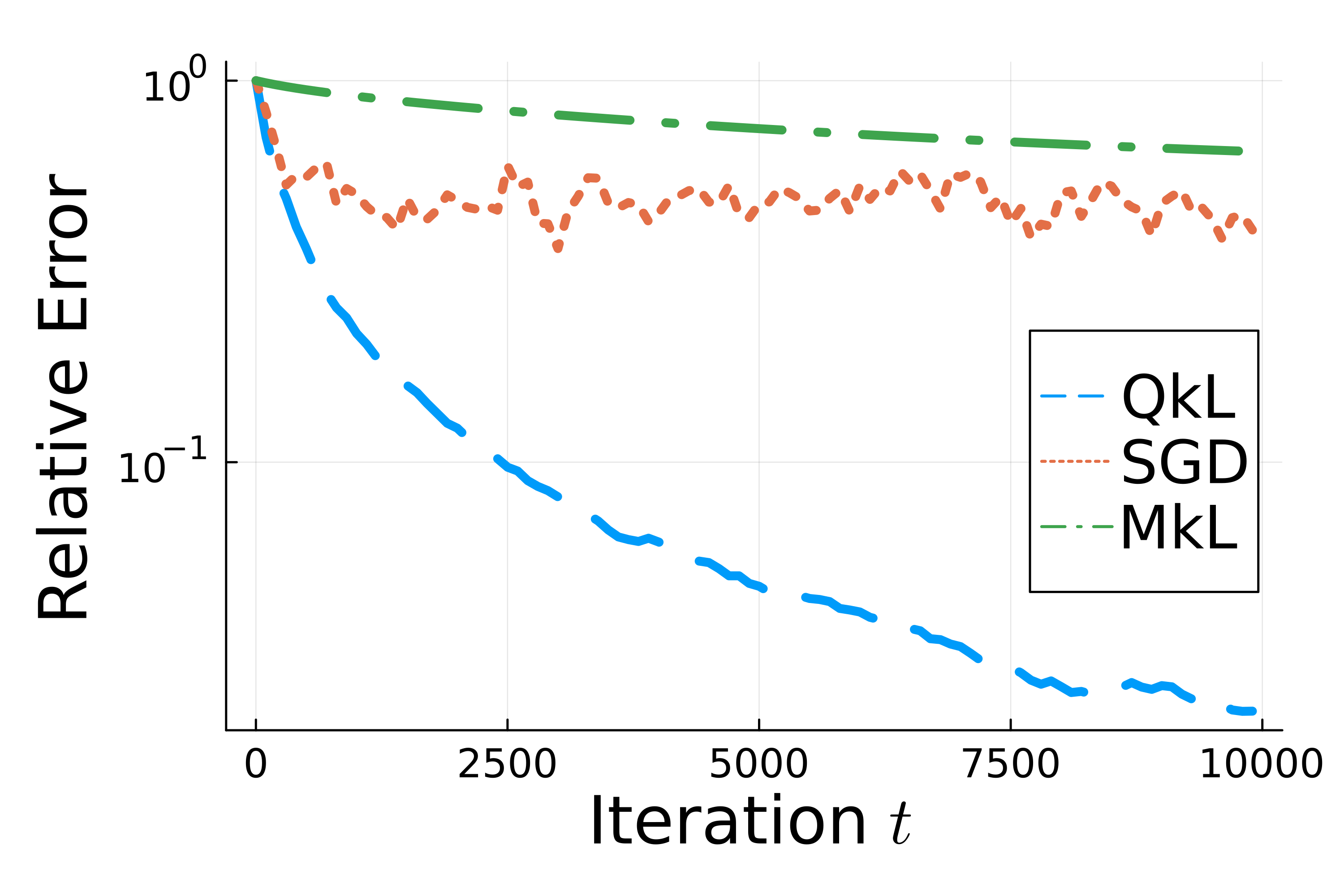}
    \hfill\includegraphics[width=0.32\textwidth]{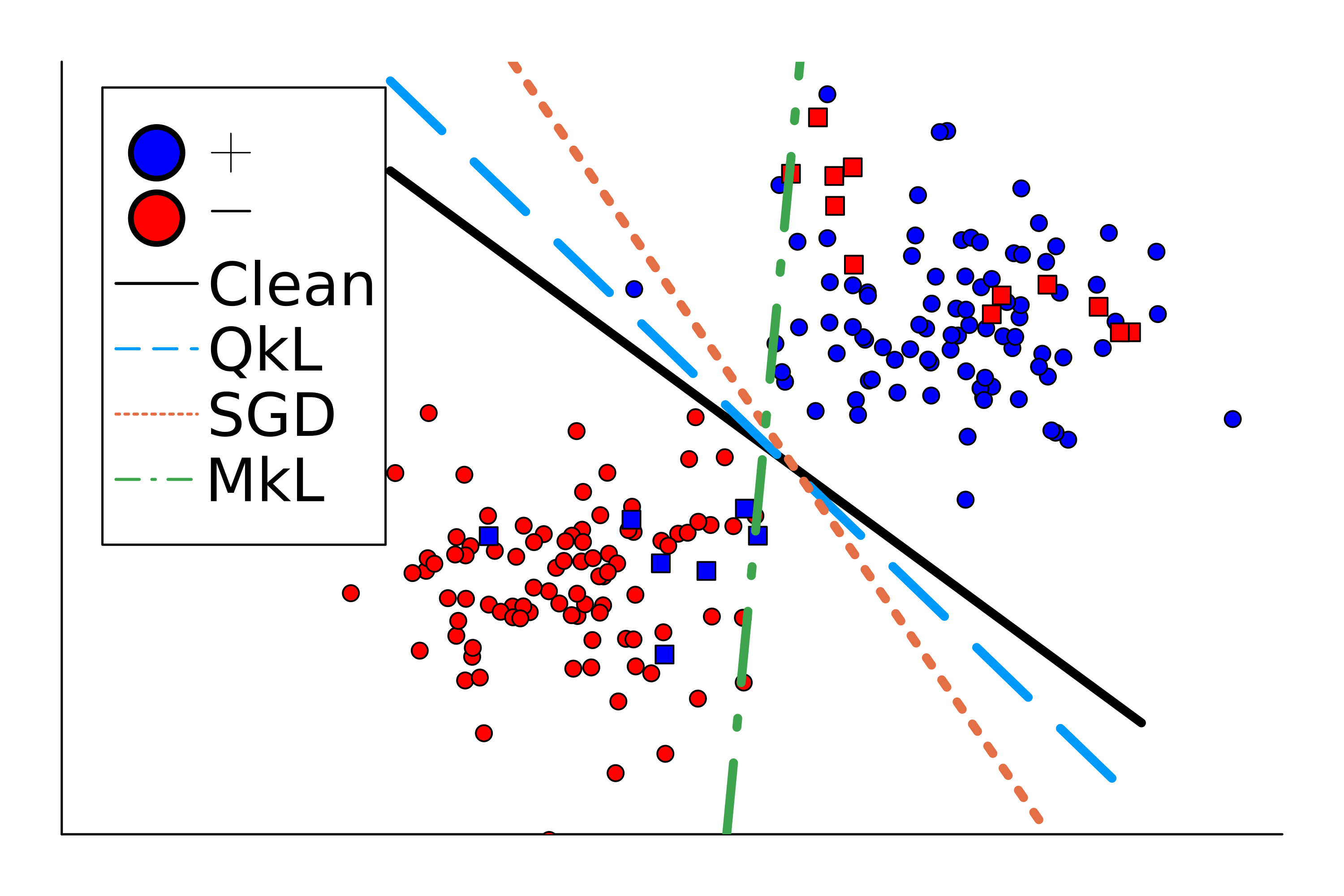}
    \caption{$\beta=0.1$ label corruption, with corrupted points sampled from the two classes. }
    \label{fig:figure-05b}
\end{subfigure}

\begin{subfigure}{\textwidth}
    \centering
    \includegraphics[width=0.3\textwidth]{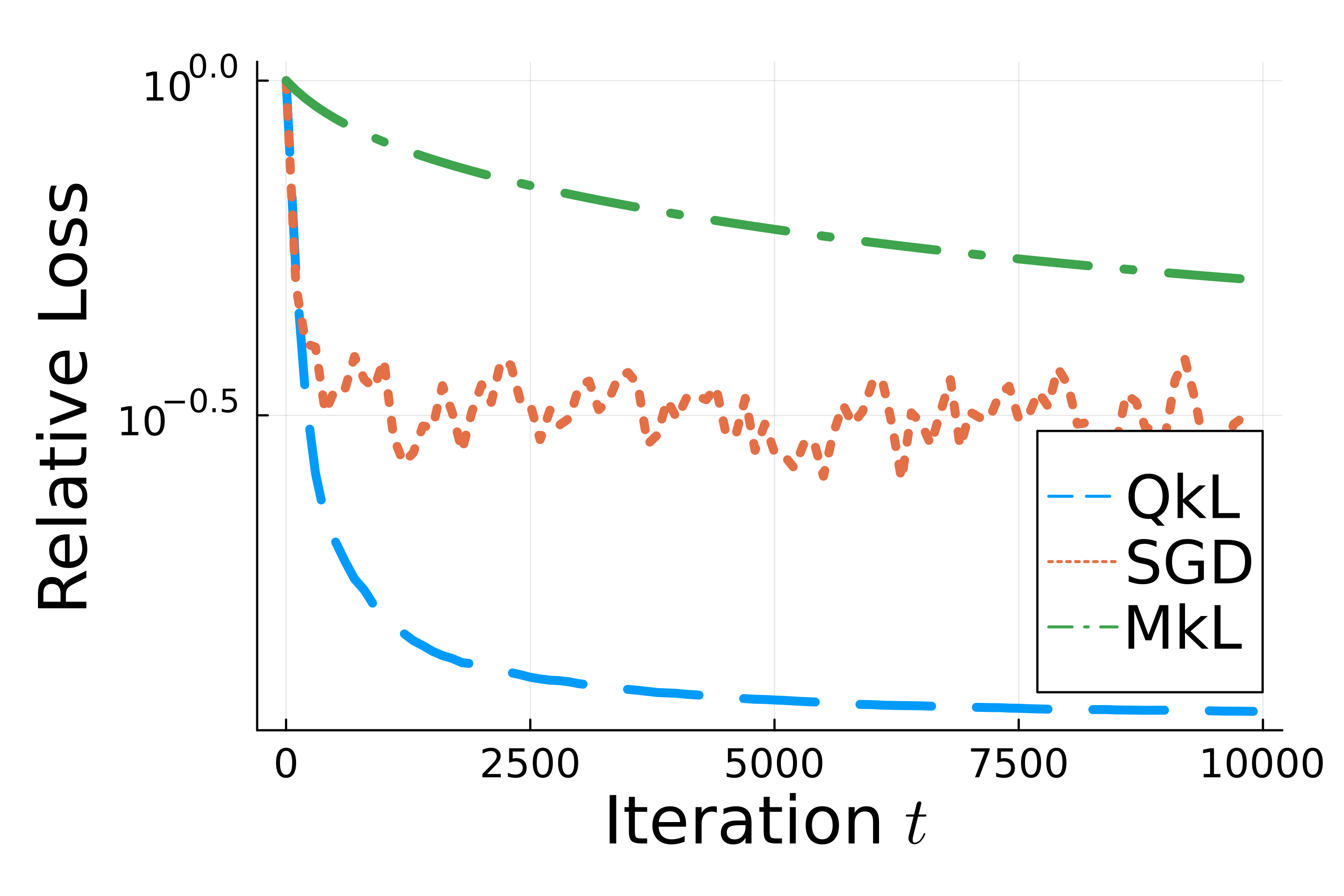}
\quad\includegraphics[width=0.3\textwidth]{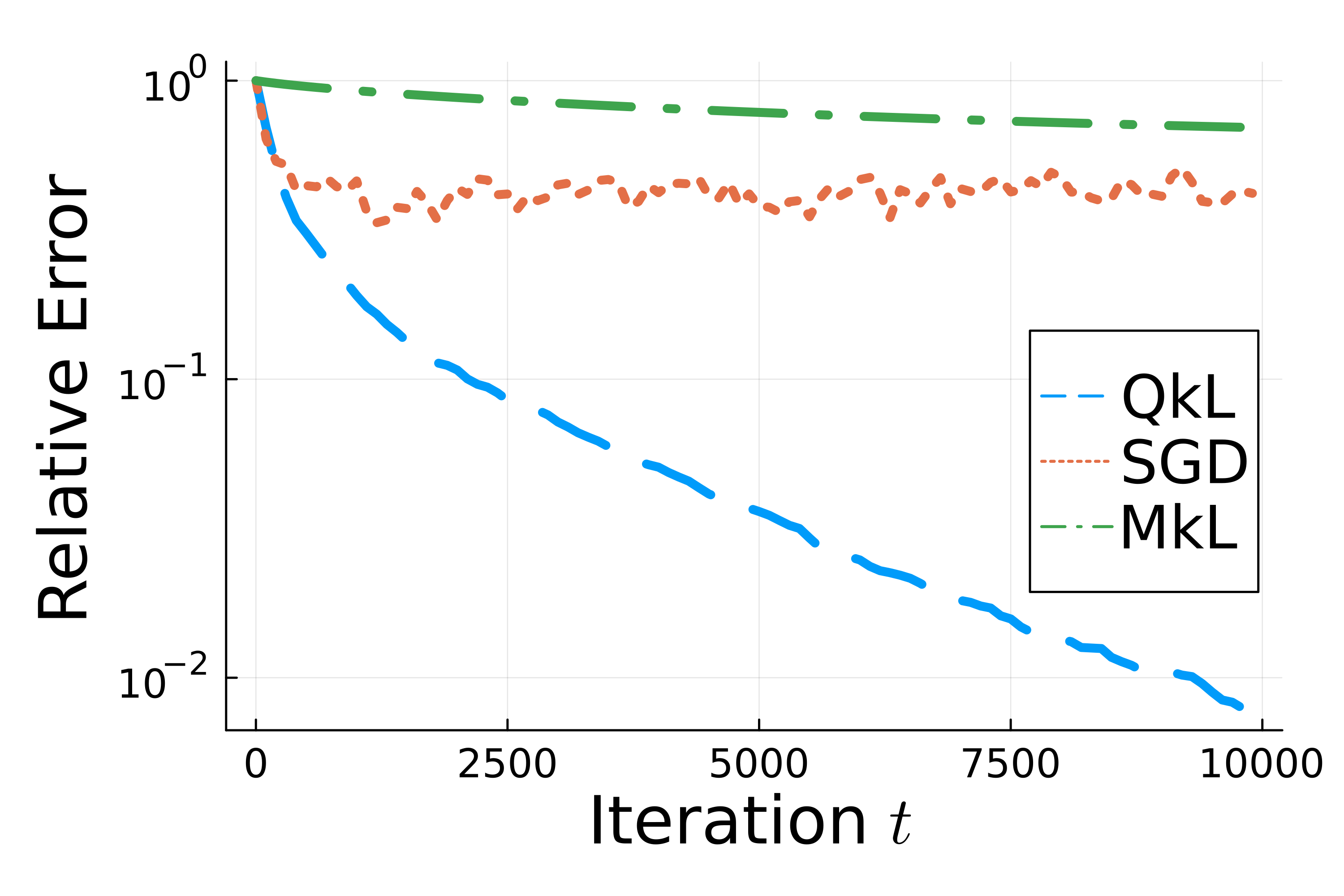}
\hfill\includegraphics[width=0.32\textwidth]{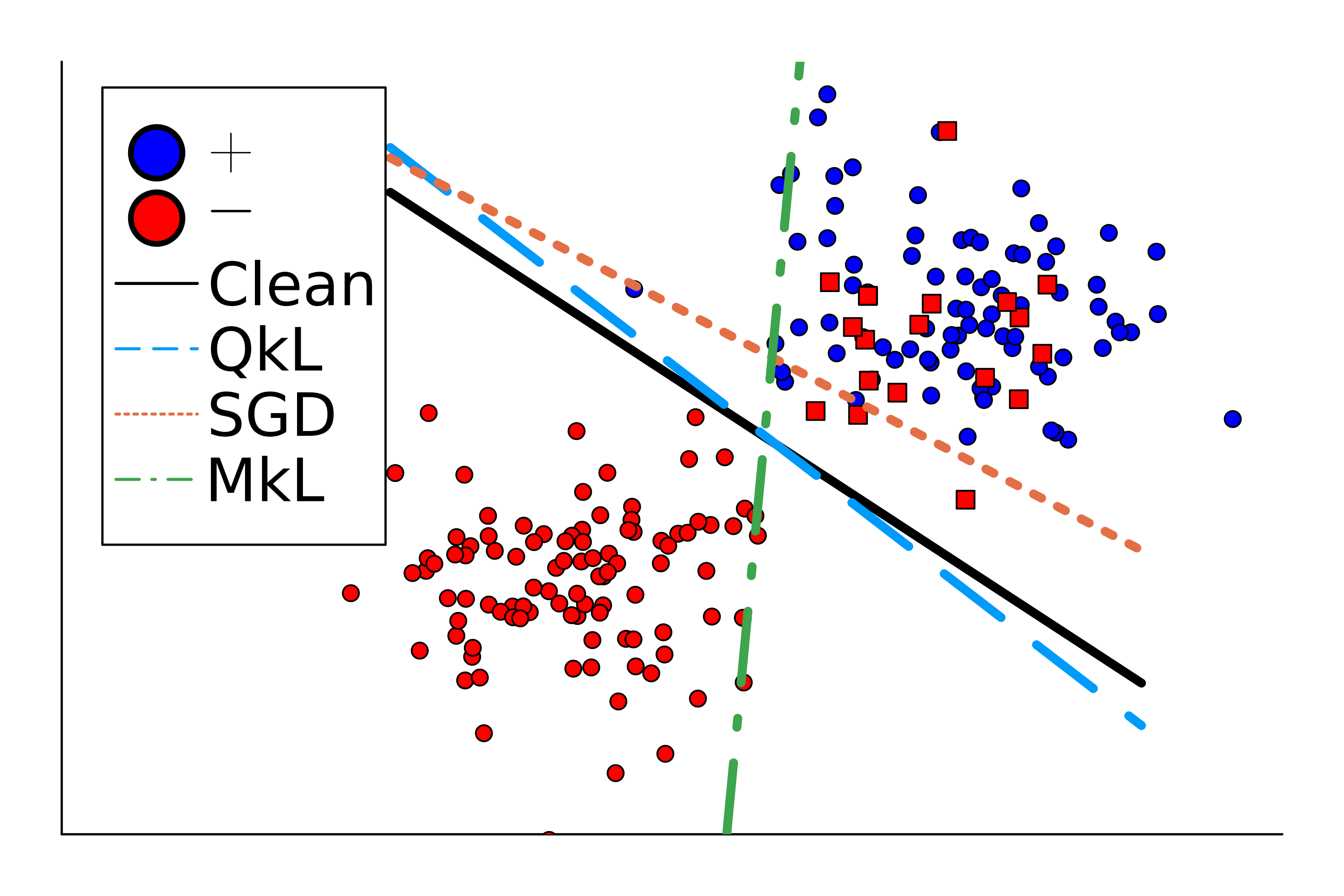}
    \caption{$\beta=0.1$ label corruption, with outliers sampled from the positive class. }
    \label{fig:figure-05c}
\end{subfigure}
\caption{QkL-SGD with $(k,q)=(m,1-\beta)$ and MkL-SGD with $k=m$ applied to regularized logistic regression~\eqref{eq:logistic_regression}. (Left) Relative good-set loss $F_G(\ve x_t)/F_G(\ve x_0)$ per iteration. (Center) Relative squared parameter error $\|\ve x_t-\ve x^\star_{\rm clean}\|_2^2/\|\ve x_0-\ve x^\star_{\rm clean}\|_2^2$. (Right) Learned decision boundaries at the final iterate, together with the clean-data reference boundary $\mathcal B_{\rm clean}$. Uncorrupted data samples are visualized as circles; corrupted data samples are visualized as squares.}
\end{figure}

\subsubsection{Varying $q$ with fixed corruption fraction $\beta$}

Our final experiment in this section examines the sensitivity of QkL-SGD to the choice of quantile parameter $q$ for the regularized logistic regression problem. For each corruption level $\beta \in \{0.1, 0.2, 0.3, 0.4\}$, we run QkL-SGD for 10,000 iterations with $q \in \{0.05, 0.1, \ldots, 1.0\}$ and record the resulting final relative loss and relative error. The underlying optimization problem is the regularized logistic regression objective~\eqref{eq:logistic_regression} using the data model described above, where a fraction $\beta$ of the samples, selected uniformly across the two classes, are corrupted. The results are presented in Figure~\ref{fig:figure-06}, with the left panel showing the final relative loss and the right panel showing the final relative error.  Relative loss and relative errors are averaged over 100 independent trials of 10000 iterations of QkL-SGD. As in polynomial regression applications, we see that the quantile $q$ that minimizes relative error and loss is approximately $1-\beta$.  In this case, we see that using a quantile $q$ less than this optimal choice produces better results than even slight overestimates of the optimal quantile.

\begin{figure}
    \includegraphics[width=0.45\textwidth]{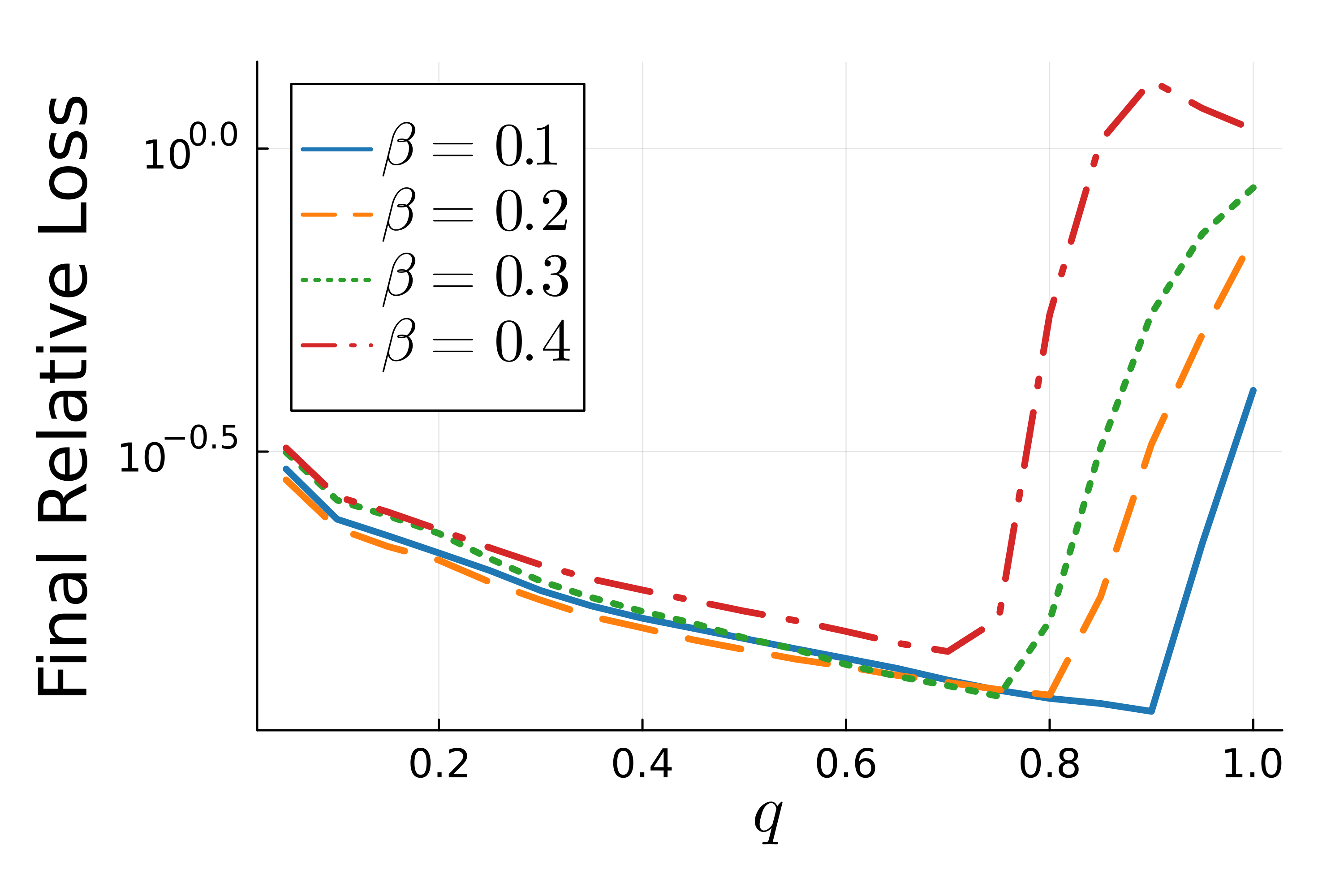}\hfill\includegraphics[width=0.45\textwidth]{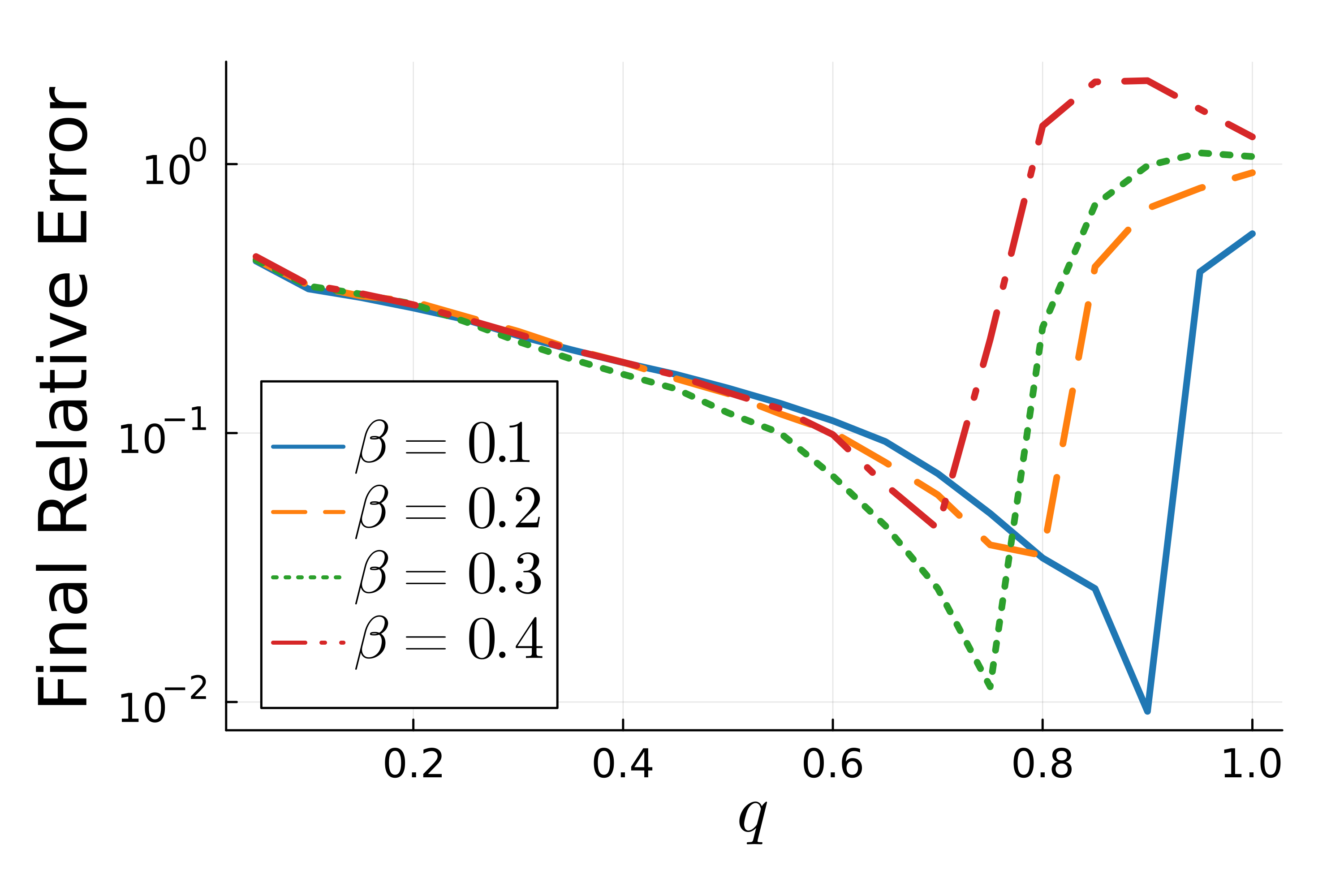}
    \caption{(Left) Final relative loss $F_G(\mathbf{x}_{10000})/F_G(\mathbf{x}_0)$ for QkL-SGD with $q \in (0,1]$ applied to the regularized logistic regression loss~\eqref{eq:logistic_regression} with fraction of corruption $\beta \in \{0.1, 0.2, 0.3, 0.4\}$; (right) final relative squared error $\|\mathbf{x}_{10000} - x^\star_{\rm clean}\|^2/\|\mathbf{x}_0 - x^\star_{\rm clean}\|^2$ for QkL-SGD with $q \in (0,1]$ (log scale). }\label{fig:figure-06}
\end{figure}

\subsection{Regularized Hinge Loss}\label{subsec:hinge_loss}

Next, we turn to the $\ell_2$-regularized hinge loss objective, which is the standard loss used in support vector machines. This loss is nonsmooth so Assumption~\ref{assump:sgd}.\ref{ass:component_smoothness} fails, unlike the other models considered.  Here, like with the noisy polynomial regression and logistic regression models, Assumption~\ref{assump:sgd}.\ref{ass:good_set_regularity} fails.
Like in Subsection~\ref{subsec:logistic_regression}, the experiments are designed to assess robustness to outliers, which are introduced by randomly flipping (corrupting) a fraction $\beta$ of the labels.

The $\ell_2$- regularized hinge loss model is typically employed as a supervised binary classification model where data points $\ve{d}_i \in \mathbb{R}^n$ with labels $l_i \in \{-1, +1\}$ are available.  The regularized hinge loss component objective for parameter vector $\ve{x} \in \mathbb{R}^n$ associated to a single data point and corresponding label $(\ve{d}_i, l_i)$ is
\begin{equation}
f_i(\ve{x}) = \max\{0, 1 - l_i \ve{d}_i^\top \ve{x}\} + \frac{\lambda}{m}\|\ve{x}\|_2^2, \label{eq:hingeloss_component}
\end{equation}

and the regularized hinge loss over all data points is
\begin{equation}
F(\ve{x}) = \frac1m \sum_{i=1}^m f_i(\ve{x}). \label{eq:hinge_loss}
\end{equation}

Here, the data generation is the same as in Subsection~\ref{subsec:logistic_regression}, with the exception that the labels $\ve{l} \in \{-1, +1\}^m$.  The positive class data points have corresponding binary label $l_i^{\rm clean} = +1$ and the negative class data points have corresponding binary label $l_i^{\rm clean} = -1$. This generation of data and labels again ensures that the data is linearly separable before label corruption, which follows the same process as in Subsection~\ref{subsec:logistic_regression} with the exception that for $i \in O$, the observed label is $l_i = -l_i^{\rm clean}$.

In the bottom plots below, we again visualize this data in two dimensions, taking the last two components of the data $(\ve{d}_i)_{2:}$.  We additionally plot the decision boundary of the learned classification models with parameter vector $\ve{x} \in \mathbb{R}^3$, $\{\ve{y} \in \mathbb{R}^2 : x_1 + \ve{y}^\top \ve{x}_{2:} = 0\}$.

As in Subsection~\ref{subsec:logistic_regression}, the aim in each experiment is to learn a parameter vector $\ve{x} \in \mathbb{R}^3$ that correctly classifies the data points with uncorrupted labels.  In all experiments, SGD, QkL-SGD, and MkL-SGD were run with step size parameter $\eta = 0.01$ and regularization parameter $\lambda = 1$.  We again let $\ve{x}^\star_{\rm clean}$ represent the iterate $\ve{x}_{10000}$ learned by 10000 iterations of SGD on the linearly separable dataset before any labels are flipped.  We plot the boundary corresponding to this parameter vector, $$\mathcal B_{\rm clean}
=
\{\,\ve y\in \mathbb R^2 :
(\ve{x}^\star_{\rm clean})_1
+
\ve y^\top(\ve{x}^\star_{\rm clean})_{2:}=0
\,\}.$$  This boundary is labeled ``Clean" in the following plots.

\paragraph{Assumptions for regularized hinge loss}

The components of the regularized hinge loss~\eqref{eq:hingeloss_component} are convex and nonnegative.  However, unlike the other model settings, they are not smooth.  Due to the regularization, by the definition of strongly convex functions in terms of subgradients, for every nonempty subset $S \subseteq G$, we have that $F_S$ is $\lambda/m$-strongly convex.  Thus, Assumptions~\ref{assump:sgd}.\ref{ass:outlier_nonnegativity} and~\ref{assump:sgd}.\ref{ass:subset_strong_convexity} hold for this model, but Assumptions~\ref{assump:sgd}.\ref{ass:good_set_regularity} and~\ref{assump:sgd}.\ref{ass:component_smoothness} fail.

\subsubsection{Various outlier settings}

We now explore the behavior of QkL-SGD on the hinge loss~\eqref{eq:hinge_loss} with $q = 1 - \beta$ in a variety of outlier settings.

\paragraph{Few outliers from a single class (\texorpdfstring{$\beta = 0.01$, $q=0.99$, $k = m$}{beta = 0.01, q = 0.99, k = m})}

Our first experiment, with results presented in Figure~\ref{fig:figure-07a}, explores the setting where 1\% of the labels, all from the same class, are corrupted (flipped). We compare QkL-SGD with $q = 0.99$ and $k = m$ to SGD and MkL-SGD with $k = m$.  We note that although QkL-SGD only slightly outperforms SGD in terms of relative loss $F_G(\ve{x}_t)/F_G(\ve{x}_0)$ (left plot), it yields a more significantly lower relative error $\|\ve{x}_t - \ve{x}^\star_{\rm clean}\|_2^2/\|\ve{x}_0 - \ve{x}^\star_{\rm clean}\|_2^2$ (center plot) and a cleaner decision boundary (right plot). MkL-SGD has increasing relative loss and stagnating relative error due to difficulties in separating the effects of the good and outlier sets.

\paragraph{Moderate outliers from both classes, (\texorpdfstring{$\beta = 0.05$, $q=0.95$, $k = m$}{beta = 0.05, q = 0.95, k = m})}

In our next experiment, presented in Figure~\ref{fig:figure-07b}, 5\% of the labels are flipped, selected uniformly from both classes. We see that QkL-SGD with $q=0.95$ and $k = m$ substantially outperforms SGD and MkL-SGD with $k = m$ in terms of both relative loss (left plot) and relative error (center plot) convergence.  In this case, the decision boundaries learned by SGD and QkL-SGD (right plot) are not significantly visually different, due to possible larger deviations between the bias terms in the learned parameter vectors which are more difficult to differentiate in this visualization.  Again, MkL-SGD has increasing relative loss and stagnating relative error.

\paragraph{Moderate outliers from a single class (\texorpdfstring{$\beta = 0.05$, $q=0.95$, $k = m$}{beta = 0.05, q = 0.95, k = m})}

We now consider a more structured form of label corruption where corruptions occur in only the positive class.  In the experiment presented in Figure~\ref{fig:figure-07c}, we corrupt 5\% of the labels, but only from the positive class. QkL-SGD with $q = 0.95$ and $k = m$ again demonstrates robustness, yielding lower relative loss (left plot) and smaller relative error (center plot) than SGD and MkL-SGD.  In the right plot, the QkL-SGD decision boundary is not visibly significantly better than that of SGD, again likely due to differences in the bias term. Again, MkL-SGD has increasing relative loss and stagnating relative error.

\begin{figure}
\begin{subfigure}{\textwidth}
    \centering
    \includegraphics[width=0.3\textwidth]{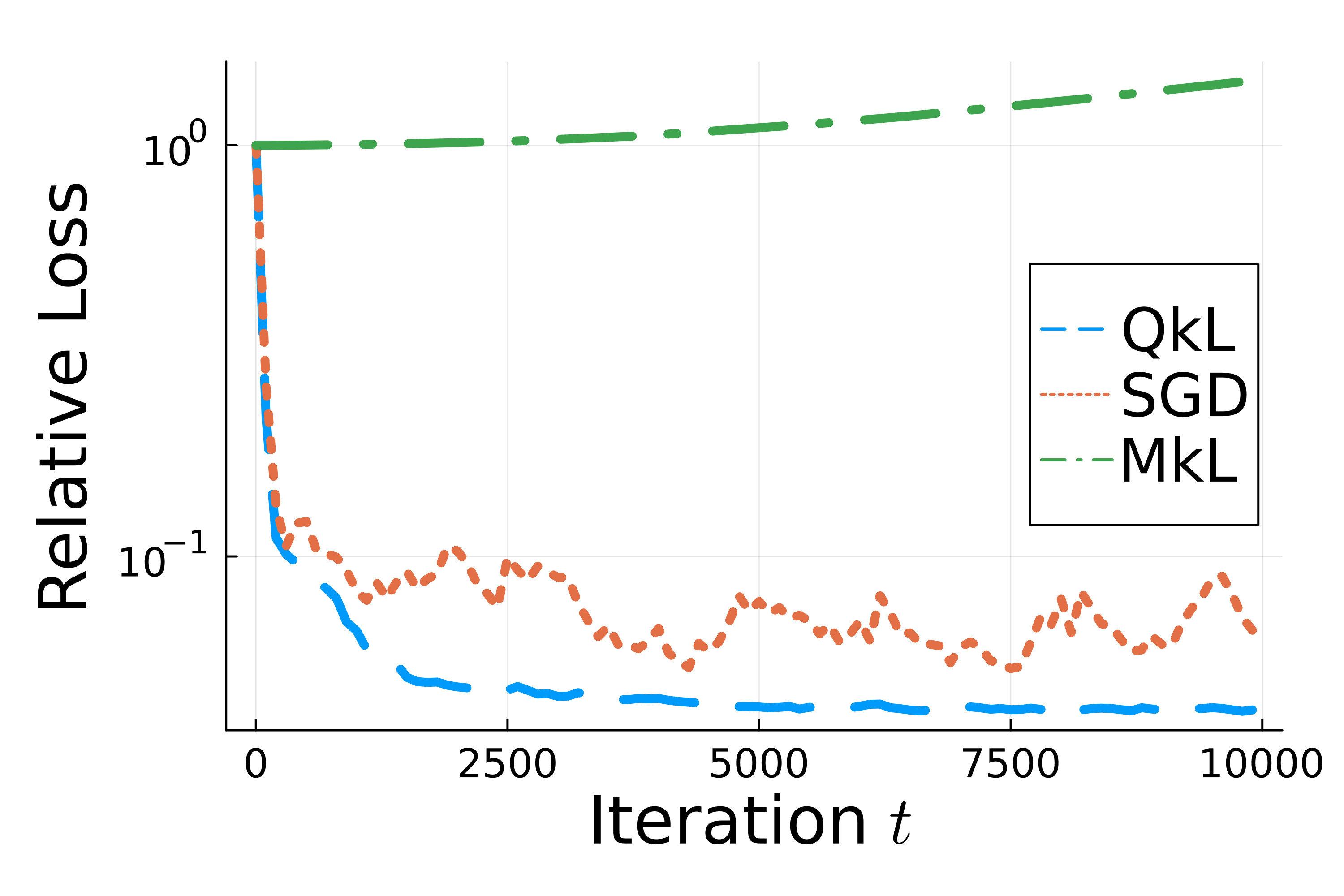}
\quad\includegraphics[width=0.3\textwidth]{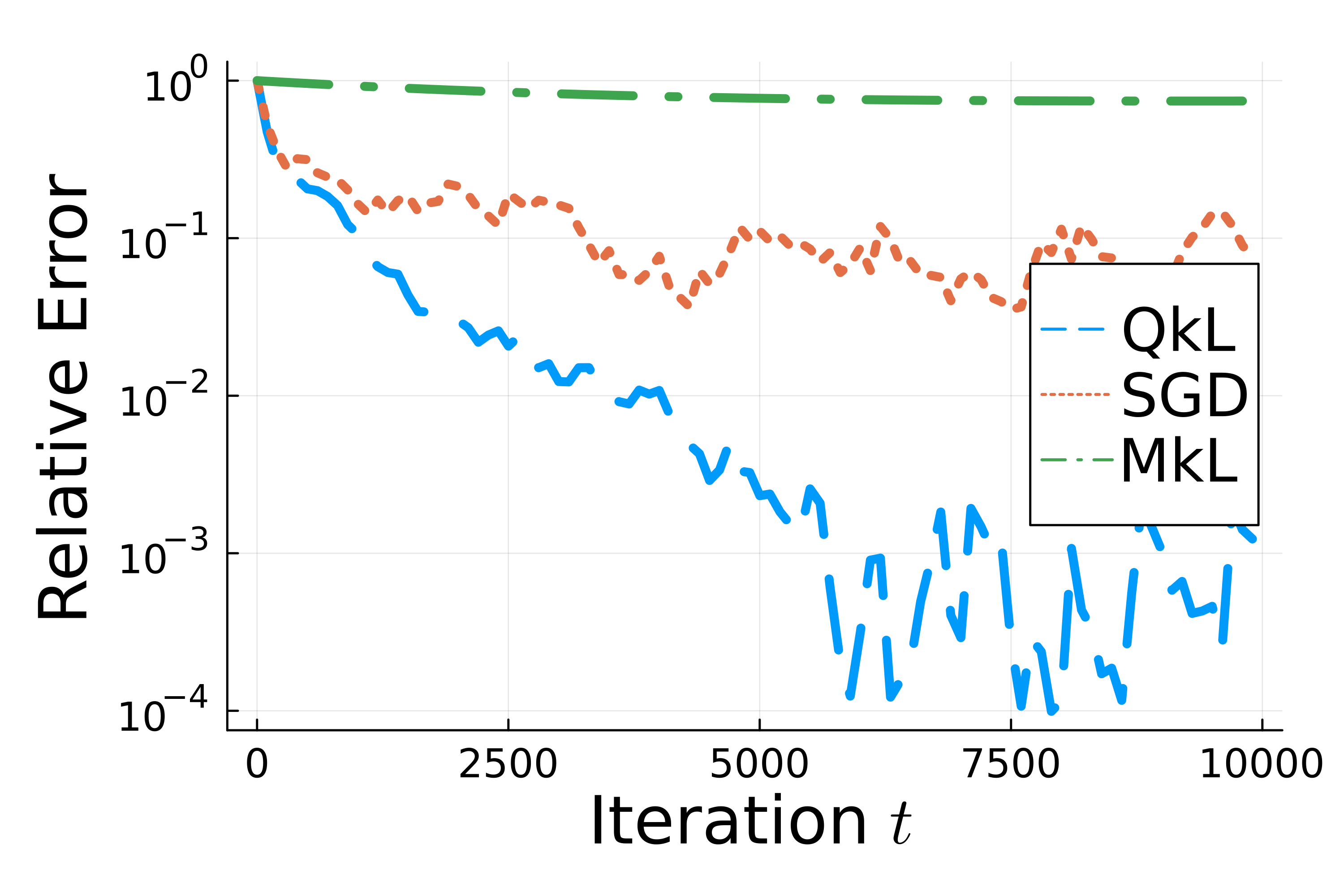}\hfill\includegraphics[width=0.32\textwidth]{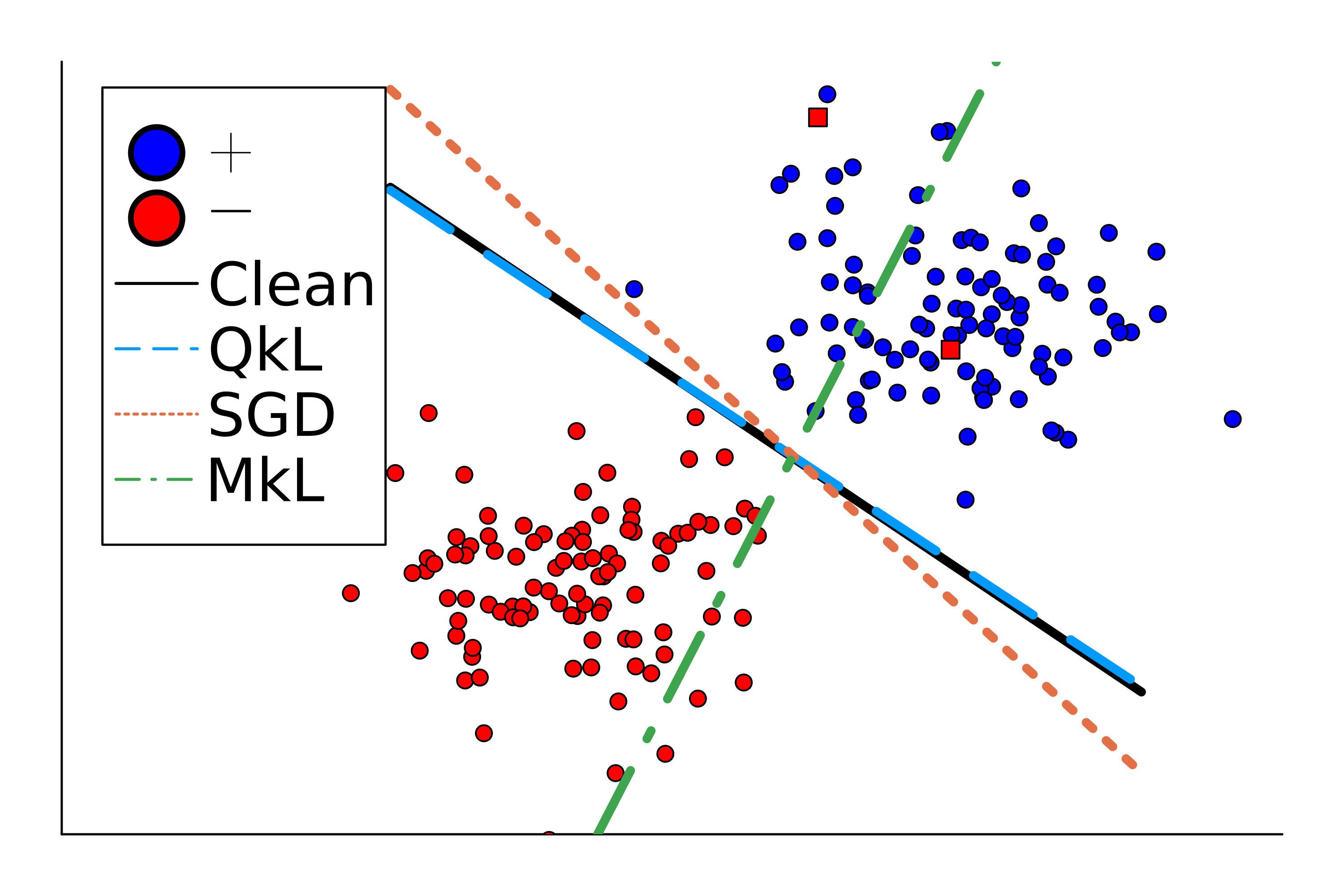}
    \caption{$\beta=0.01$ label corruption, with outliers sampled from the positive class.}
    \label{fig:figure-07a}
\end{subfigure}

\begin{subfigure}{\textwidth}
    \centering
    \includegraphics[width=0.3\textwidth]{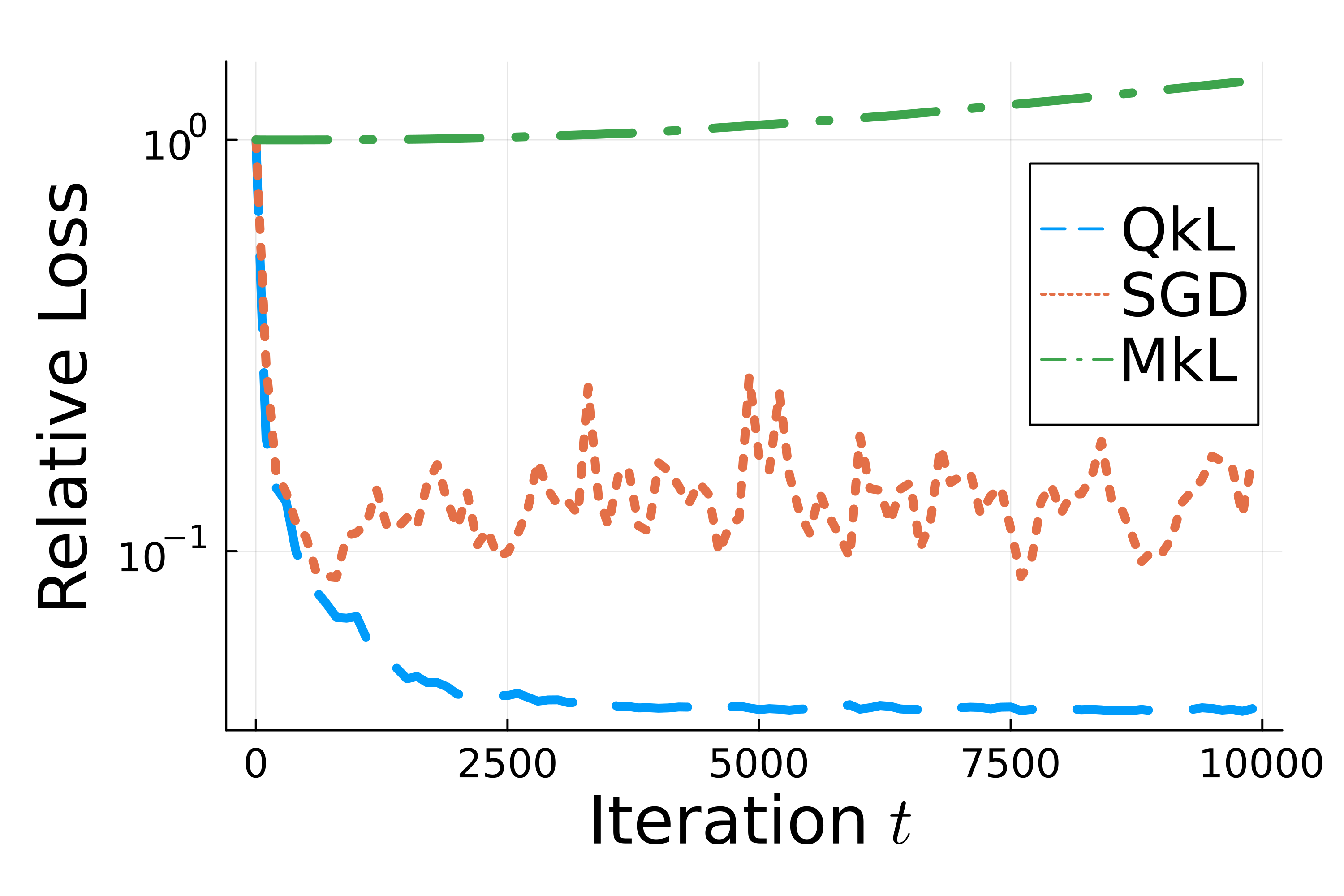}
\quad\includegraphics[width=0.3\textwidth]{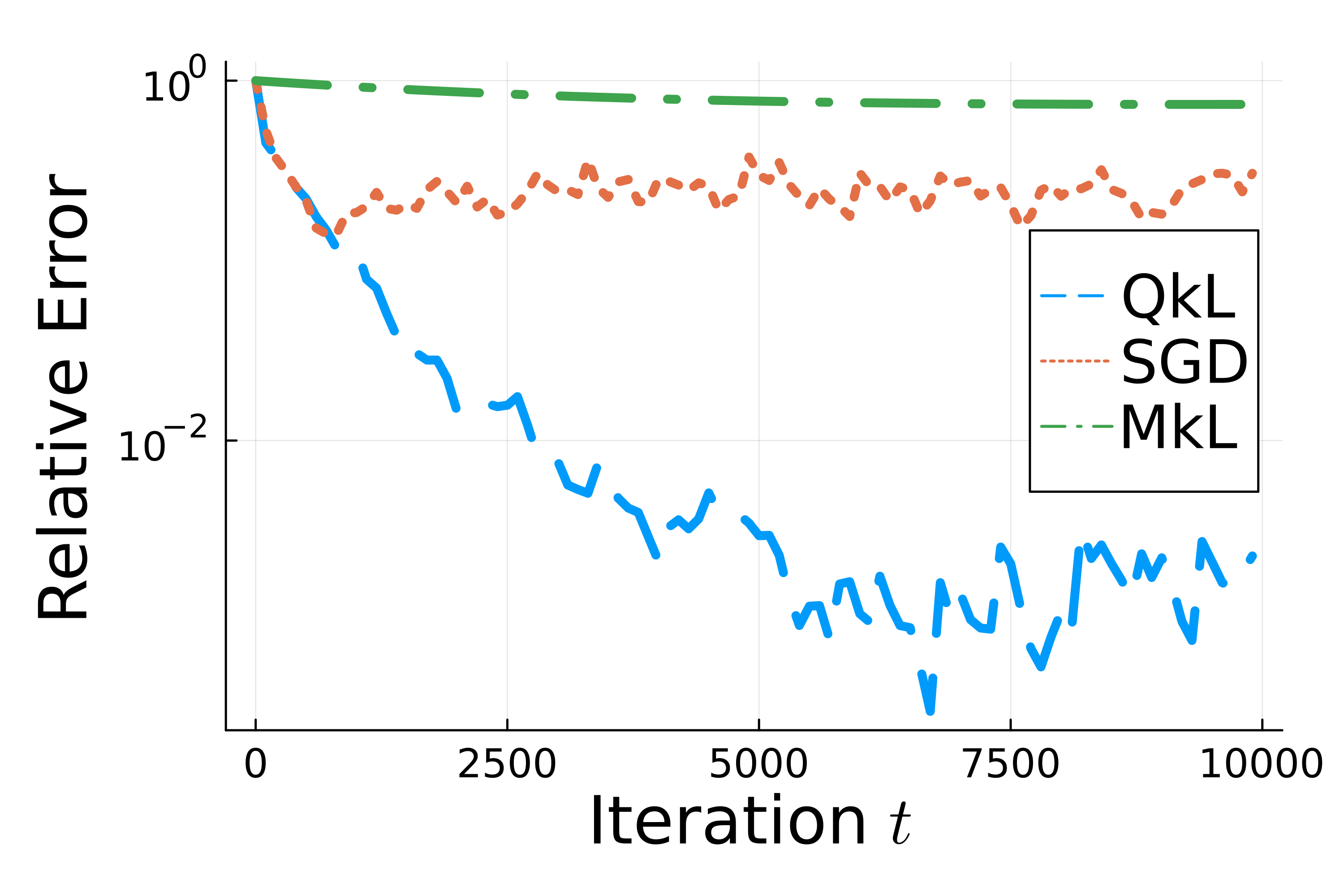}
\hfill\includegraphics[width=0.32\textwidth]{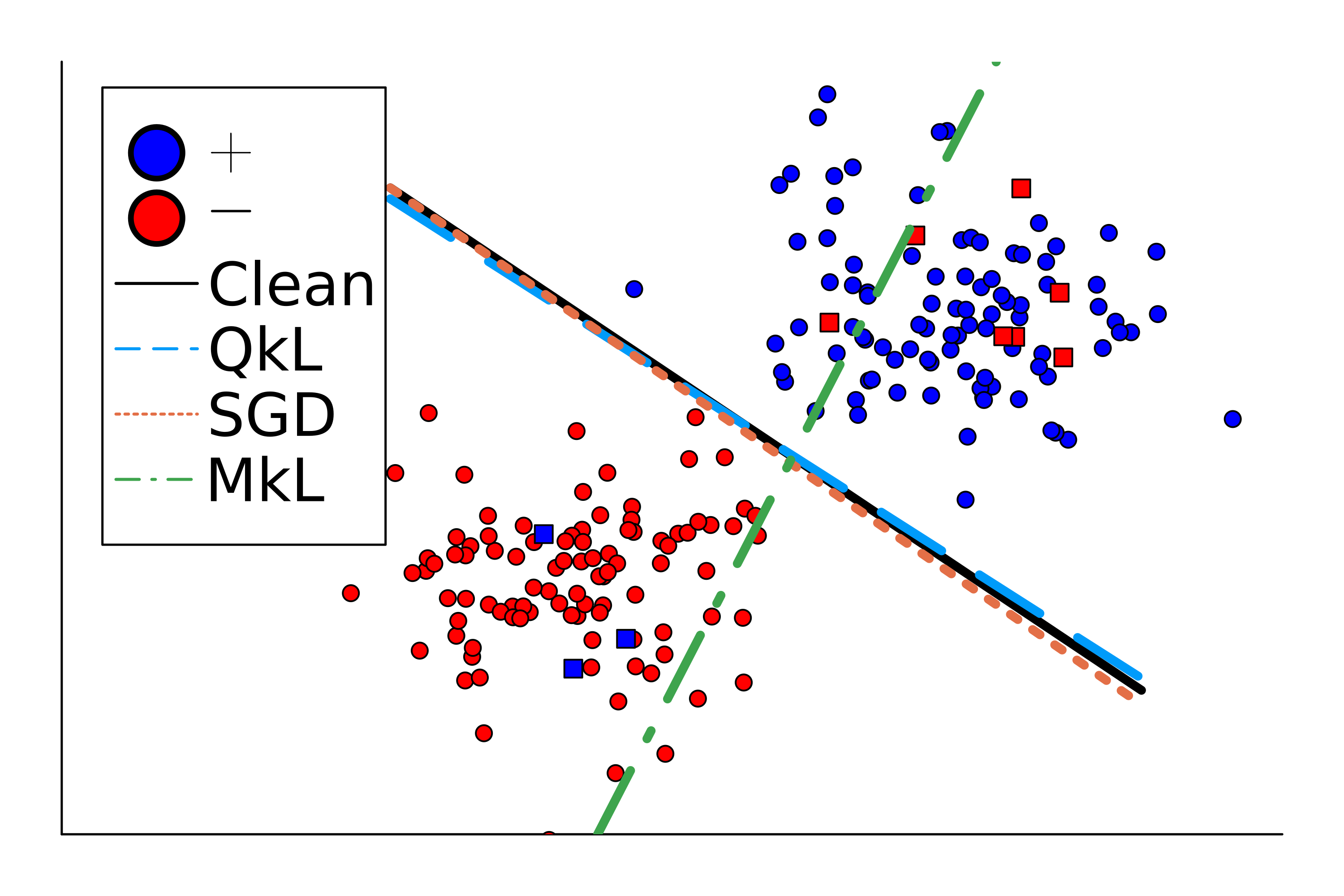}
    \caption{$\beta=0.05$ label corruption, with points sampled from both classes.}
    \label{fig:figure-07b}
\end{subfigure}

\begin{subfigure}{\textwidth}
    \centering
    \includegraphics[width=0.3\textwidth]{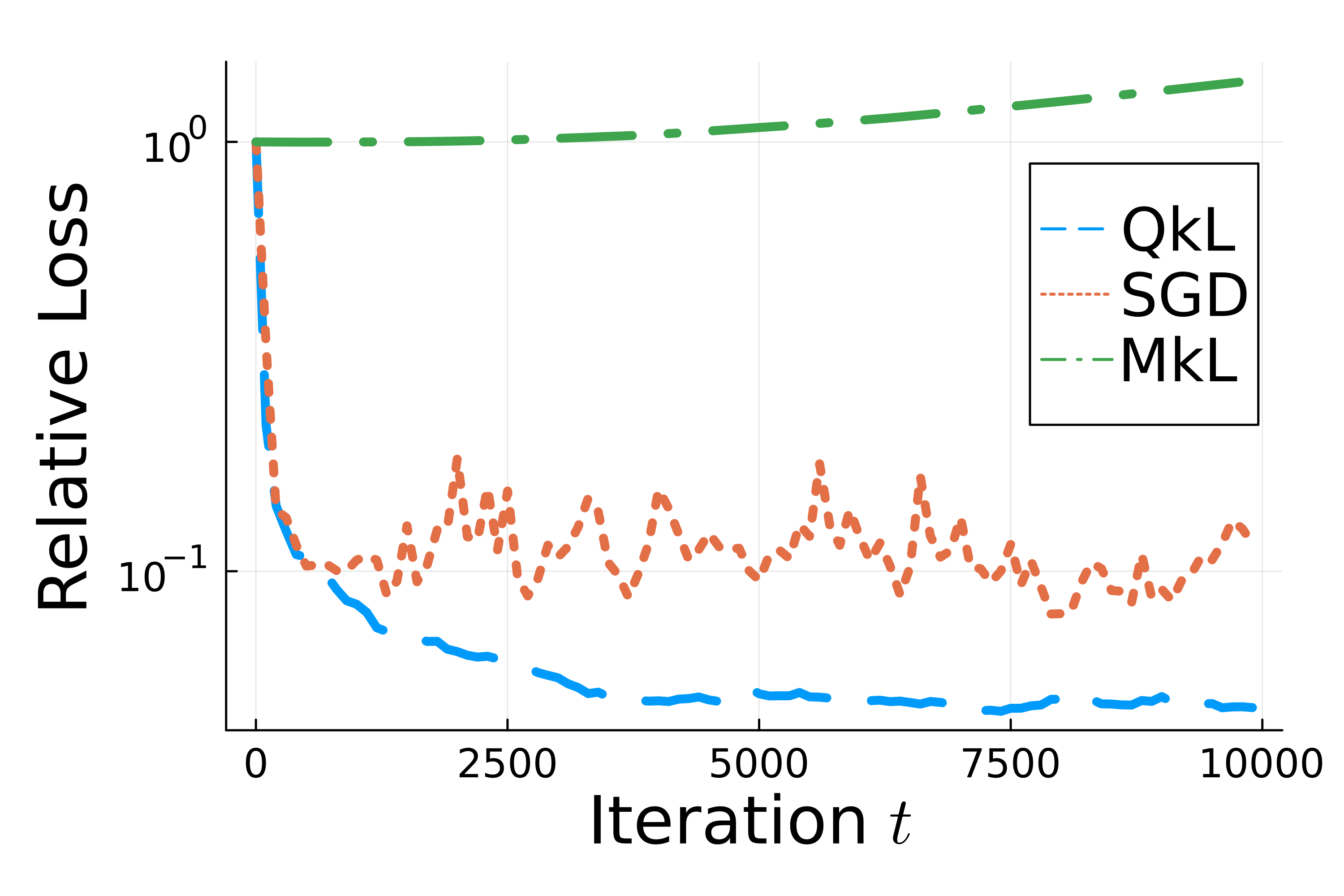}
\quad\includegraphics[width=0.3\textwidth]{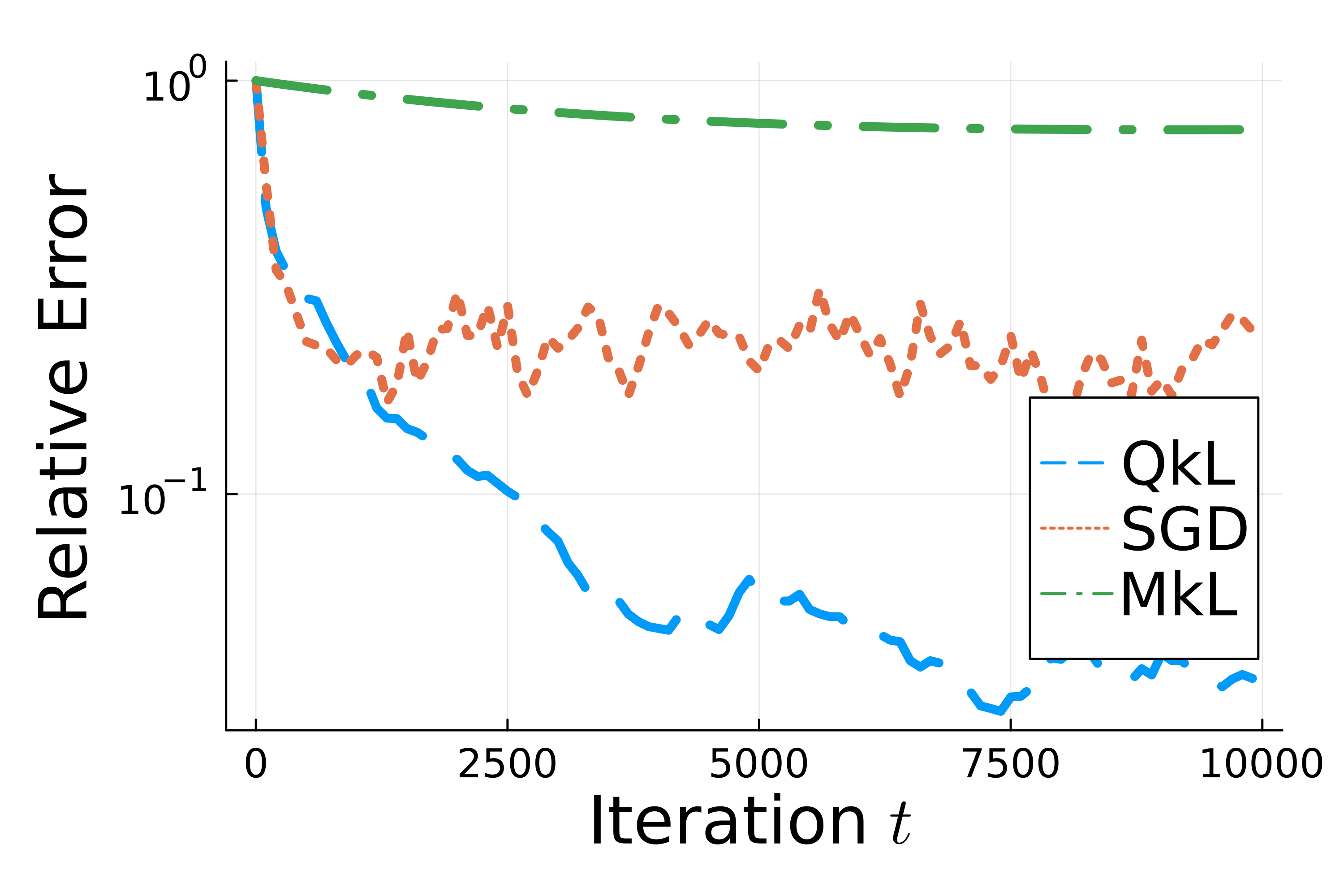} \hfill\includegraphics[width=0.32\textwidth]{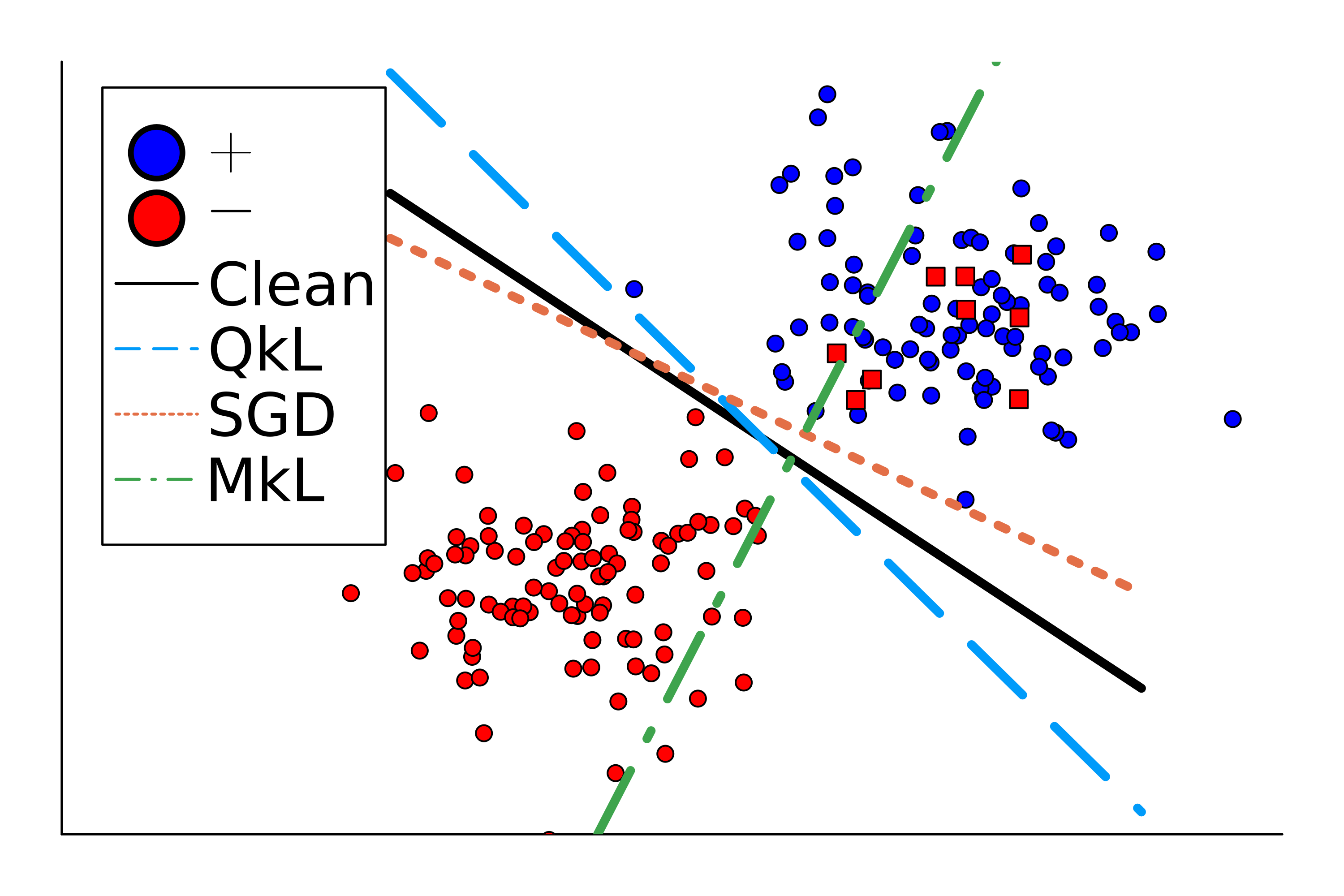}
    \caption{$\beta=0.05$ label corruption, with points sampled from positive class.}
    \label{fig:figure-07c}
\end{subfigure}

   \caption{QkL-SGD with $(k,q)=(m,1-\beta)$ and MkL-SGD with $k=m$ applied to regularized hinge loss~\eqref{eq:hinge_loss}. (Left) Relative good-set loss $F_G(\ve x_t)/F_G(\ve x_0)$ per iteration. (Center) Relative squared parameter error $\|\ve x_t-\ve x^\star_{\rm clean}\|_2^2/\|\ve x_0-\ve x^\star_{\rm clean}\|_2^2$. (Right) Learned decision boundaries at the final iterate, together with the clean-data reference boundary $\mathcal B_{\rm clean}$. Uncorrupted data samples are visualized as circles; corrupted data samples are visualized as squares.}
\end{figure}

\subsubsection{Varying $q$ with fixed corruption fraction $\beta$}

Figure~\ref{fig:figure-08}  explores how different quantile values $q$ affect QkL-SGD under a fixed corruption fraction $\beta$.  We apply QkL-SGD with $k = m$ to the regularized hinge loss~\eqref{eq:hinge_loss} applied to the data model described above where $\beta \in \{0.1, 0.2, 0.3, 0.4\}$ fraction of data sampled uniformly across both classes are corrupted, and plot the final relative loss (left plot) and final relative error (right plot) after 10000 iterations of QkL-SGD with $q \in \{0.05, 0.1, \cdots, 1.0\}.$  Relative loss and relative errors are averaged over 100 independent trials of 10000 iterations of QkL-SGD.  Again, in this case, we see that the quantile $q$ producing the lowest relative error and loss values is approximately $1-\beta$.  However, in this case, we see that QkL-SGD produces lower final relative loss and error when the corruption fraction $\beta$ is smaller.

\begin{figure}
    \includegraphics[width=0.45\textwidth]{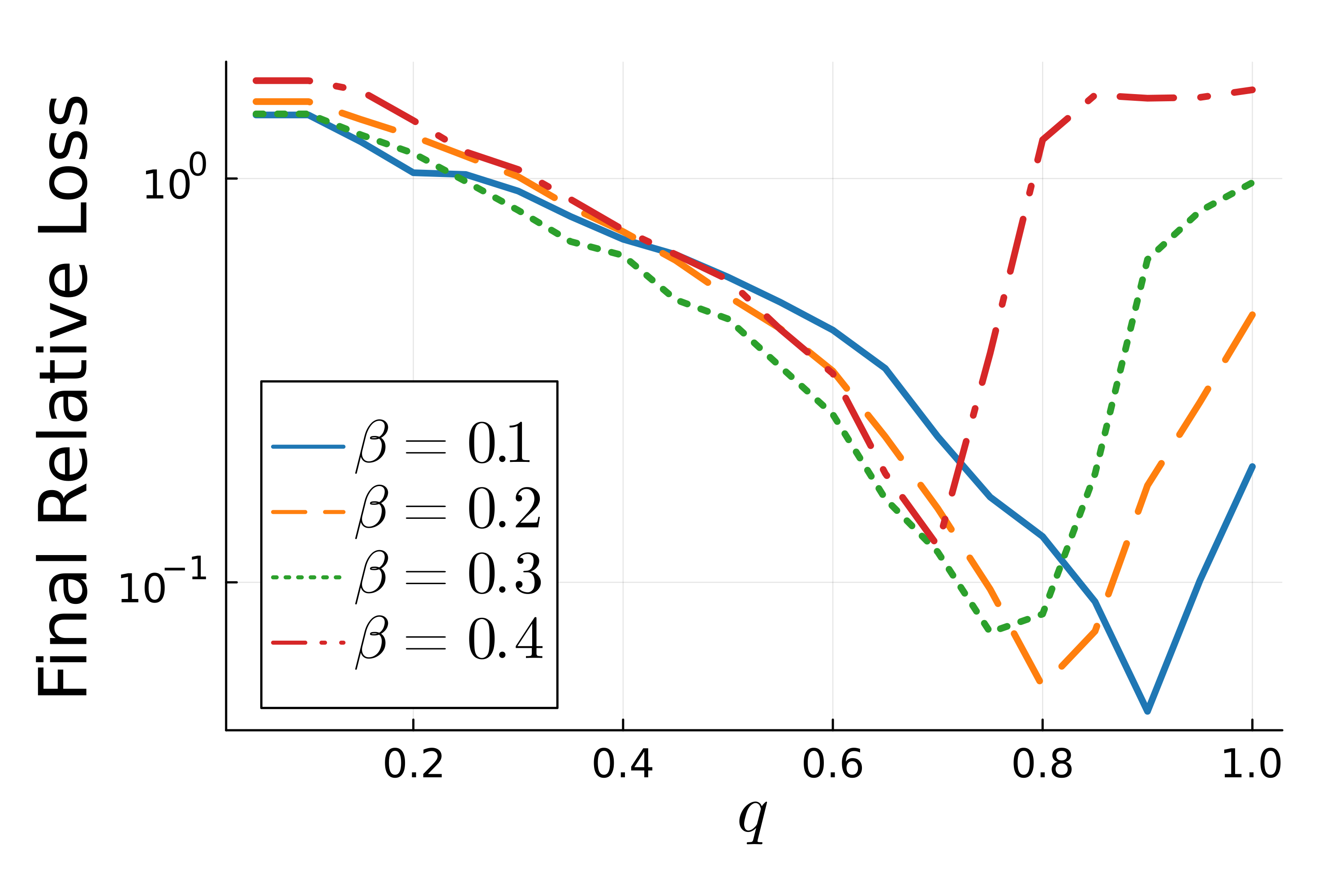}\hfill\includegraphics[width=0.45\textwidth]{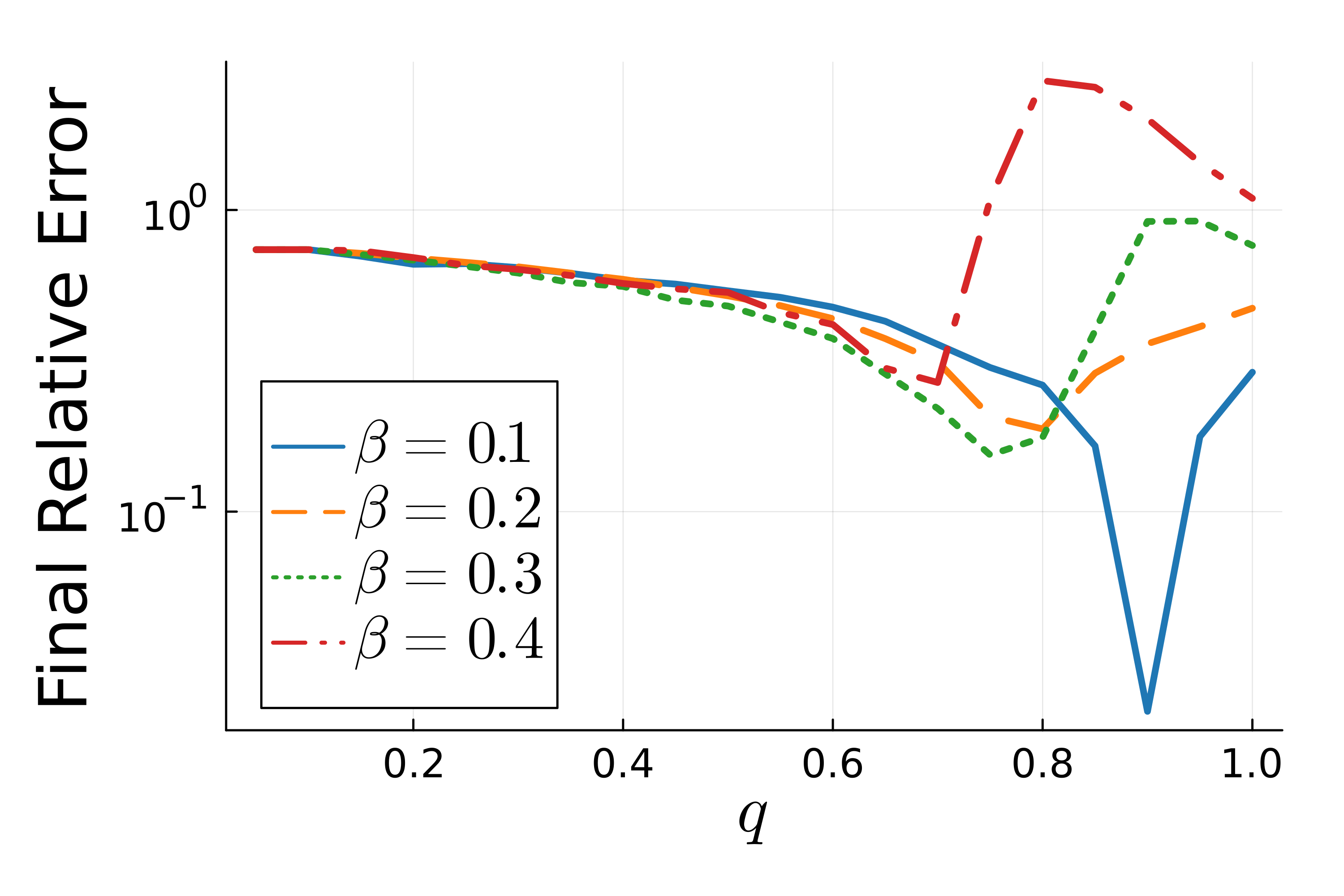}
    \caption{(Left) Final relative loss $F_G(\mathbf{x}_{10000})/F_G(\mathbf{x}_0)$ for QkL-SGD with $q \in (0,1]$ applied to the regularized hinge loss~\eqref{eq:hinge_loss} with fraction of corruption $\beta \in \{0.1, 0.2, 0.3, 0.4\}$; (right) final relative squared error $\|\mathbf{x}_{10000} - \ve{x}^\star_{\rm clean}\|^2/\|\mathbf{x}_0 - \ve{x}^\star_{\rm clean}\|^2$ for QkL-SGD with $q \in (0,1]$ (log scale). }\label{fig:figure-08}
\end{figure}

\subsection{Varying Sample Size $k$}\label{subsec:vary_k}

We now explore the effect of the sample size $k$ on the behavior of QkL-SGD and MkL-SGD.  We first consider the case where $k$ is a fraction of $m$, denoted $\alpha$ in Subsection~\ref{subsec:sample_fraction}.  We then compare the convergence of QkL-SGD where $k=m$ and MkL-SGD where $k$ is a small constant value \emph{as a function of the total number of loss evaluations} in Subsection~\ref{subsec:small_sample}.

\subsubsection{Sample Fraction \texorpdfstring{$\alpha$}{alpha}}\label{subsec:sample_fraction}

In this section, we explore how the sample size parameter $\alpha$ affects the convergence of QkL-SGD and MkL-SGD. Figure~\ref{fig:figure-09a} compares the relative loss and relative error of QkL-SGD and MkL-SGD with sampling rates $\alpha \in \{0.1, 0.5, 0.9\}$ on the noiseless polynomial regression problem~\eqref{eq:poly_regression}. We generate data as described in Subsection~\ref{subsec:noiseless_poly_regression}; starting from the uncorrupted dataset, we corrupt 5\% of the target values $y_i$ using additive noise drawn from $\mathcal{N}(0,100)$.  For this simple problem, we see very little difference in the behavior of MkL-SGD under different sampling rates.  QkL-SGD, on the other hand, suffers from increased variance in both loss and error as the sampling rate decreases; the method is likely occasionally sampling a small corrupted component.  Regardless, the relative error and relative loss of all QkL-SGD decrease significantly over all iterations.

In Figure~\ref{fig:figure-09b}, we plot the loss and error over iterations of QkL-SGD and MkL-SGD applied to the polynomial regression loss~\eqref{eq:poly_regression} with sampling rate $\alpha \in \{0.1, 0.5, 0.9\}$.  We generate the ideal (uncorrupted) data as described in Section~\ref{subsec:poly_regression} and then sample 5\% of the data and corrupt the corresponding targets $y_i$ by large additive noise sampled from $\mathcal{N}(0,100)$.  In this experiment, we again note that MkL-SGD performs better with smaller sampling rate, while QkL-SGD again achieves smaller variance with larger sampling rate.

In the experiment presented in Figure~\ref{fig:figure-09c}, we generate the uncorrupted
data as in Section~\ref{subsec:logistic_regression}.  We then sample 5\% of the data points for corruption and their labels are flipped.  We implement QkL-SGD and MkL-SGD on the regularized logistic regression loss~\eqref{eq:logistic_regression} with sampling rate $\alpha \in \{0.1,0.5,0.9\}$ and compare the relative loss and relative error across iterations.  We note that MkL-SGD suffers when the sample size is large, as selecting the minimum loss components can actually cause the loss and error to increase.  On the other hand, QkL-SGD can suffer when the sample size is too small, although the improvement in both loss and error is much better than that of MkL-SGD.

In the experiment given in Figure~\ref{fig:figure-09d}, we generate the data as in Section~\ref{subsec:hinge_loss} and then sample 5\% of the data and flip the labels associated to these data points.  We compare the relative loss and relative error of QkL-SGD and MkL-SGD applied to the regularized hinge loss objective~\eqref{eq:hinge_loss} with sample rates $\alpha \in \{0.1, 0.5, 0.9\}$.  Again, MkL-SGD suffers when the sampling rate is too large, while QkL-SGD suffers when the sampling rate is too small, although QkL-SGD achieves significantly lower values of loss and error.

\begin{figure}
\begin{subfigure}{\textwidth}
    \centering
    \includegraphics[width=0.48\textwidth]{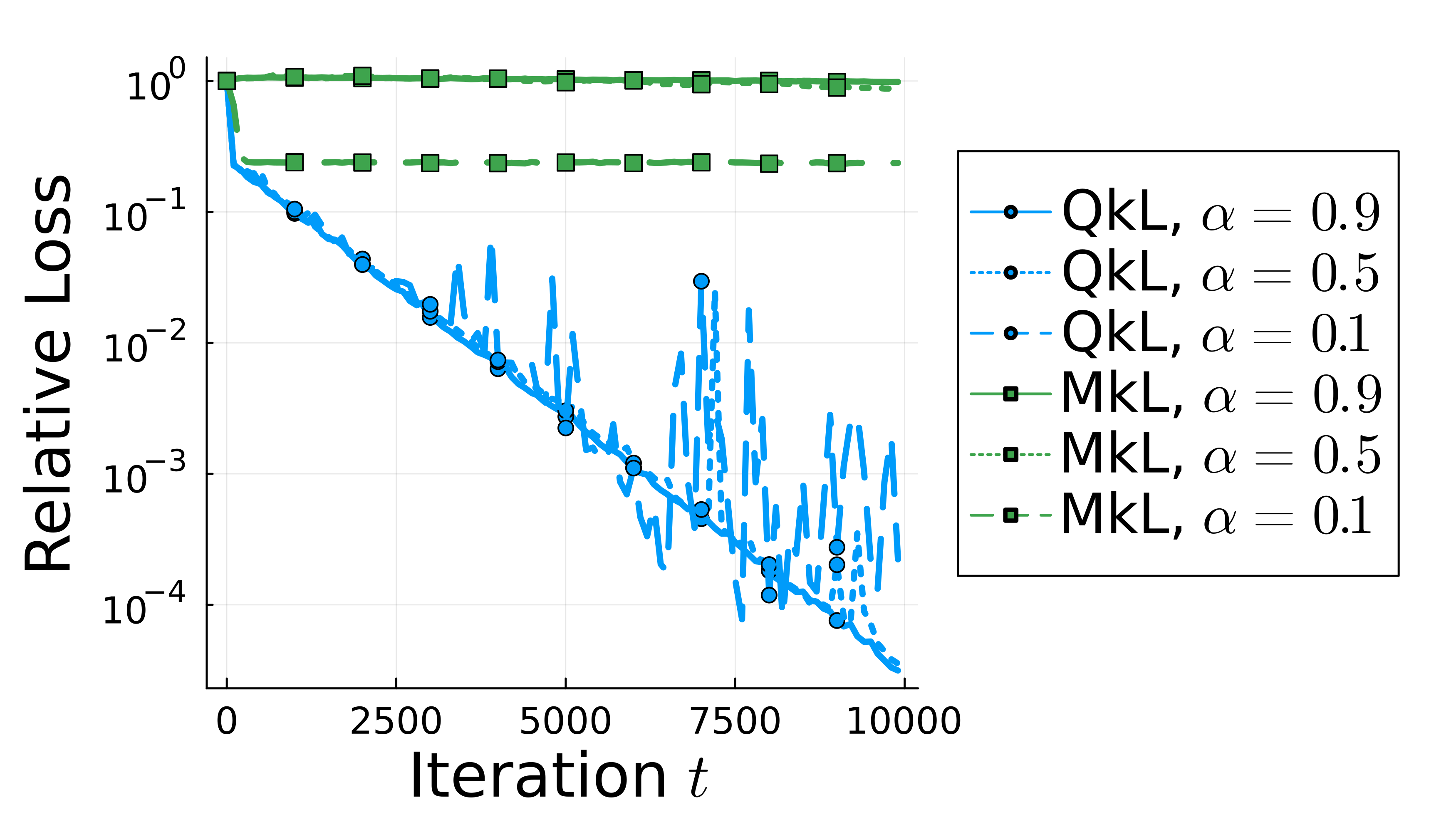}
    \hfill\includegraphics[width=0.48\textwidth]{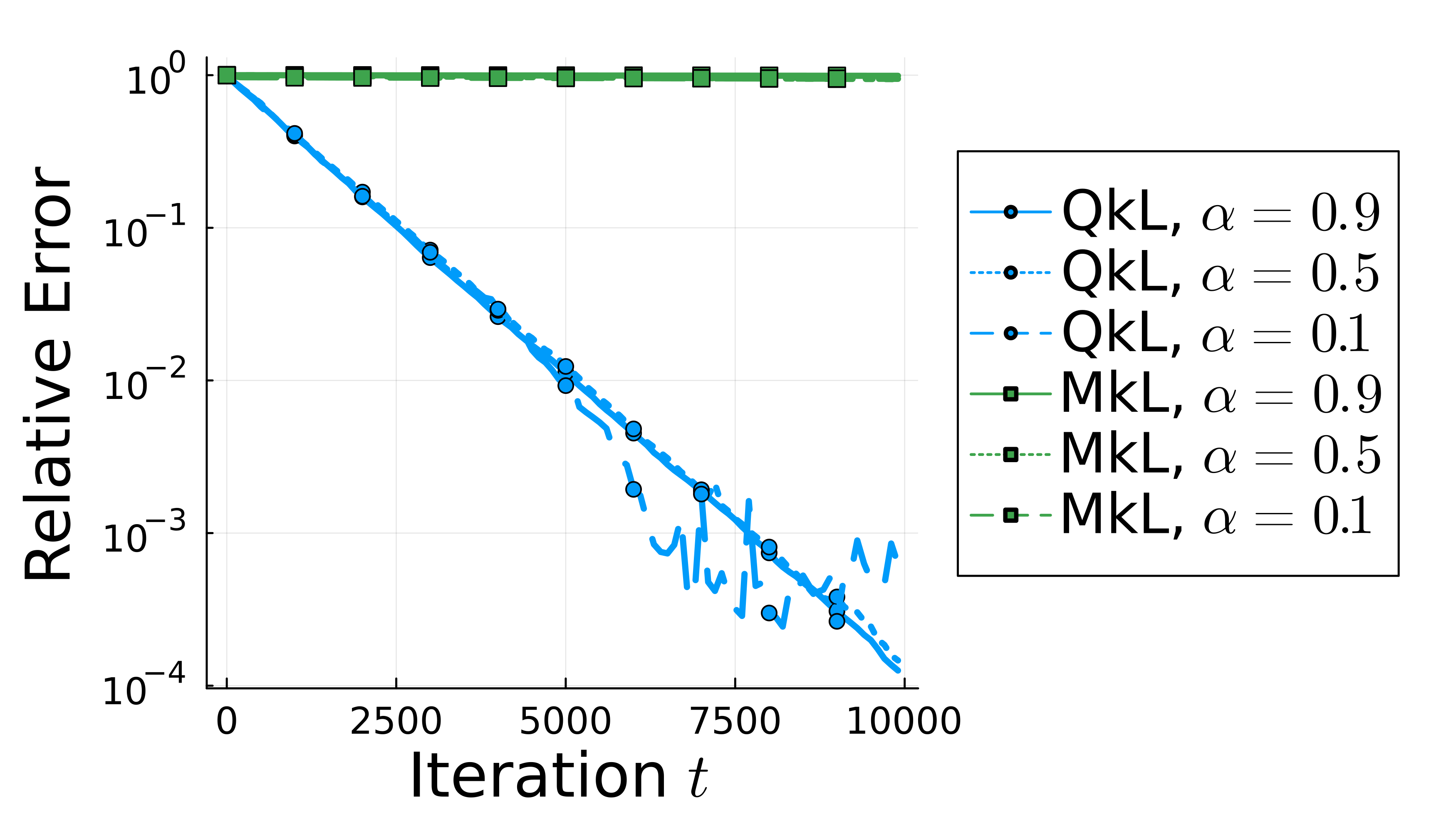}
    \caption{Polynomial regression objective~\eqref{eq:poly_regression} on noiseless uncorrupted data as in Subsection~\ref{subsec:noiseless_poly_regression}.}
\label{fig:figure-09a}
\end{subfigure}

\begin{subfigure}{\textwidth}
    \centering
    \includegraphics[width=0.48\textwidth]{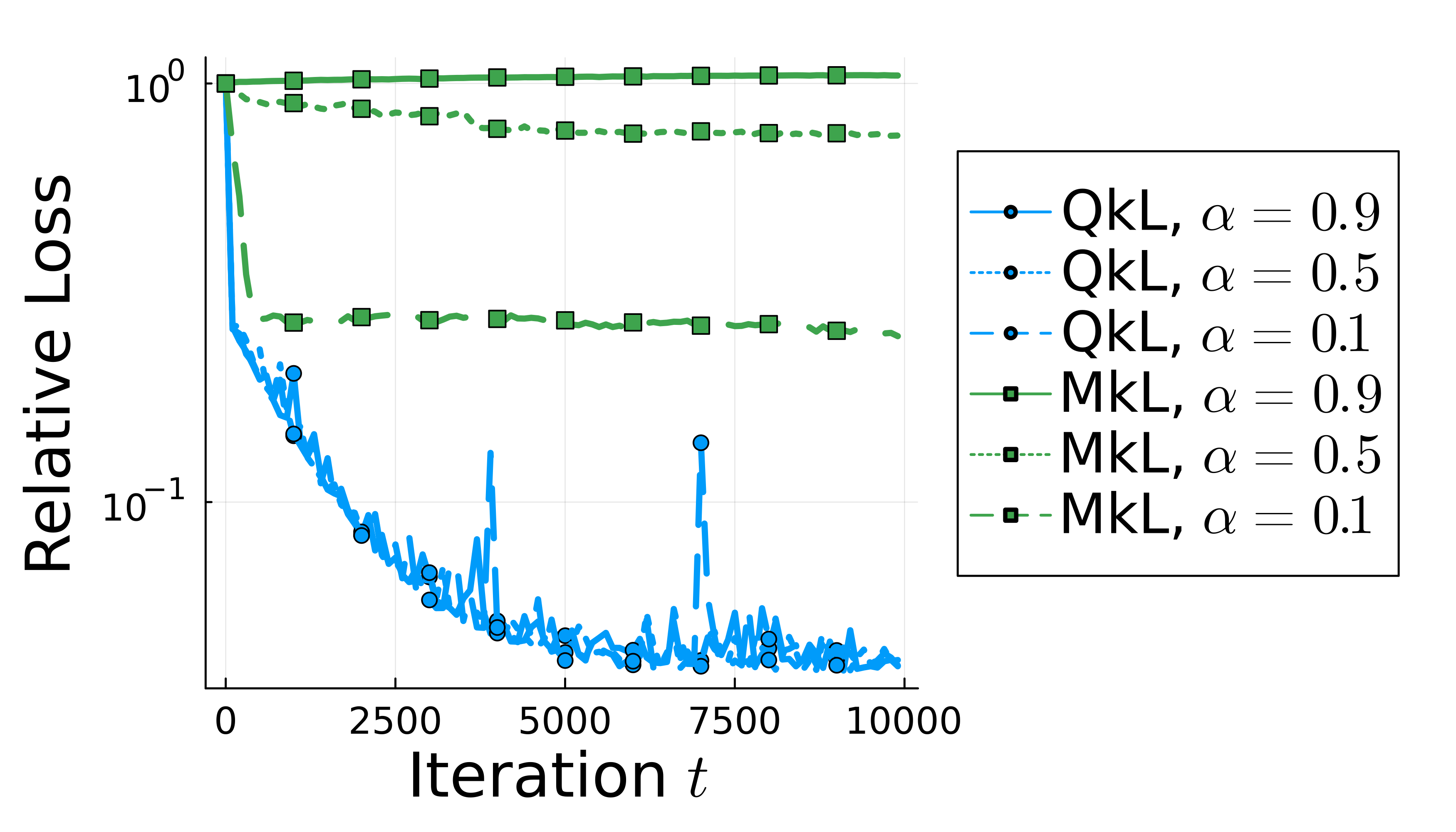}
    \hfill\includegraphics[width=0.48\textwidth]{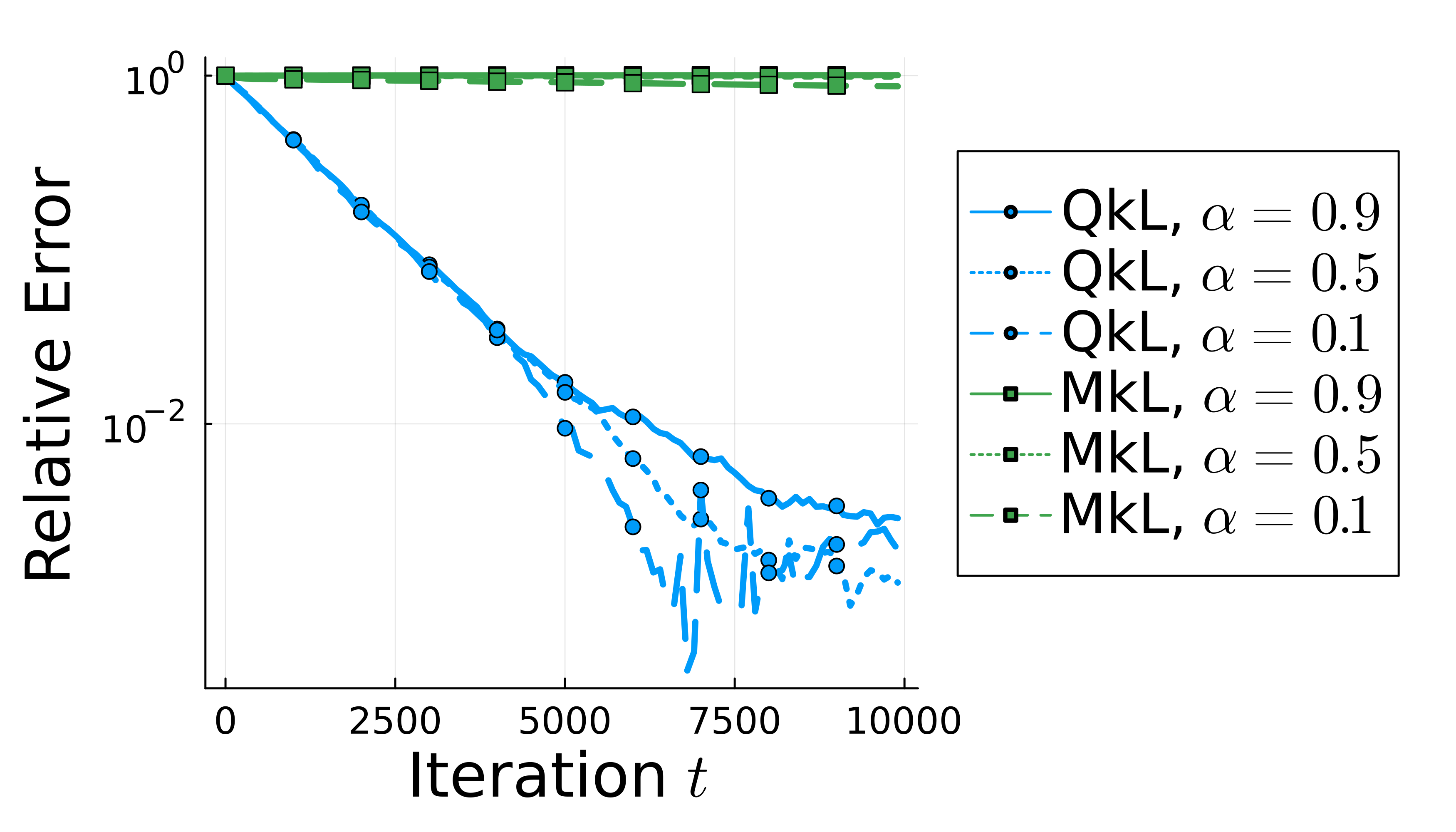}
    \caption{Polynomial regression objective~\eqref{eq:poly_regression} on noisy uncorrupted data as in Subsection~\ref{subsec:poly_regression}.}
    \label{fig:figure-09b}
\end{subfigure}

\begin{subfigure}{\textwidth}
    \centering
    \includegraphics[width=0.48\textwidth]{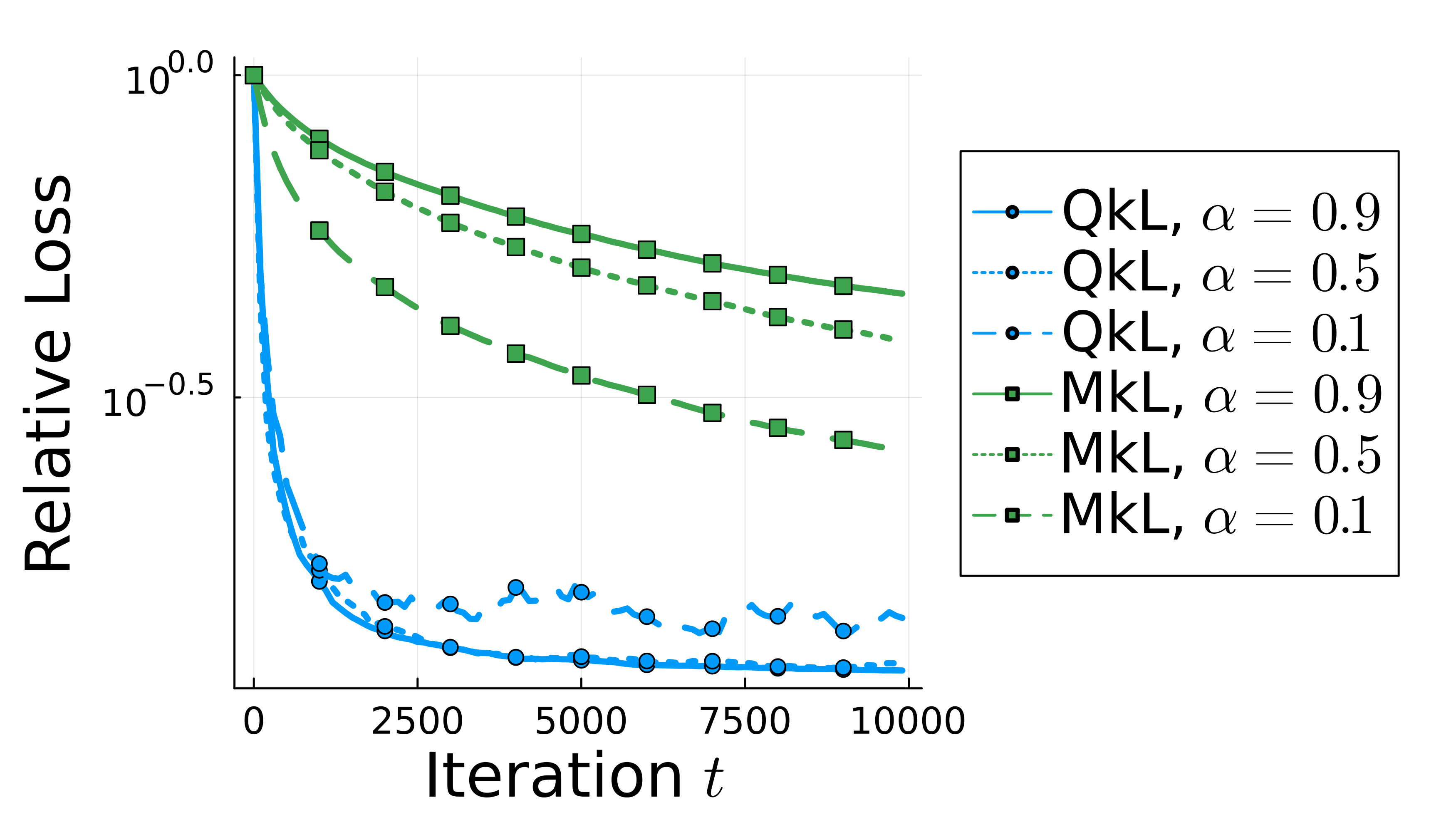}
    \hfill\includegraphics[width=0.48\textwidth]{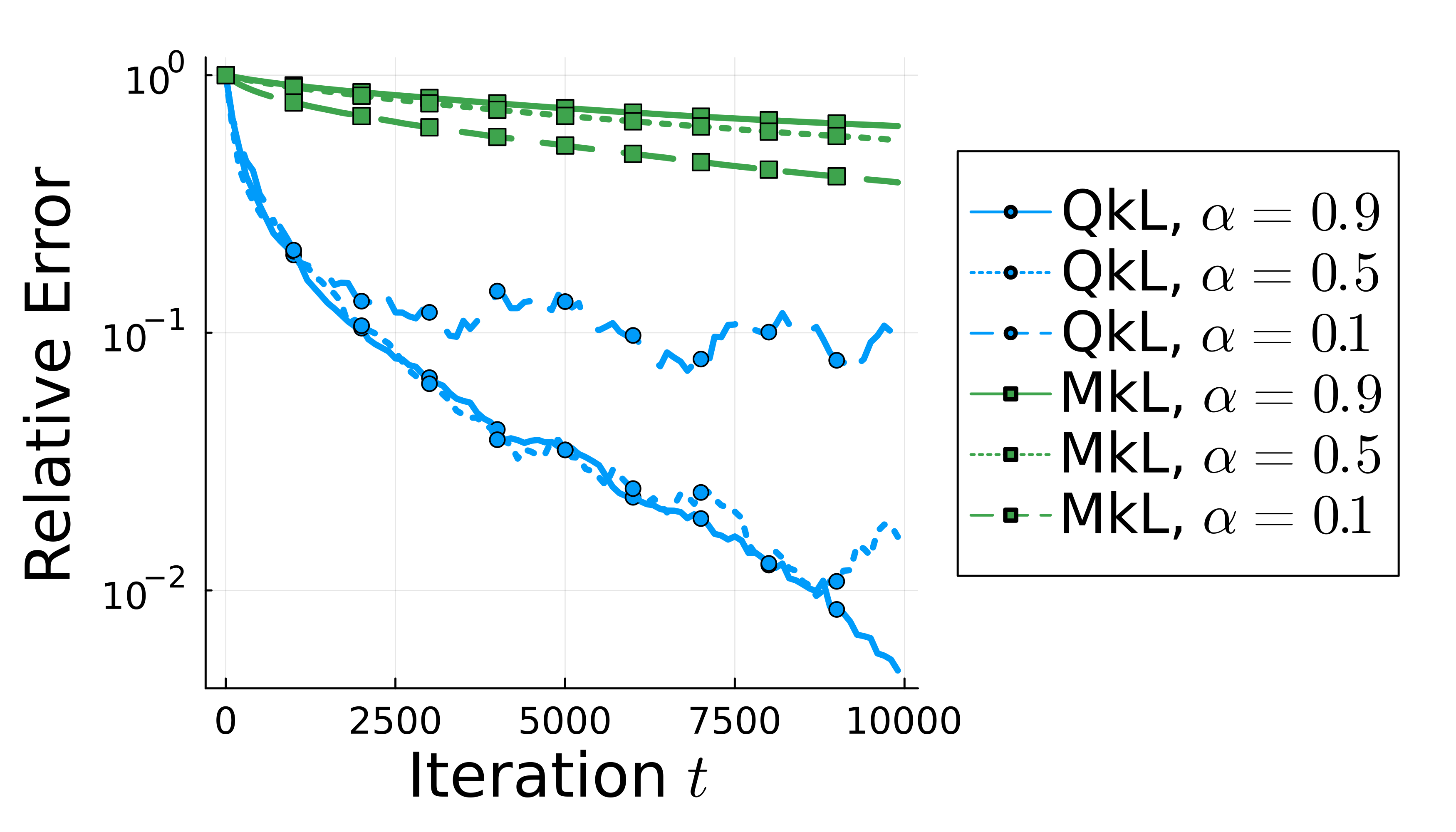}
    \caption{Regularized logistic regression objective~\eqref{eq:logistic_regression}.}
    \label{fig:figure-09c}
\end{subfigure}

\begin{subfigure}{\textwidth}
    \centering
    \includegraphics[width=0.48\textwidth]{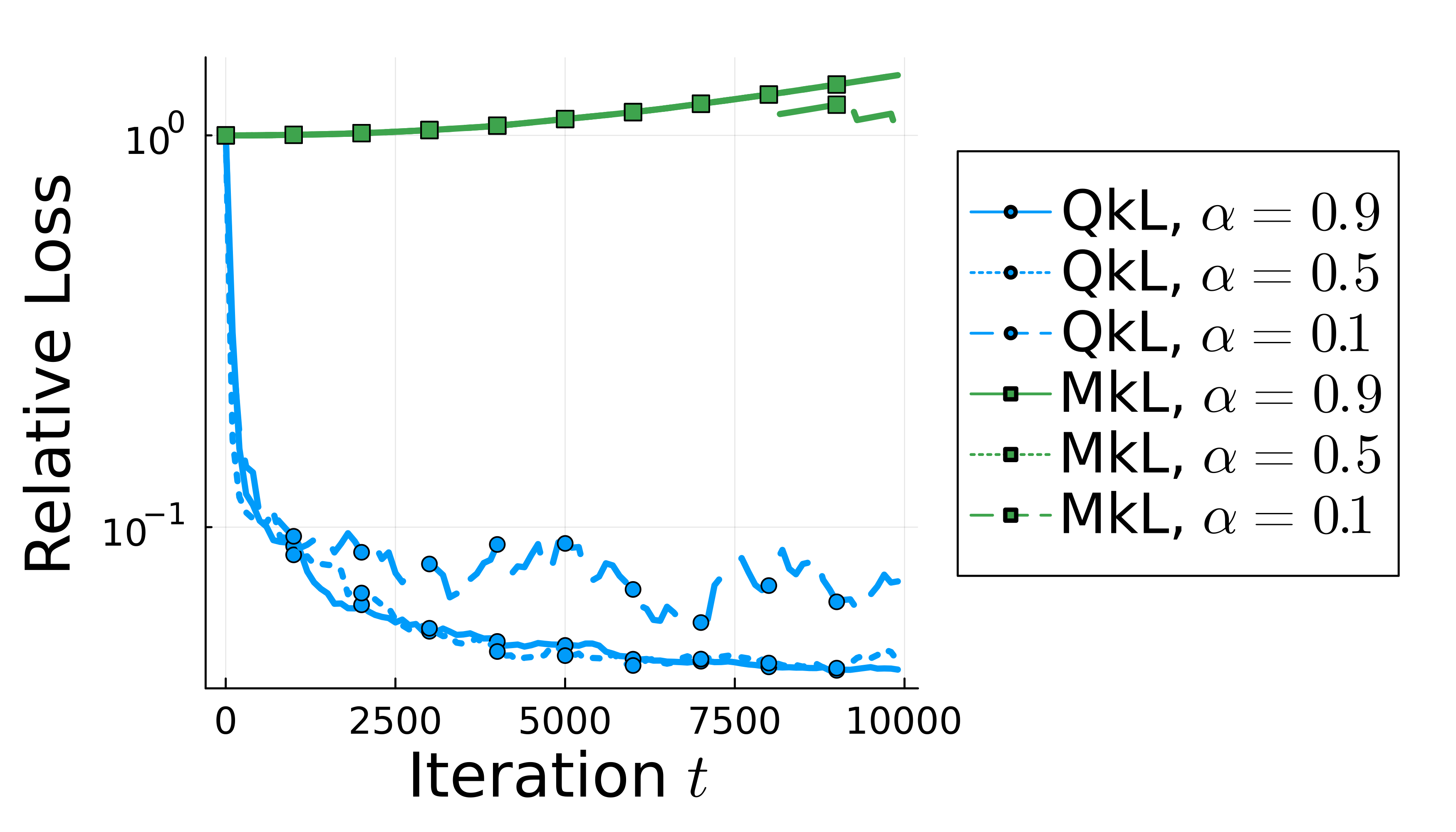}
    \hfill\includegraphics[width=0.48\textwidth]{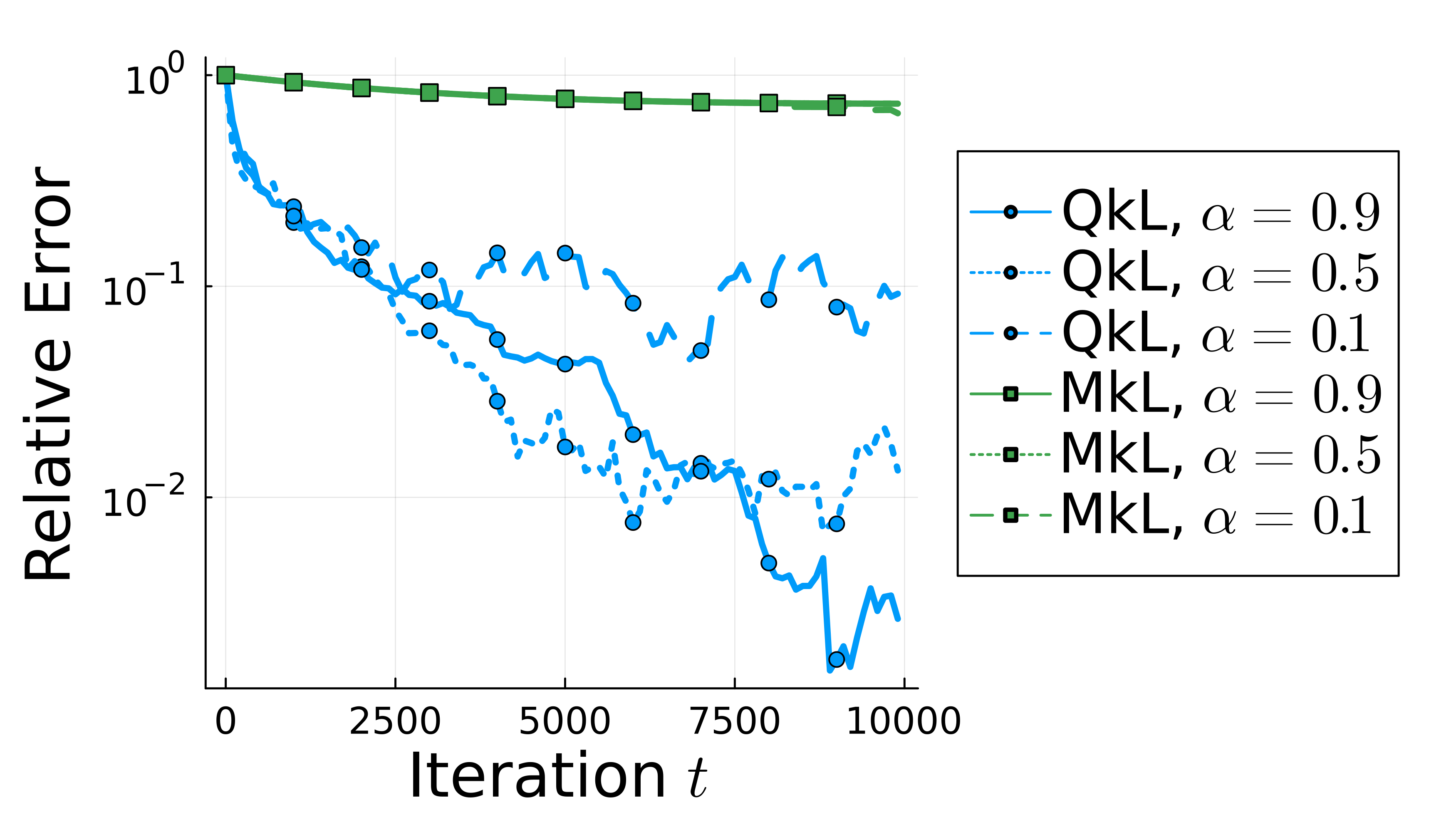}
    \caption{Regularized hinge loss objective~\eqref{eq:hinge_loss}.}
    \label{fig:figure-09d}
\end{subfigure}

\caption{(Left) Relative loss $F_G(\mathbf{x}_t)/F_G(\mathbf{x}_0)$ per iteration (log scale); (right) relative squared error $\|\mathbf{x}_t - \mathbf{x}^\star_{\rm clean}\|^2/\|\mathbf{x}_0 - \mathbf{x}^\star_{\rm clean}\|^2$ per iteration (log scale). In these experiments, $\beta = 0.05$, $q = 0.95$, and we vary $\alpha \in \{0.1, 0.5, 0.9\}$ for QkL-SGD and MkL-SGD.}
\end{figure}

In our final experiment presented in Figure~\ref{fig:figure-10}, we plot the final relative loss or error (after 10000 iterations) for $\alpha \in (0,1]$ for QkL-SGD and MkL-SGD applied to the polynomial regression, logistic regression, and hinge loss problems described above.  We use the same problem set-up described previously in this section for noiseless polynomial regression (top row), noisy polynomial regression (row two), regularized logistic regression (row three), and regularized hinge loss (bottom row).  The lines represent the average over 100 independent trials of 10000 iterations, and the shaded ribbon presents the spread of the maximum and minimum final relative losses and errors across the 100 trials.  We note that generally MkL-SGD performs best with smaller values of $\alpha$ in terms of both relative error and loss, while QkL-SGD performs best with a larger value of $\alpha$, although the best behavior is for $\alpha < 1$.

\begin{figure}[ht]
    \centering
    \begin{subfigure}[b]{\textwidth}\includegraphics[width=0.42\textwidth]{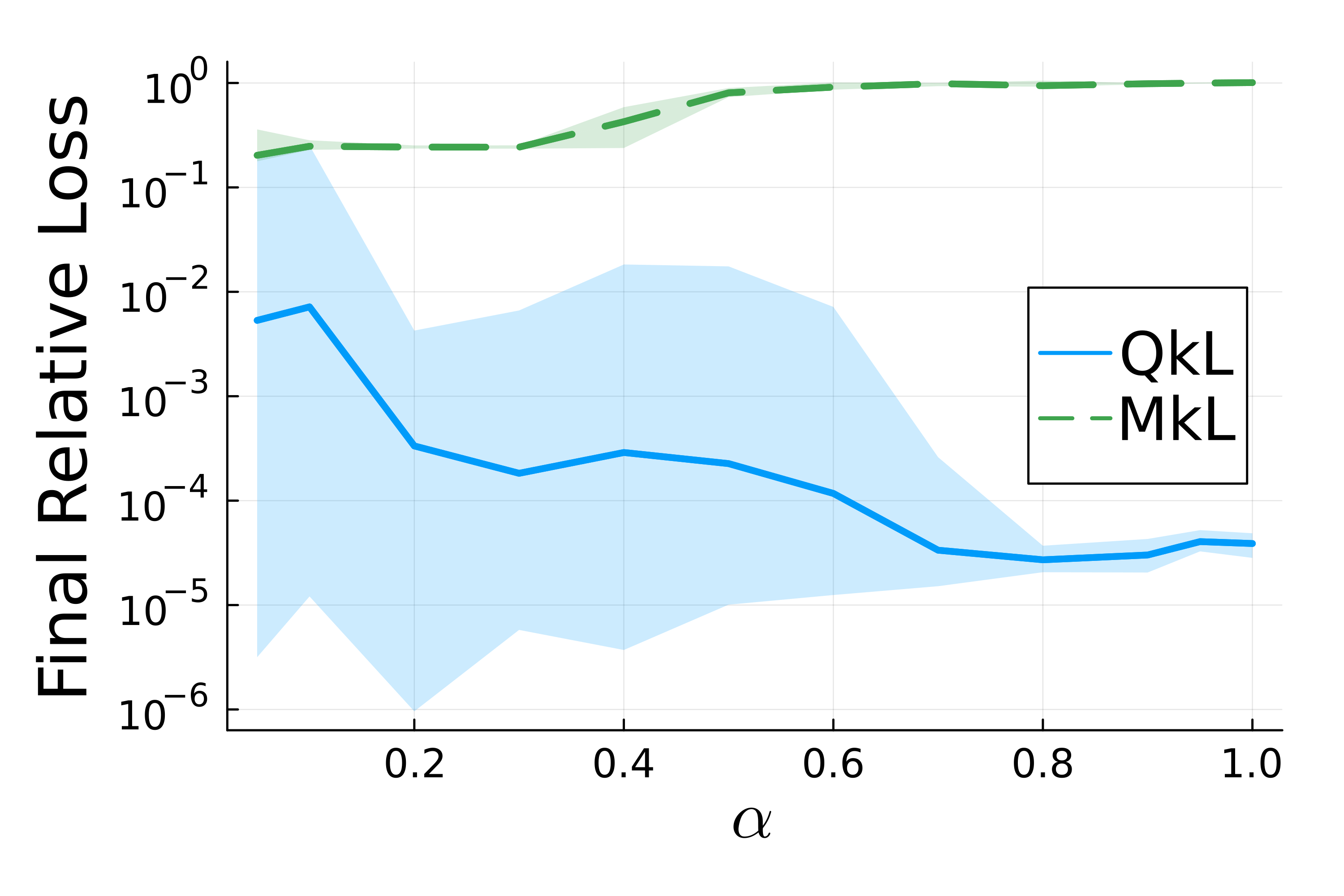}\hfil\includegraphics[width=0.42\textwidth]{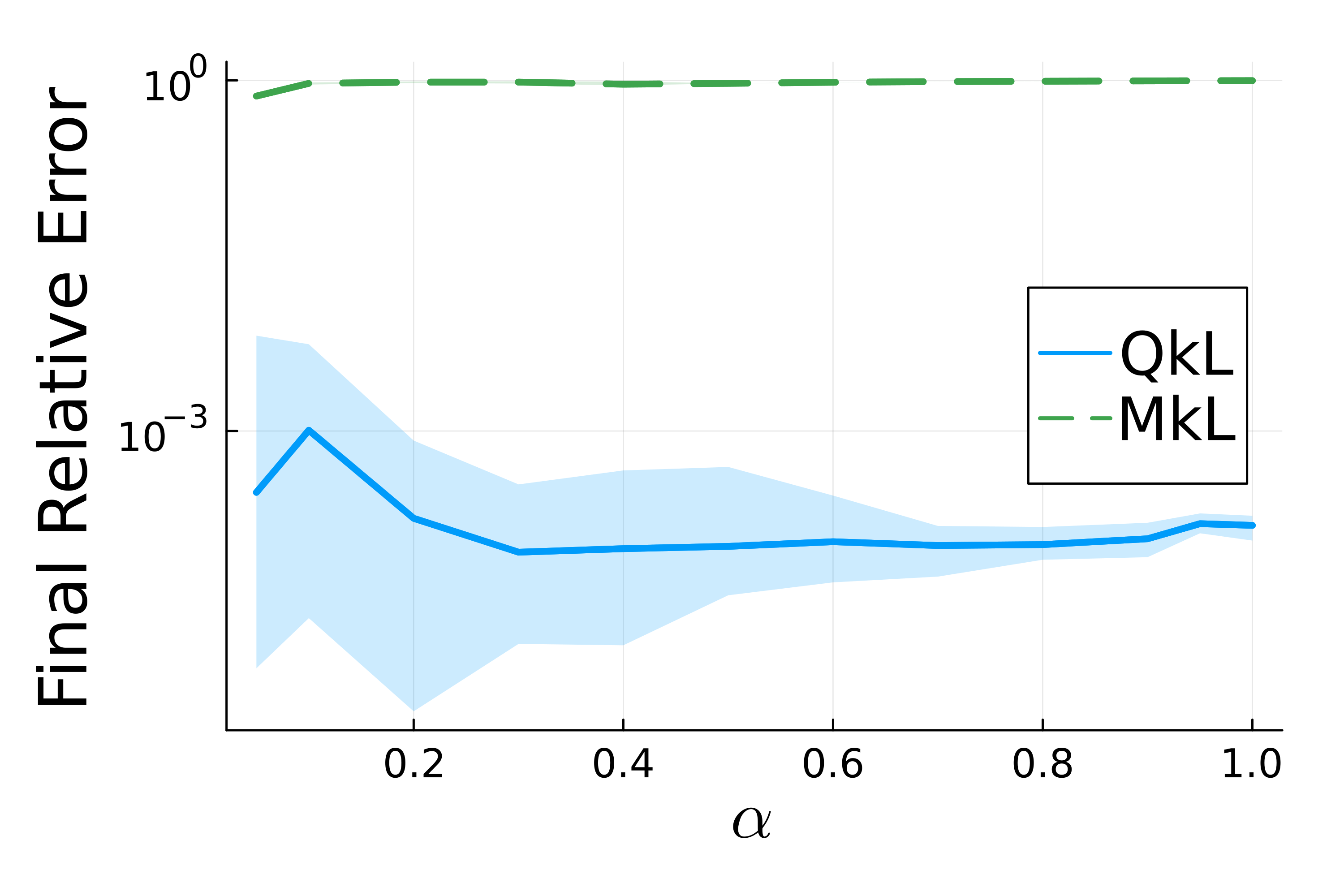}
    \caption{Noiseless polynomial regression loss~\eqref{eq:poly_regression}}\end{subfigure}

    \begin{subfigure}[b]{\textwidth}\includegraphics[width=0.42\textwidth]{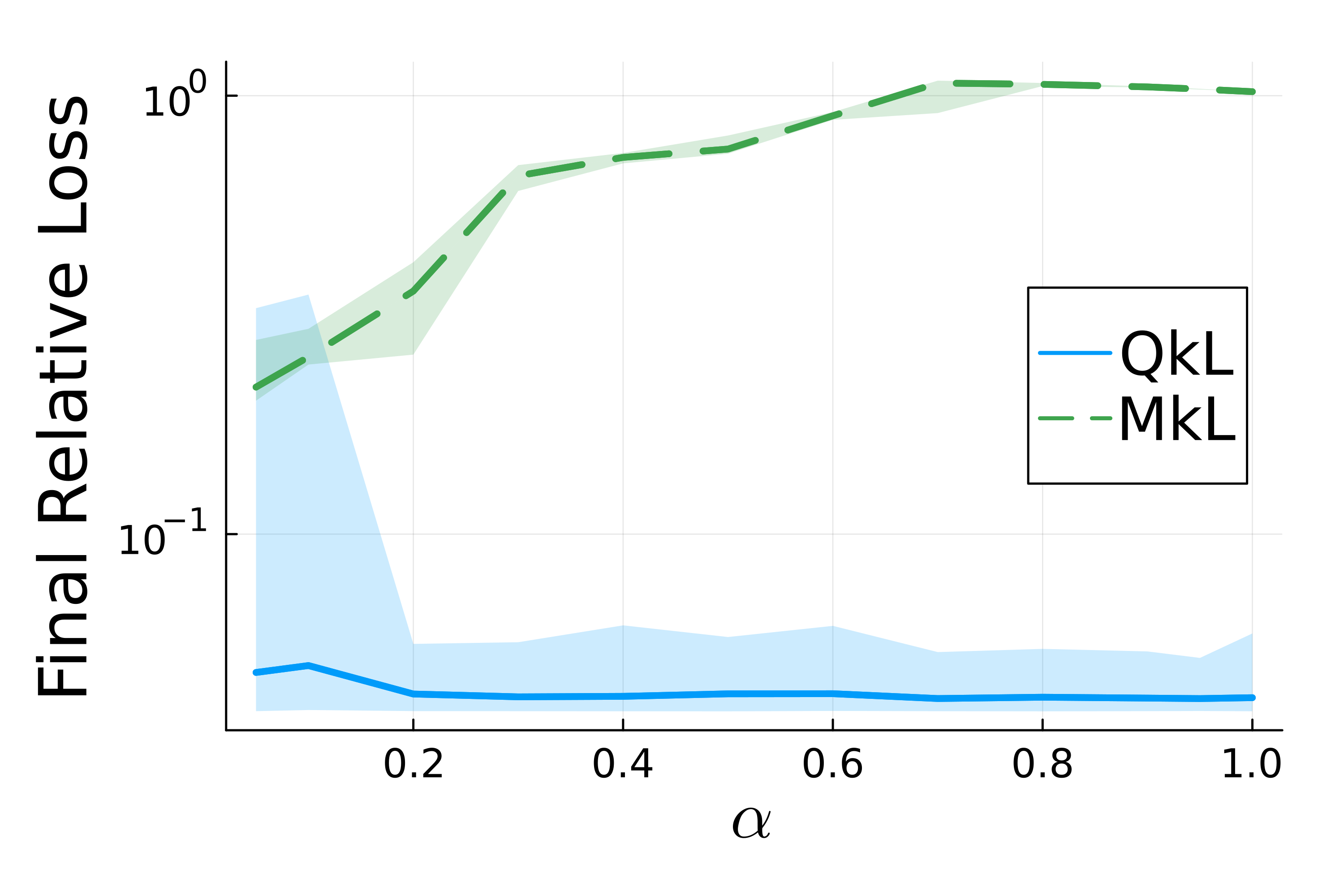}\hfil\includegraphics[width=0.42\textwidth]{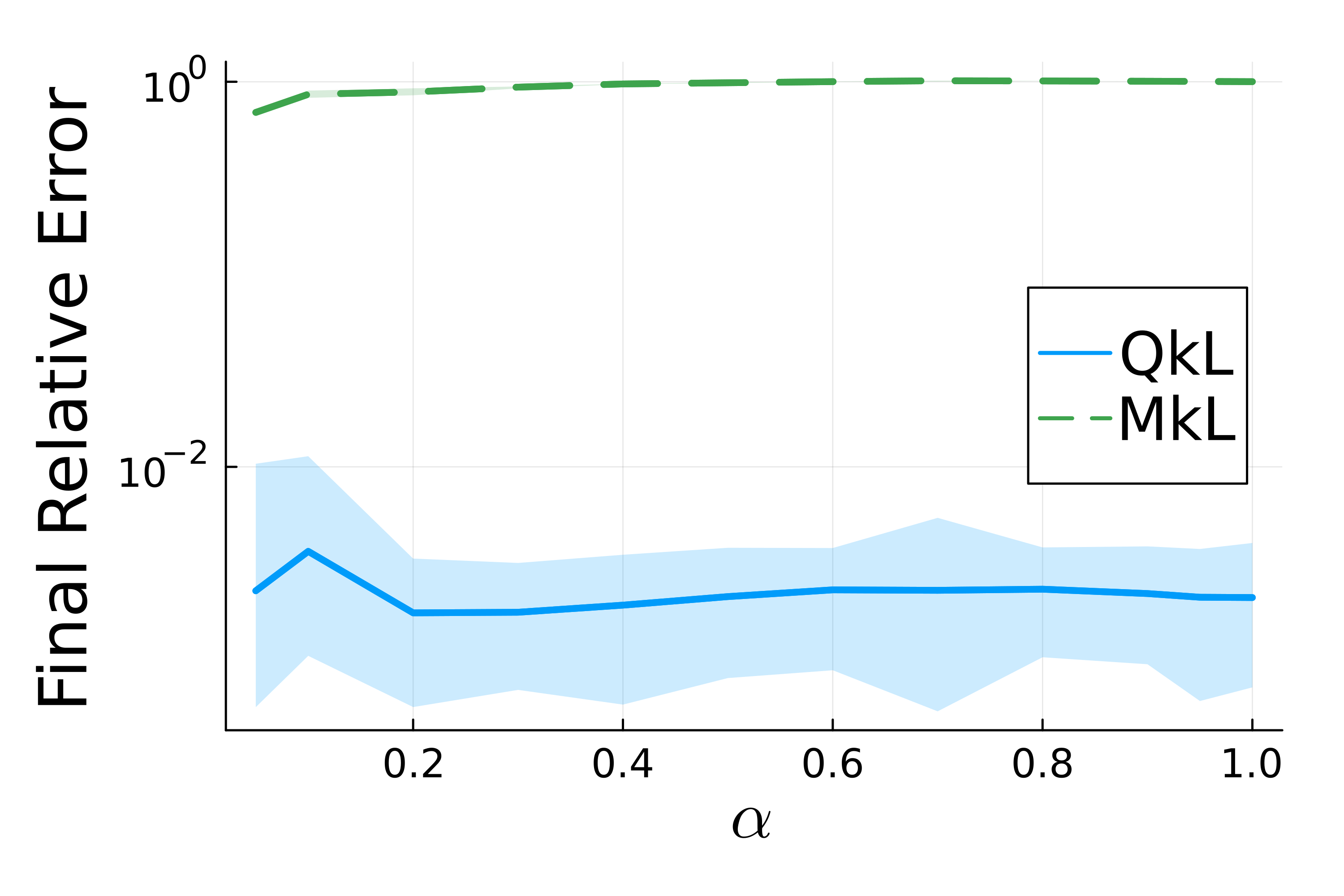}
    \caption{Noisy polynomial regression loss~\eqref{eq:poly_regression}}\end{subfigure}

    \begin{subfigure}[b]{\textwidth}\includegraphics[width=0.42\textwidth]{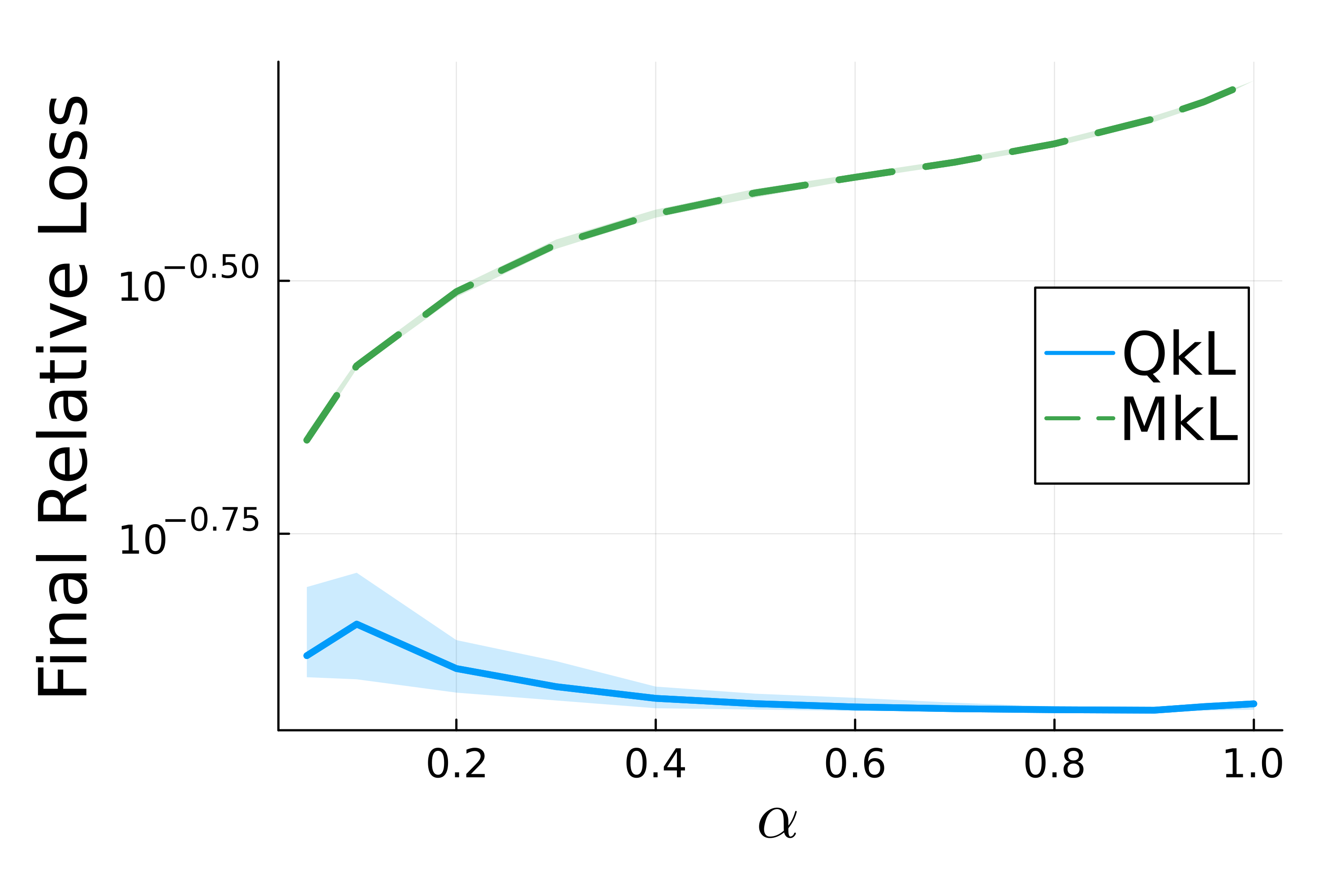}\hfil\includegraphics[width=0.42\textwidth]{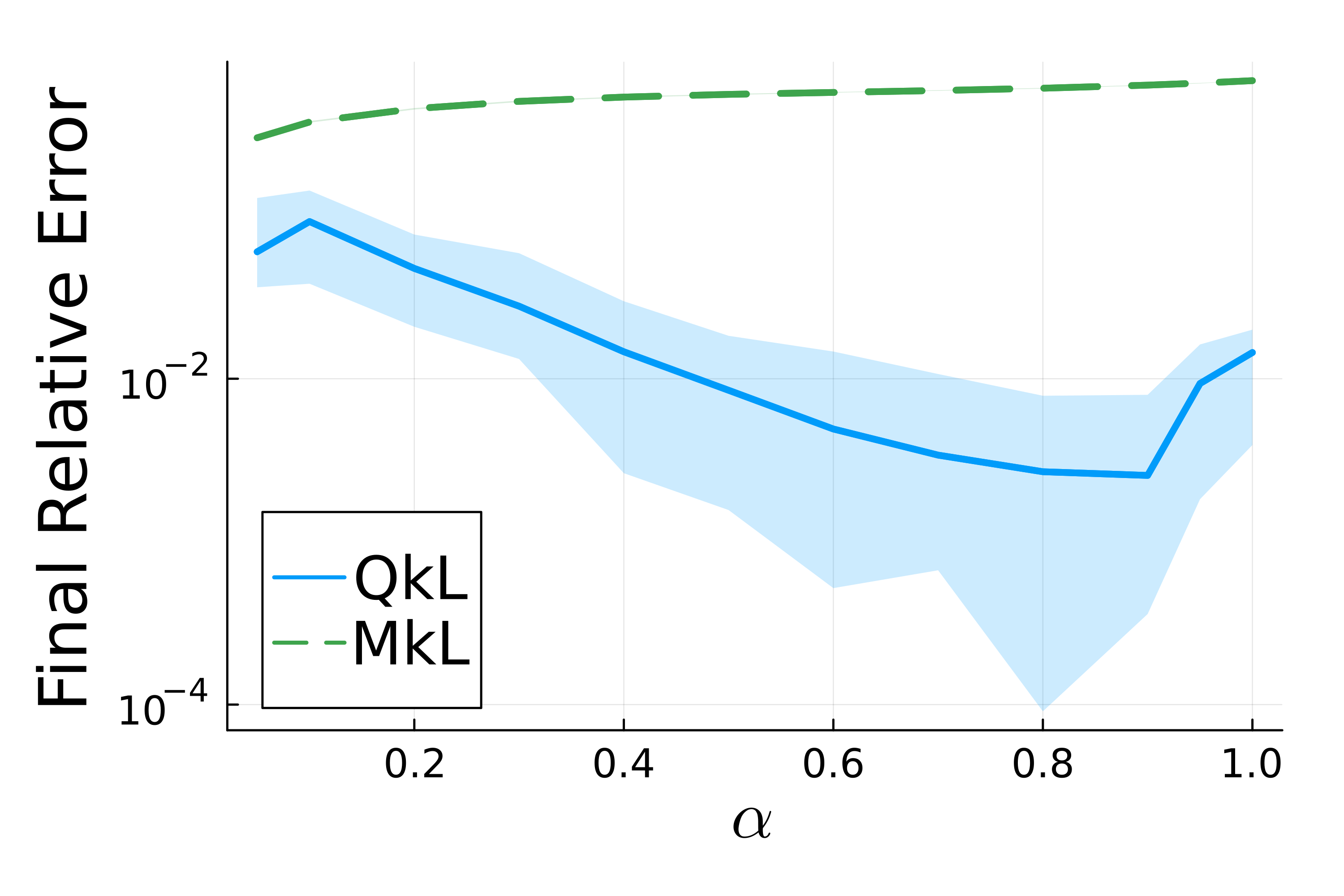}
    \caption{Regularized logistic regression loss~\eqref{eq:logistic_regression}}\end{subfigure}

    \begin{subfigure}[b]{\textwidth}\includegraphics[width=0.42\textwidth]{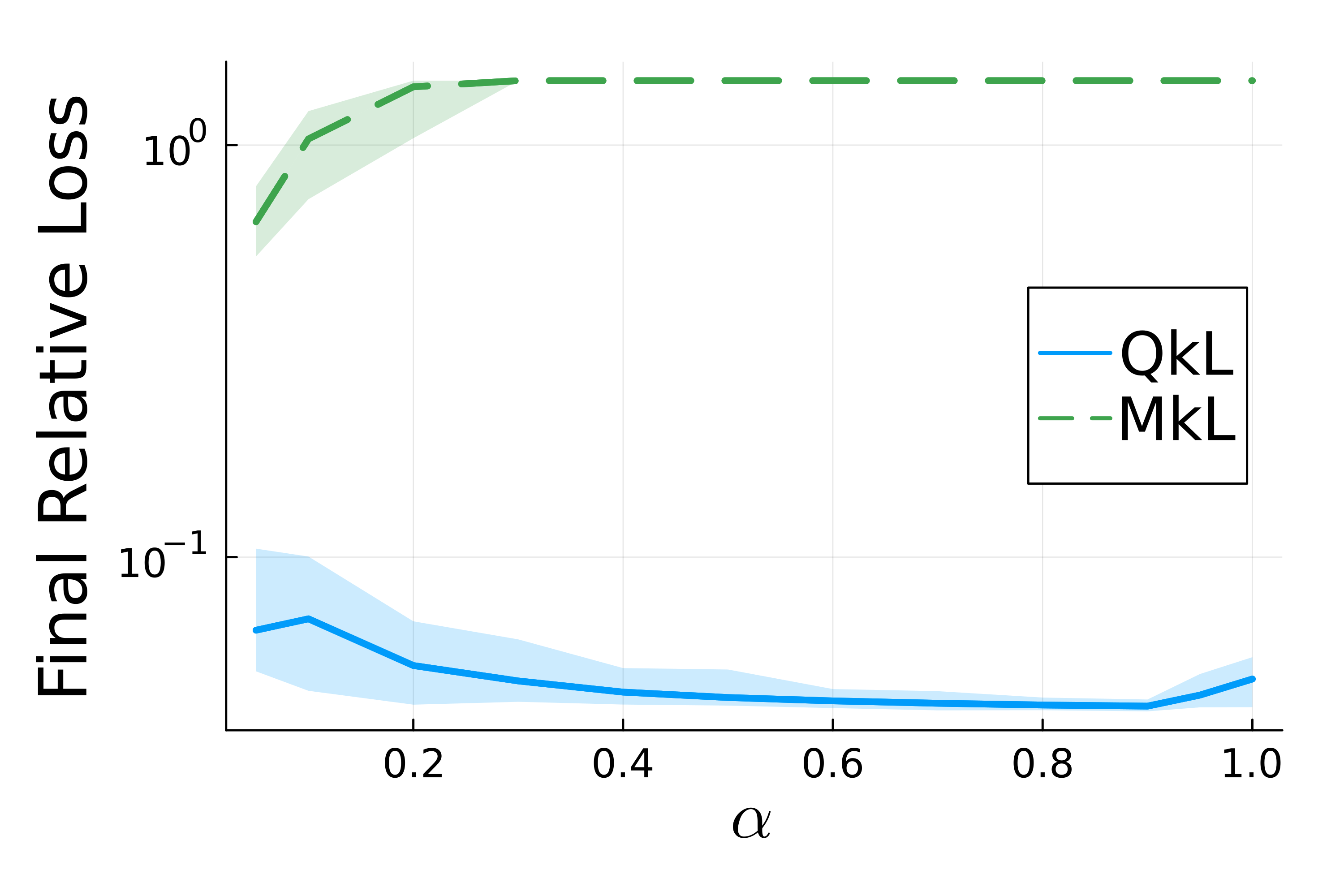}\hfil\includegraphics[width=0.42\textwidth]{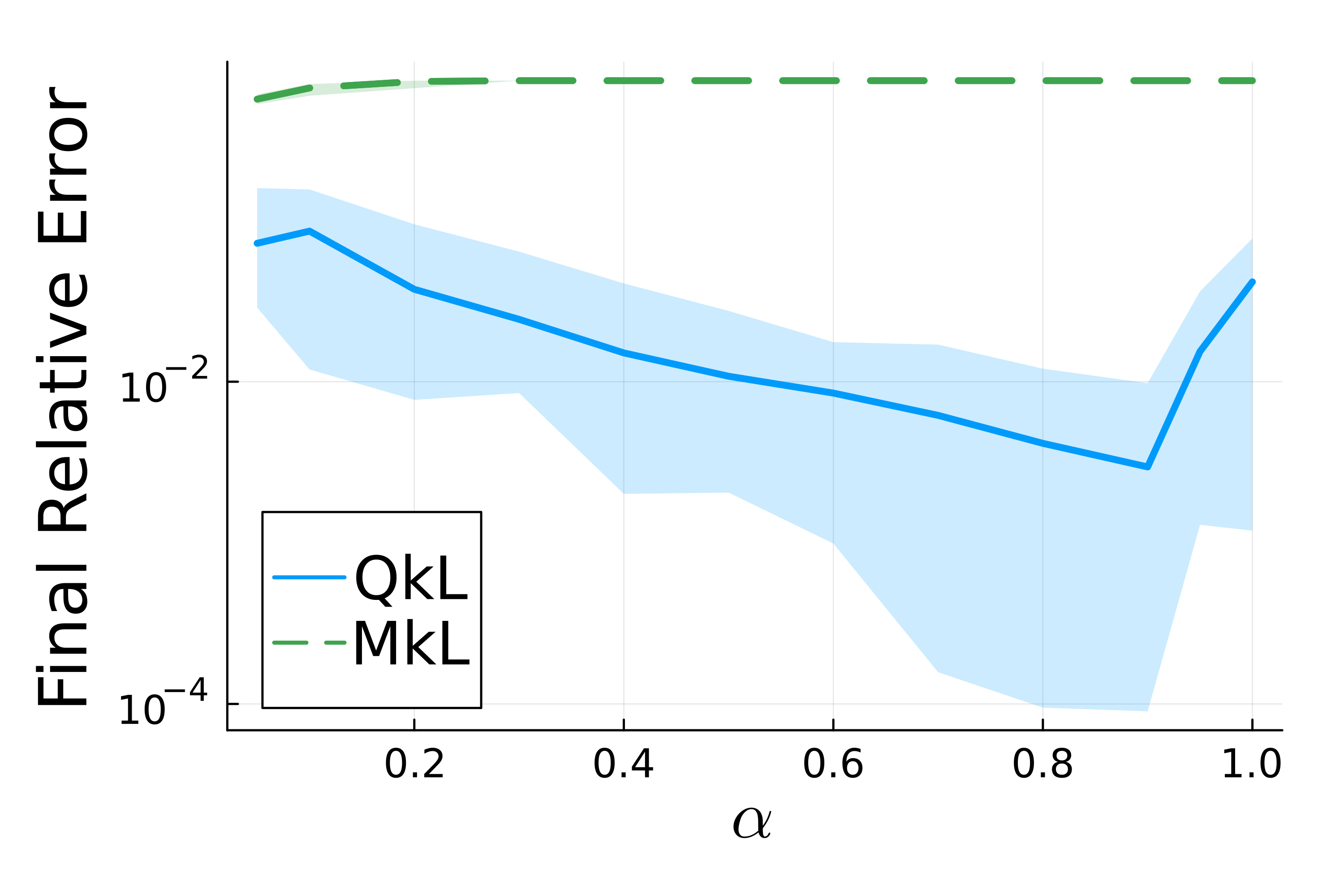}
    \caption{Regularized hinge loss~\eqref{eq:hinge_loss}}\end{subfigure}

    \caption{(Left) Final relative loss $F_G(\mathbf{x}_{10000})/F_G(\mathbf{x}_0)$ for QkL-SGD and MkL-SGD with $\alpha \in (0,1]$; (right) final relative squared error $\|\mathbf{x}_{10000} - \mathbf{x}^\star_{\rm clean}\|^2/\|\mathbf{x}_0 - \mathbf{x}^\star_{\rm clean}\|^2$ for QkL-SGD and MkL-SGD with $\alpha \in (0,1]$ (log scale).  In all experiments, we have $\beta = 0.05$ and we apply QkL-SGD with $q = 0.95$.}\label{fig:figure-10}
\end{figure}

\subsubsection{Small Sample Size}\label{subsec:small_sample}
Noting that MkL-SGD was intended to be used with a very small sample size $k$, we now explore how the loss and error of MkL-SGD and QkL-SGD decrease with respect to cumulative evaluations of loss components. Figure~\ref{fig:figure-11a} shows the relative loss and relative error measured against the number of loss component evaluations for QkL-SGD with sampling rate $\alpha = 1.0$ and $q = 0.95$ and for MkL-SGD with $k \in \{2,5,10\}$ on the polynomial regression problem~\eqref{eq:poly_regression} on the noiseless uncorrupted data. Starting from the baseline dataset described in Section~\ref{subsec:noiseless_poly_regression}, we corrupt 5\% of the target values $y_i$ using additive noise drawn from $\mathcal{N}(0,100)$. The results indicate that QkL-SGD attains performance comparable to MkL-SGD with $k=2$ and $k=10$ in terms of both relative loss and relative error, while MkL-SGD with $k = 5$ achieves the lowest relative loss and error on this very simple problem.

In Figure~\ref{fig:figure-11b}, we plot the relative loss and relative error over loss component evaluations of QkL-SGD with sample rate $\alpha = 1.0$ and $q = 0.95$ and MkL-SGD with $k \in \{2, 5, 10\}$ applied to the polynomial regression loss~\eqref{eq:poly_regression}.  We generate the ideal (uncorrupted) data as described in Section~\ref{subsec:poly_regression} and then sample 5\% of the data and corrupt the corresponding targets $y_i$ by large additive noise sampled from $\mathcal{N}(0,100)$. In this case, QkL-SGD and MkL-SGD with $k = 5$ and $k=10$ achieve similar relative loss.  However, QkL-SGD achieves the lowest final relative error.

In the experiment presented in Figure~\ref{fig:figure-11c}, we generate the uncorrupt data as in Section~\ref{subsec:logistic_regression}.  We then sample 5\% of the data points for corruption and their labels are flipped.  We implement QkL-SGD with sampling rate $\alpha = 1.0$ and $q = 0.95$, and MkL-SGD with $k \in \{2, 5, 10\}$ on the regularized logistic regression loss~\eqref{eq:logistic_regression} and compare the relative loss and relative error across loss component evaluations.  We note that while the MkL-SGD loss and error decrease quickly with respect to loss component evaluations, QkL-SGD eventually achieves lower values.

Next, in the experiment given in Figure~\ref{fig:figure-11d}, we generate the data as in Section~\ref{subsec:hinge_loss} and then sample 5\% of the data and flip the labels associated to these data points.  We compare the loss and error of QkL-SGD with sampling rate $\alpha = 1.0$ and $q = 0.95$, and MkL-SGD with $k \in \{2, 5, 10\}$ applied to the regularized hinge loss objective~\eqref{eq:hinge_loss}.  In this case, the relative loss and relative error of QkL-SGD decrease at roughly the same rate as MkL-SGD and again eventually achieve lower values for each.

\begin{figure}
\begin{subfigure}{\textwidth}
    \centering
    \includegraphics[width=0.48\textwidth]{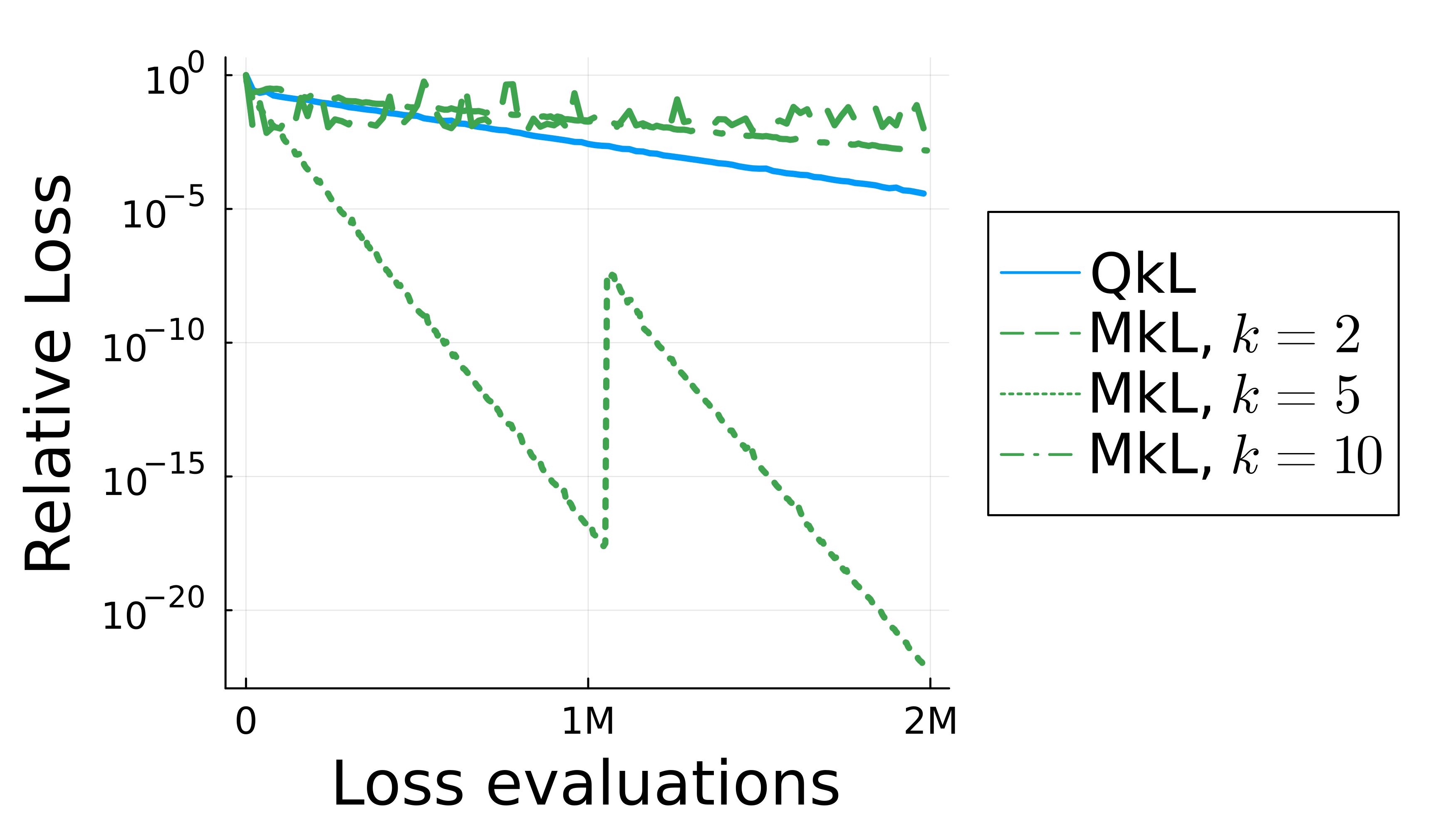}
    \hfill\includegraphics[width=0.48\textwidth]{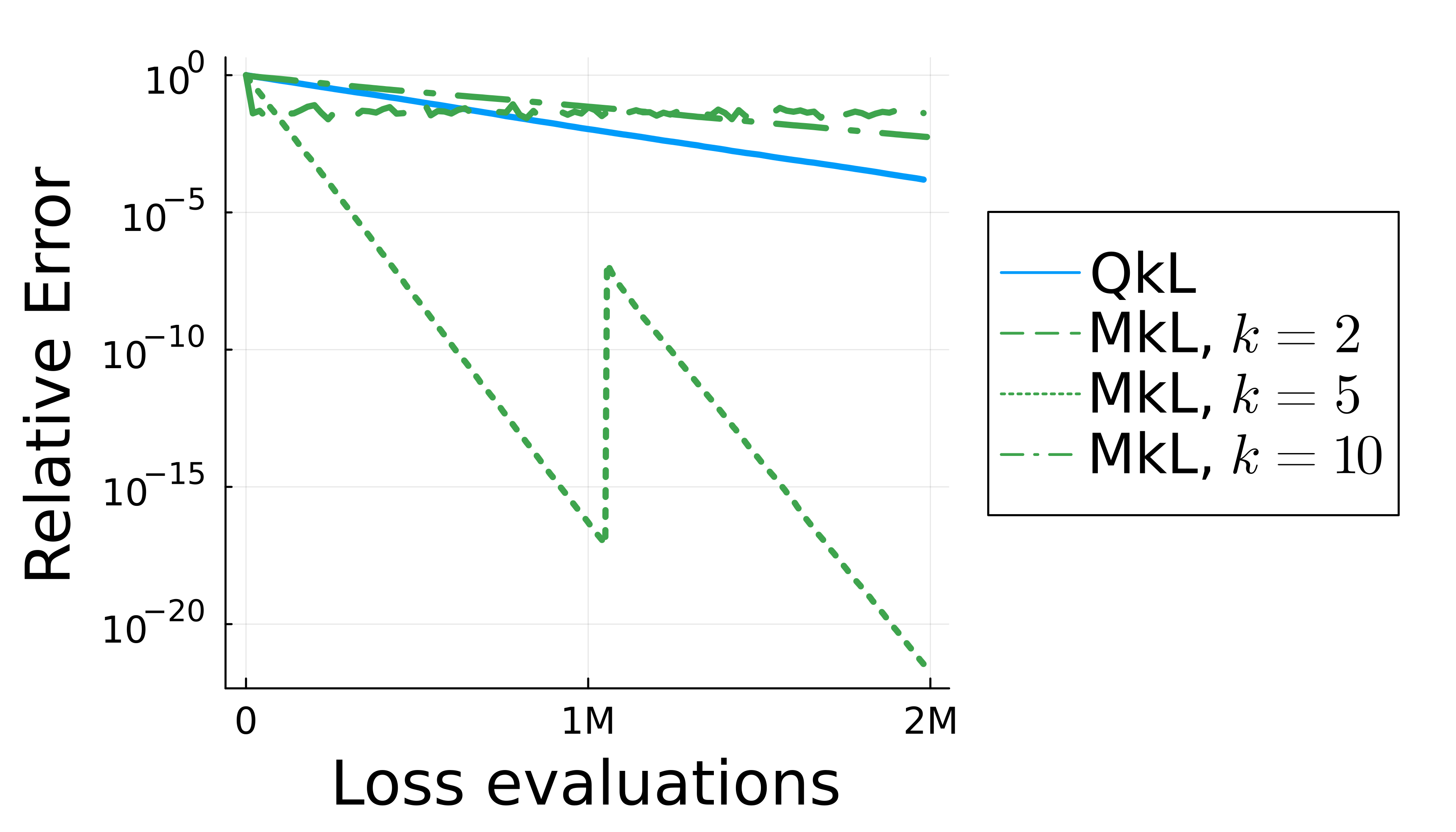}
    \caption{Polynomial regression objective~\eqref{eq:poly_regression} on noiseless uncorrupted data as in Subsection~\ref{subsec:noiseless_poly_regression}.}
    \label{fig:figure-11a}
\end{subfigure}

\begin{subfigure}{\textwidth}
    \centering
    \includegraphics[width=0.48\textwidth]{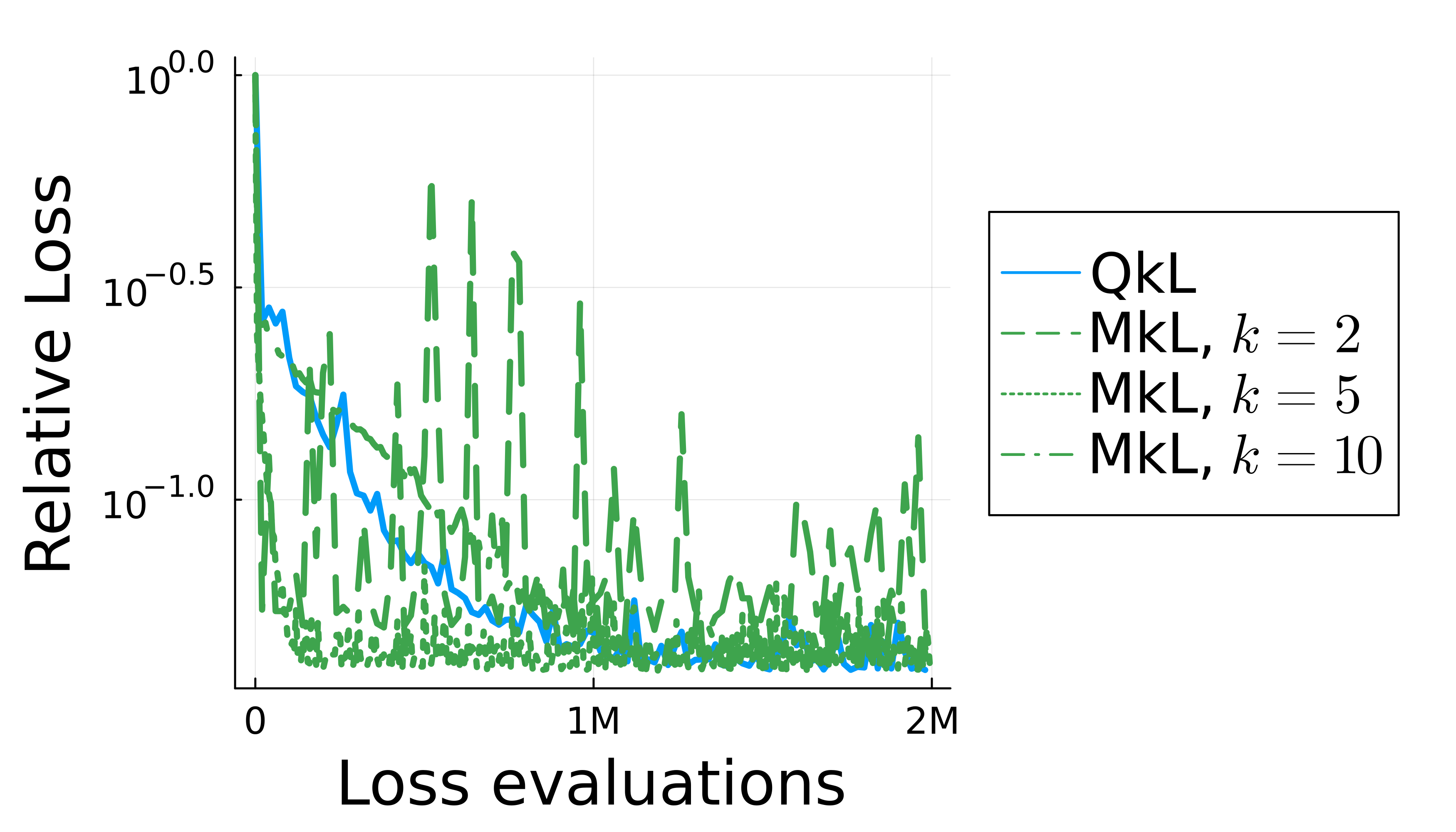}
    \hfill\includegraphics[width=0.48\textwidth]{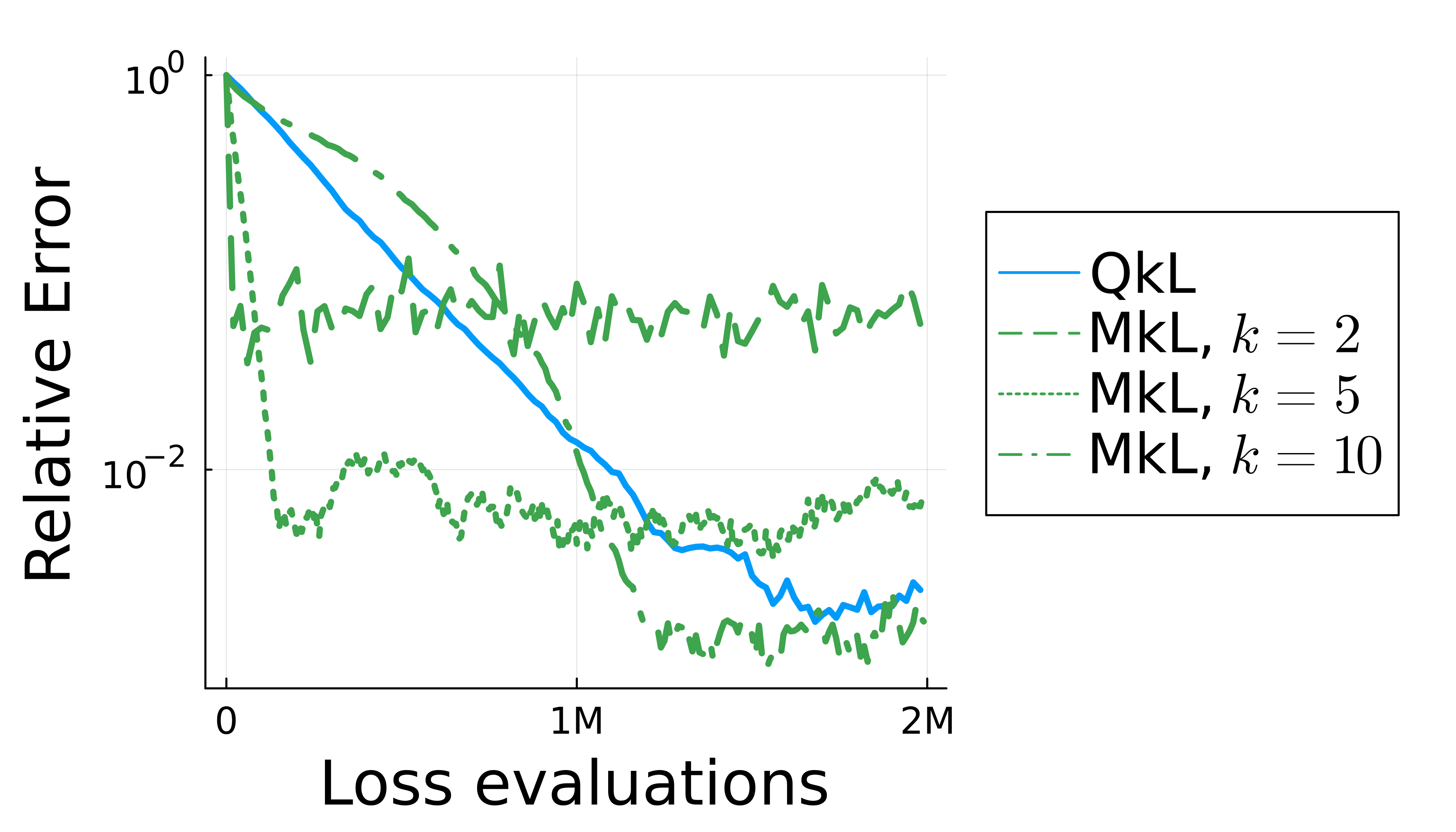}
    \caption{Polynomial regression objective~\eqref{eq:poly_regression} on noisy uncorrupted data as in Subsection~\ref{subsec:poly_regression}.}
    \label{fig:figure-11b}
\end{subfigure}

\begin{subfigure}{\textwidth}
    \centering
    \includegraphics[width=0.48\textwidth]{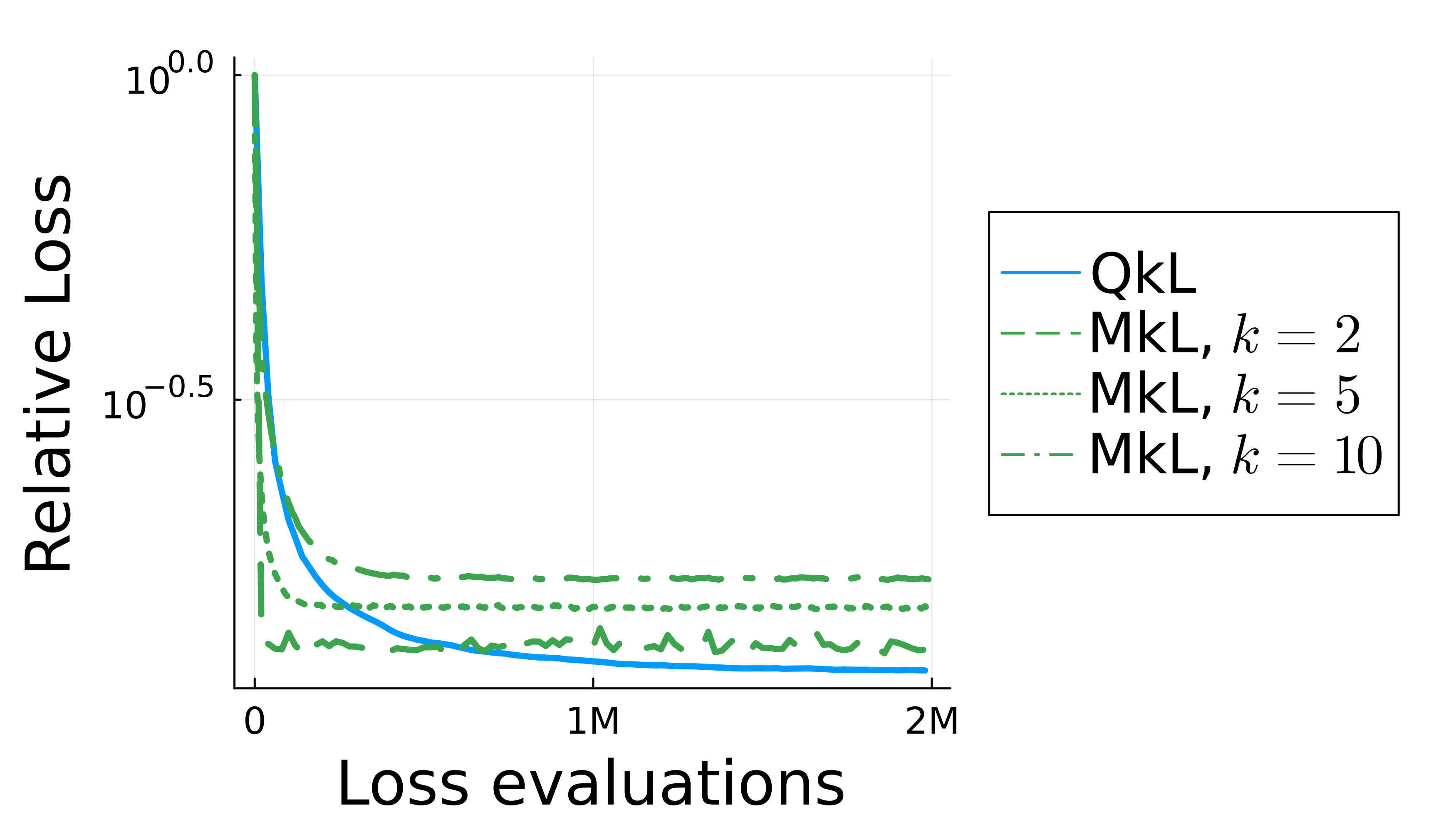}
    \hfill\includegraphics[width=0.48\textwidth]{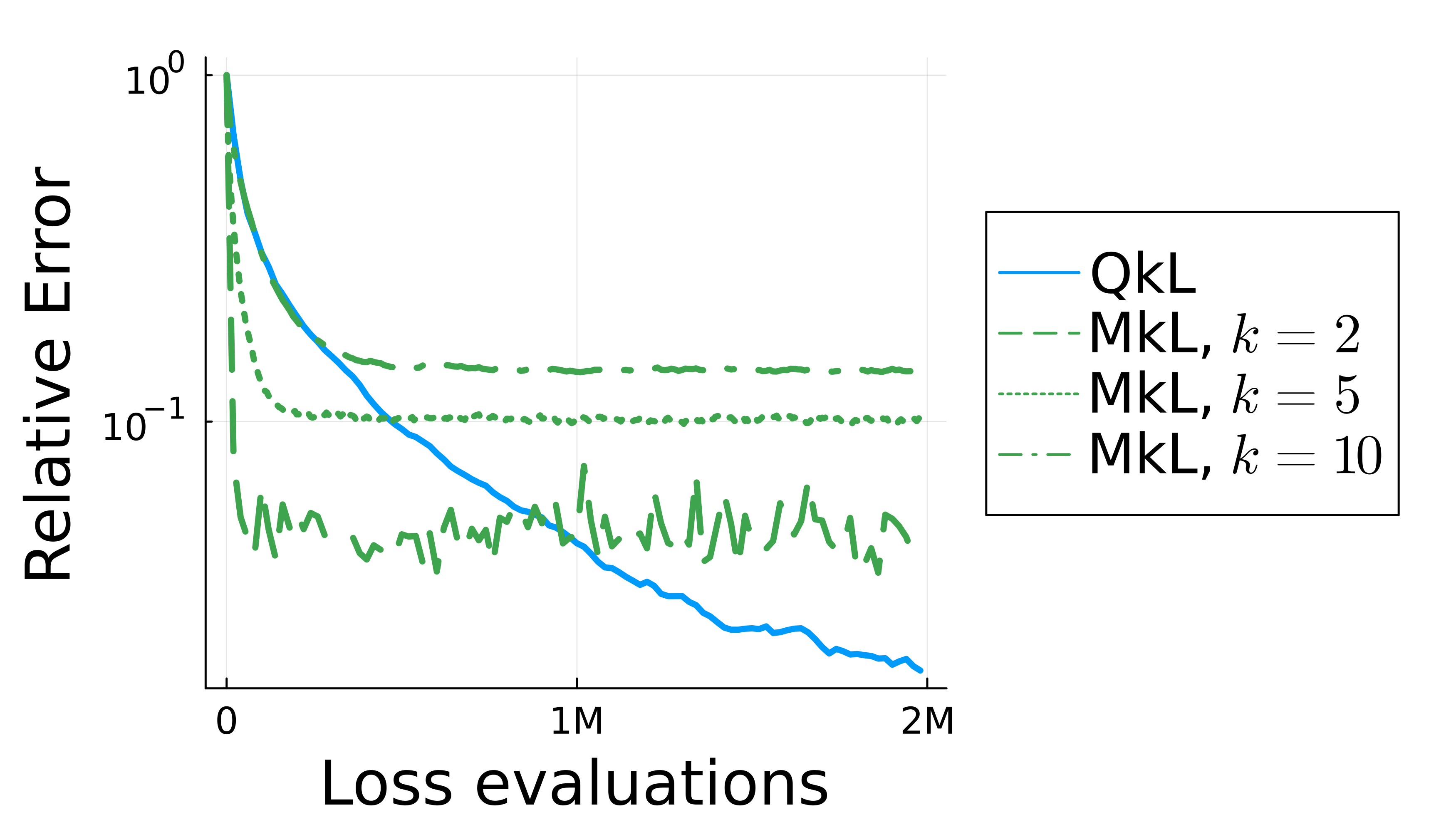}
    \caption{Regularized logistic regression objective~\eqref{eq:logistic_regression}.}
    \label{fig:figure-11c}
\end{subfigure}

\begin{subfigure}{\textwidth}
    \centering
    \includegraphics[width=0.48\textwidth]{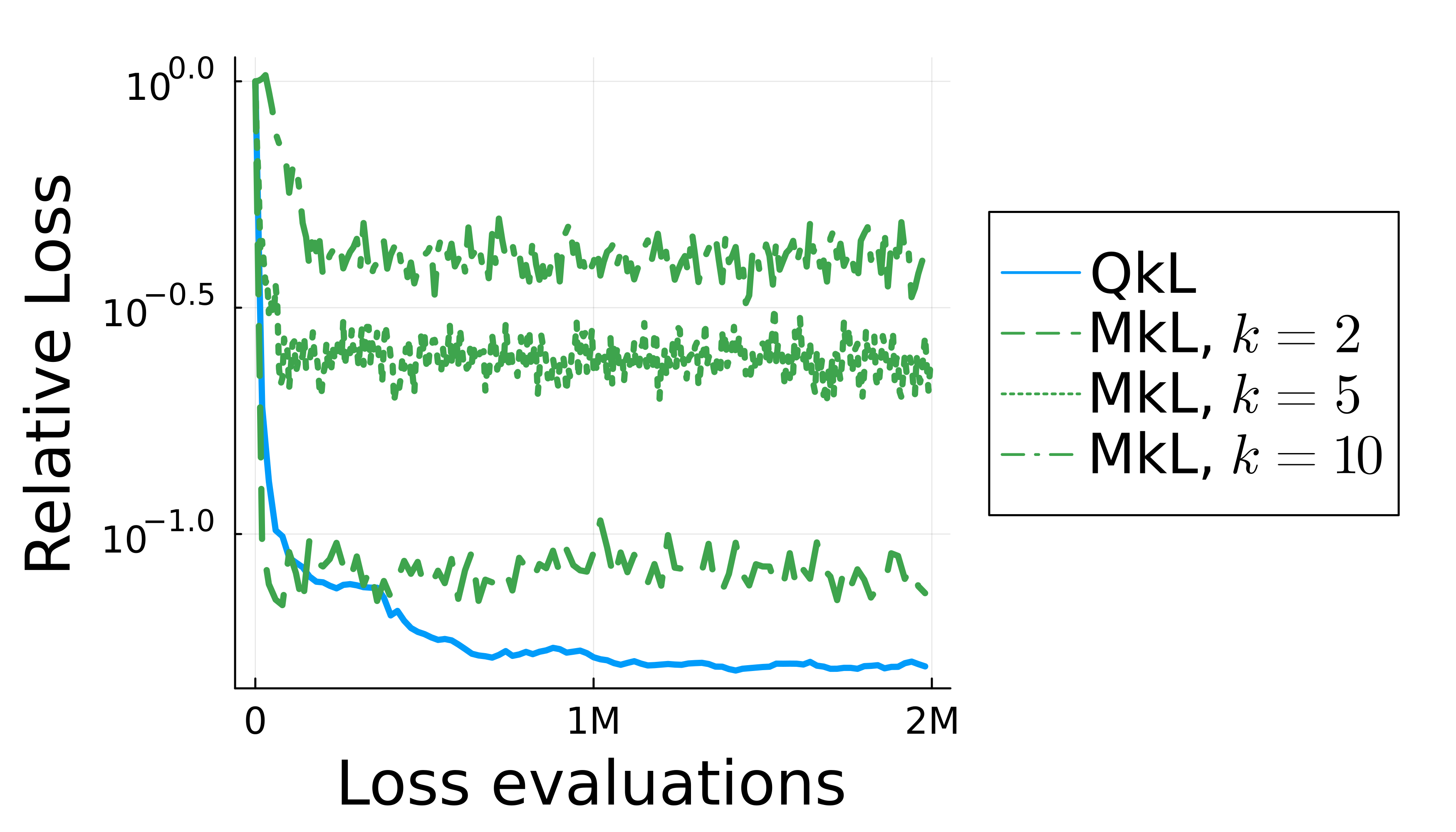}
    \hfill\includegraphics[width=0.48\textwidth]{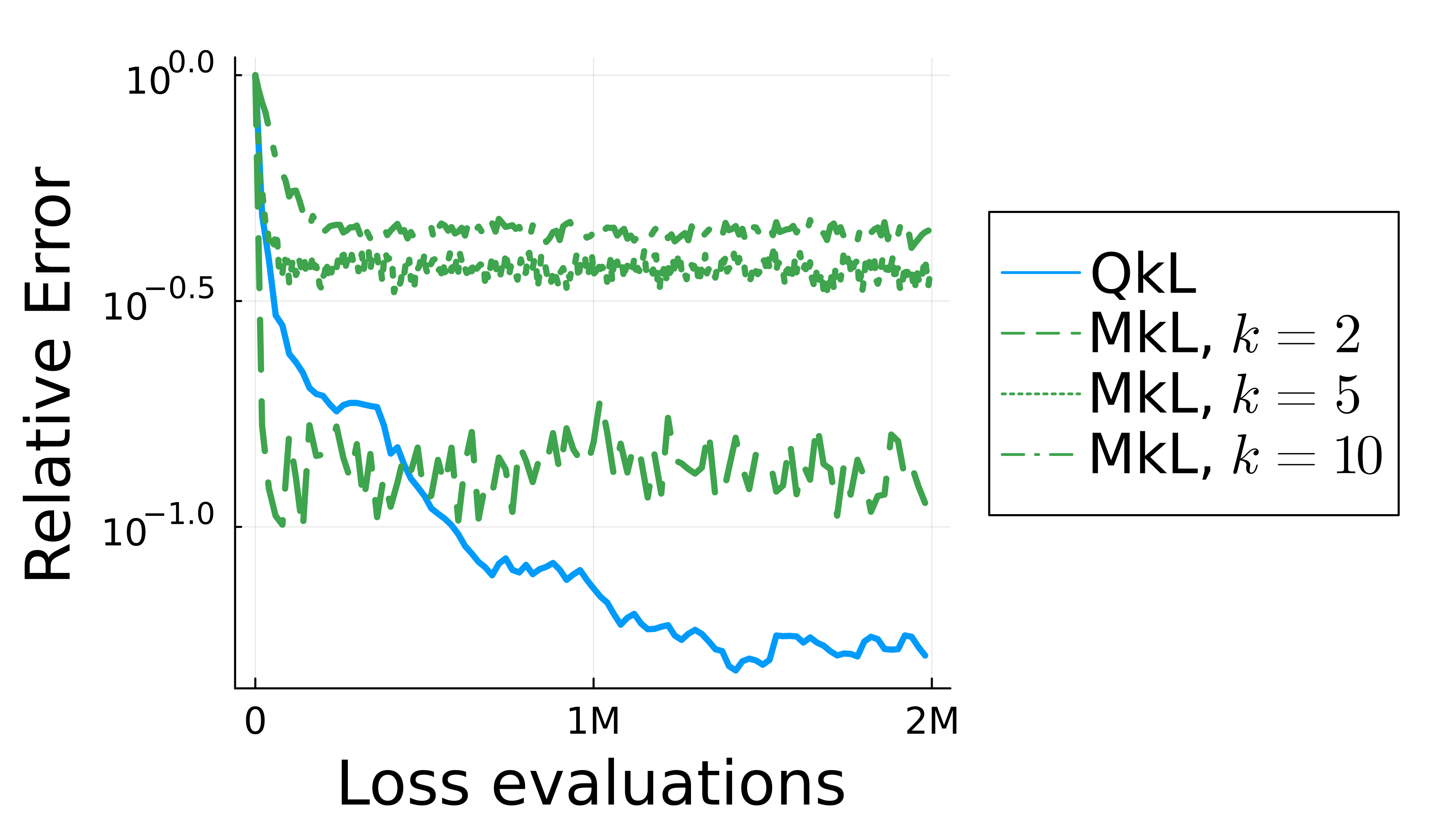}
    \caption{Regularized hinge loss objective~\eqref{eq:hinge_loss}.}
    \label{fig:figure-11d}
\end{subfigure}

    \caption{(Left) Relative loss $F_G(\mathbf{x}_t)/F_G(\mathbf{x}_0)$ per loss evaluation (log scale); (right) relative squared error $\|\mathbf{x}_t - \mathbf{x}^\star_{\rm clean}\|^2/\|\mathbf{x}_0 - \mathbf{x}^\star_{\rm clean}\|^2$ per loss evaluation (log scale). In this experiment, $\beta = 0.05$, and we set $q = 0.95$ and $\alpha = 1$ for QkL-SGD and vary $k \in \{2, 5, 10\}$ for MkL-SGD.}
\end{figure}

\section{Conclusions and Future Work} \label{sec:conclusion}

We propose a broad framework for optimization on corrupted data through quantile-based loss filtering. Our work shows that the  Q\(k\)L-SGD admits both deterministic guarantees, when the sample size scales with the corruption size, and probabilistic guarantees, in the small sample size regime. Empirically, we show that our framework is effective across several machine learning and optimization problems under a variety of models of corruption.

Many natural extensions remain open. First, Q\(k\)L-SGD uses a potentially
large set of sampled losses only to select a single update index. Since the
method has already paid the cost of evaluating \(k\) component losses, it is
natural to ask whether one can use the accepted lower-tail set more fully, for
example by averaging gradients over accepted components or by designing
large-batch variants with robust loss filtering. Second, momentum and
variance-reduced versions of Q\(k\)L-SGD may improve practical performance,
especially in regimes where the quantile rule successfully filters out large
corruptions but the selected good updates remain noisy. Developing convergence
theory for such batch, momentum, and adaptive variants are interesting
directions for future work.

\section*{Acknowledgments} J.H.\ was partially supported by NSF CAREER \#2440040 A.M.\ was partially supported by the Rose Hills Foundation and Simons Grant MPS-TSM-00014026. E.R.\ was partially supported by NSF DMS \#2309685.

Generative AI tools were used to assist with drafting and refining portions of the manuscript. The authors developed all ideas, numerical results, and conclusions and made all final editorial decisions. The authors reviewed, revised, and take full responsibility for the manuscript’s content.

\bibliography{references}

\end{document}